\documentclass[10pt]{article} % For LaTeX2e

\usepackage[accepted]{tmlr}
\usepackage{amsmath,amsfonts,bm}

\def\eqref#1{equation~\ref{#1}}
\def\1{\bm{1}}

\DeclareMathAlphabet{\mathsfit}{\encodingdefault}{\sfdefault}{m}{sl}
\SetMathAlphabet{\mathsfit}{bold}{\encodingdefault}{\sfdefault}{bx}{n}

\usepackage{hyperref}
\usepackage{url}
\usepackage{algorithm}

\usepackage{algorithmic}
\usepackage{float}
\usepackage{booktabs}
\usepackage{amsmath}
\usepackage{amssymb}
\usepackage{mathtools}
\usepackage{amsthm}
\usepackage{enumitem}
\usepackage{array}
\usepackage{subcaption}
\usepackage{soul}
\usepackage{tcolorbox}

\usepackage{wrapfig}

\usepackage{caption}
\usepackage{algorithm}
\usepackage{algorithmic}
\theoremstyle{plain}
\newtheorem{theorem}{Theorem}[section]
\newtheorem{proposition}[theorem]{Proposition}
\newtheorem{lemma}[theorem]{Lemma}
\newtheorem{corollary}[theorem]{Corollary}
\theoremstyle{definition}
\newtheorem{definition}[theorem]{Definition}
\newtheorem{assumption}[theorem]{Assumption}
\theoremstyle{remark}

\newtheorem*{theorem*}{Theorem}
\newtheorem*{lemma*}{Lemma}
\newtheorem*{corollary*}{Corollary}

\usepackage{todonotes}
 \ifodd 1
\newcommand{\com}[1]{\textbf{\color{red}(comment: #1)}} %comment of the text
\newcommand{\res}[1]{\textbf{\color{magenta}(RESPONSE: #1)}} %response to comments
\newcommand{\ding}[1]{\textbf{\color{magenta} #1}} %response to comments
\newtheoremstyle{proofsketchstyle}%  % Name
  {}{}%                             % Space above and below
  {\itshape}%                      % Body font
  {}%                              % Indent amount
  {\itshape}%                      % Theorem head font --> italic!
  {.}%                             % Punctuation after theorem head
  { }%                             % Space after theorem head
  {\thmname{#1}\thmnote{ #3}}      % Theorem head spec
\usepackage{wrapfig}
\theoremstyle{proofsketchstyle}

\newcommand{\alg}{\texorpdfstring{$\mathsf{SPRINT\text{-}SVRG}$\,}{SPRINT}}
\newcommand{\algns}{\texorpdfstring{$\mathsf{SPRINT\text{-}SVRG}$}{SPRINT}}

\ifodd 1

\else

\newcommand{\com}[1]{}
\newcommand{\res}[1]{}
\fi

\newcommand{\algorithmicoutput}{\textbf{Output:}}
\newcommand{\OUTPUT}{\item[\algorithmicoutput]}

\title{\algns: Stochastic Performative Prediction With Variance Reduction}

\author{\name Tian Xie\thanks{Equal contribution.} \email xie.1379@osu.edu \\
      \addr Department of Computer Science and Engineering\\
      The Ohio State University\\
    \name Ding Zhu\footnotemark[1]
    \email zhu.3723@osu.edu \\
      \addr Department of Computer Science and Engineering\\
      The Ohio State University\\
      \name Jia Liu \email liu.1736@osu.edu\\
     \addr Department of Electrical and Computer Engineering\\
      The Ohio State University \\
      \name Mohammad Mahdi Khalili  \email khalili.17@osu.edu\\
     \addr Department of Computer Science and Engineering\\
      The Ohio State University \\
      \name Xueru Zhang  \email zhang.12807@osu.edu\\
     \addr Department of Computer Science and Engineering\\
      The Ohio State University \\
      }

\def\month{10}  % Insert correct month for camera-ready version
\def\year{2026} % Insert correct year for camera-ready version
\def\openreview{\url{https://openreview.net/forum?id=q0L3a1zSro}
} % Insert correct link to OpenReview for camera-ready version

\begin{document}

\maketitle

\begin{abstract}
Performative prediction (PP) provides a framework for learning in settings where model deployment influences the data distribution on which the model is trained. Unlike conventional learning under a fixed data distribution, PP involves model-induced distribution shifts that make it challenging to design algorithms converging to stable points, known as stationary performatively stable (SPS) solutions. Existing algorithms for finding SPS solutions, including repeated gradient descent (RGD) and greedy stochastic gradient descent (SGD-GD), have largely focused on strongly convex losses. Although recent work established an $\mathcal{O}(1/\sqrt{T})$ convergence guarantee for smooth non-convex losses, it relies on a bounded-variance assumption on stochastic gradients and converges only to a non-vanishing error neighborhood whose size scales with the variance. To address these limitations, we develop a variance-reduced stochastic optimization method for non-convex performative prediction. Our algorithm, \ul{\textbf{s}}tochastic \ul{\textbf{p}}erformative p\ul{\textbf{r}}ediction w\ul{\textbf{i}}th varia\ul{\textbf{n}}ce reduc\ul{\textbf{t}}ion (\algns), converges at a rate of $\mathcal{O}(1/T)$ to a variance-independent $\mathcal{O}(\epsilon^2 + \epsilon^4)$ neighborhood of an SPS solution. This guarantee removes the bounded-variance assumption on stochastic gradients and contains no variance-dependent residual error. Experiments on multiple real-world datasets with non-convex models demonstrate that \algns~converges faster and more stably than SGD-GD {in terms of the SPS stationarity measure $ \| \nabla\mathcal{J} (\boldsymbol{\widetilde{\theta}}^s; \boldsymbol{\widetilde{\theta}}^s)\|^2$}. The implementation and experimental code are publicly available at \url{https://github.com/osu-srml/SPRINT}.
\end{abstract}

\section{Introduction}

Machine learning (ML) models are often trained under the assumption that both training and testing data are drawn independently from a fixed distribution. In many real-world applications, however, this assumption fails. When ML systems make predictions about humans or strategically adaptive agents, their deployment can influence the behavior of the population interacting with the system, thereby changing the data distribution on which the model is trained. This phenomenon, known as \emph{performative effects} or \emph{model-induced distribution shift} \citep{perdomo_performative_2021,drusvyatskiy2023stochastic}, poses fundamental challenges to conventional learning paradigms. For example, in strategic classification, individuals may modify their features, such as resumes or financial profiles, to obtain more favorable predictions \citep{hardt_strategic_2015,zhang2022fairness,xie2024automating,xie2024learning,EC_22_Jin,zuo2024lookahead}; on digital platforms, users may adjust their engagement behavior based on their beliefs about how recommendation models operate \citep{Zhang_2019_Retention, chi2022towards,jin2024addressingpolarizationunfairnessperformative}; and in adversarial settings such as spam detection, spammers continuously adapt their tactics to evade newly deployed filters. Understanding and mitigating such performative effects is therefore essential for building robust ML systems in adaptive environments.

To this end, \citet{perdomo_performative_2021} introduced performative prediction (PP), an optimization framework in which the data distribution depends explicitly on the deployed model.
In PP, the distribution $\mathcal{D}(\boldsymbol{\theta})$ is modeled as a mapping from the model parameters $\boldsymbol{\theta}$ to a space of data distributions, and the goal is to minimize the expected loss under this model-dependent distribution, known as the \textit{performative risk} (PR):  
\begin{align}\label{eq:PO}
    \boldsymbol{\theta}^{\text{PO}} =\underset{\boldsymbol{\theta}}{\operatorname{argmin}}~\mathcal{V}(\boldsymbol{\theta})\stackrel{\text{def}}{=} \mathbb{E}_{z \sim \mathcal{D}({\boldsymbol{\theta}})} [\ell(z;\boldsymbol{\theta})],
\end{align}
where $\ell(z; \boldsymbol{\theta})$ is the loss function and $z = (x,y)$ is sampled from the distribution $\mathcal{D}(\boldsymbol{\theta})$. The key challenge in PP is that $\boldsymbol{\theta}$ determines both the model and the data distribution. The minimizer $\boldsymbol{\theta}^{\text{PO}}$ is called the \textbf{performative optimal} (PO) solution. However, because the distribution is itself a function of the variable $\boldsymbol{\theta}$ being optimized, directly finding $\boldsymbol{\theta}^{\text{PO}}$ is often intractable \citep{perdomo_performative_2021, izzo2021performative}.
To address this, much of the existing literature instead studies \textbf{performatively stable} (PS) solution $\boldsymbol{\theta}^{\text{PS}}$, which minimizes the \textbf{decoupled performative risk} $\mathcal{J}(\boldsymbol{\theta};\boldsymbol{\theta}^{\text{PS}})$ defined as follows:
\begin{align}\label{eq:PS}
\displaystyle \boldsymbol{\theta}^{\text{PS}} = \underset{\boldsymbol{\theta}}{\operatorname{argmin}}~\mathcal{J}(\boldsymbol{\theta};\boldsymbol{\theta}^{\text{PS}}) \stackrel{\text{def}}{=} \mathbb{E}_{z \sim \mathcal{D}(\boldsymbol{\theta}^{\text{PS}})} [\ell(z; \boldsymbol{\theta})].
\end{align}
Unlike the performative risk $\mathcal{V}(\boldsymbol{\theta})$, where the data distribution  $\mathcal{D}(\boldsymbol{\theta})$ depends on the optimization variable,  $\mathcal{J}(\boldsymbol{\theta};\boldsymbol{\theta}^{\text{PS}})$ decouples the two: the data distribution is induced by a fixed parameter $\boldsymbol{\theta}^{\text{PS}}$, while optimization is performed over $\boldsymbol{\theta}$. 
This decoupling enables iterative schemes for finding $\boldsymbol{\theta}^{\text{PS}}$ via \textit{repeated optimization}: at each iteration, the learner updates $\boldsymbol{\theta}_{t+1}$ by minimizing the risk under the distribution induced by the previously deployed model $\mathcal{D}(\boldsymbol{\theta}_t)$. If this procedure converges, the iterates approach a PS solution. Examples of this include repeated risk minimization \citep{perdomo_performative_2021}, repeated gradient descent \citep{perdomo_performative_2021}, and stochastic gradient descent-greedy deploy (SGD-GD) \citep{Mendler2020Stochastic}. 

In practice, computing full-batch gradients at every deployment step can be computationally expensive, especially when the interacting population is large \citep{perdomo_performative_2021}. This motivates stochastic methods, which replace full-batch updates with stochastic gradient estimates computed from samples drawn from the current performative distribution. In this work, we focus on SGD-GD for two main reasons. First, unlike RGD~\citep{perdomo_performative_2021}, which requires computing a full gradient at each iteration, SGD-GD uses only a single stochastic gradient sample per update and is therefore more scalable to large populations. Second, unlike derivative-free methods~\citep{ray2022decision,liu2023two} and performative gradient descent~\citep{izzo2021performative}, SGD-GD does not require explicit knowledge of the distribution map $\mathcal{D}(\theta)$. In general, the SGD-GD update can be formalized as follows.
\begin{align}\label{eq:sgdgd}
    \boldsymbol{\theta}_{t+1} = \boldsymbol{\theta}_t - \gamma_{t+1} \nabla \ell(z_{t+1}; \boldsymbol{\theta}_t), \quad z_{t+1} \sim \mathcal{D}(\boldsymbol{\theta}_t).
\end{align}
Prior work has shown that Eqn.~\ref{eq:sgdgd} can converge to a unique PS solution \citep{Mendler2020Stochastic}, but  under two strong assumptions: 
(i) the loss function $\ell$ is strongly convex and smooth in $\boldsymbol{\theta}$; 
and (ii) the sensitivity of the distribution map $\mathcal{D}$, measured by the Wasserstein-1 distance, is bounded. 
However, the strong convexity requirement substantially limits the applicability of these results. 
\citet{li2024stochastic}  relaxed this assumption and established the first convergence guarantees for SGD-GD under smooth non-convex losses, showing convergence to a stationary performatively stable (SPS) solution (cf. Def.~\ref{def:sps}).

\begin{definition}[$\delta$-Stationary Performatively Stable Solution \citep{li2024stochastic}]\label{def:sps}
For a given \(\delta \geq 0\), the vector \(\boldsymbol{\theta}_{\delta\text{-SPS}} \in \mathbb{R}^d\) is said to be a \(\delta\)-stationary performatively stable (\(\delta\)-SPS) solution of Eqn.~\ref{eq:PS} if
\begin{align*}
\left\|\nabla \mathcal{J}(\boldsymbol{\theta}_{\delta\text{-SPS}}; \boldsymbol{\theta}_{\delta\text{-SPS}})\right\|^2 = \left\|\mathbb{E}_{z \sim \mathcal{D}(\boldsymbol{\theta}_{\delta\text{-SPS}})}\left[\nabla \ell(z; \boldsymbol{\theta}_{\delta\text{-SPS}})\right]\right\|^2 \leq \delta.
\end{align*}
\end{definition}
{where $\nabla \mathcal{J}(\theta_{\delta\text{-SPS}}, \theta_{\delta\text{-SPS}})$ denotes the gradient of the decoupled performative risk $\mathcal{J}(\cdot;\theta_{\delta\text{-SPS}})$ with respect to its first argument, evaluated at $\theta_{\delta\text{-SPS}}$, while the second argument (which determines the induced distribution) is held fixed.} The parameter $\delta$ in Def.~\ref{def:sps} measures the stability of a solution: 
{when $\delta = 0$, $\boldsymbol{\theta}_{\delta\text{-SPS}}$ is a stationary point of $\mathcal{J}(\cdot;\boldsymbol{\theta}_{\delta\text{-SPS}})$, i.e., $\nabla\mathcal{J}(\boldsymbol{\theta}_{\delta\text{-SPS}};\boldsymbol{\theta}_{\delta\text{-SPS}}) = 0$. Note that this first-order condition does not imply that $\boldsymbol{\theta}_{\delta\text{-SPS}}$ minimizes $\mathcal{J}(\cdot;\boldsymbol{\theta}_{\delta\text{-SPS}})$, and hence in general does not coincide with a performatively stable (PS) solution as defined in Eqn.~\ref{eq:PS}; the two coincide when $\mathcal{J}(\cdot;\boldsymbol{\theta}_{\delta\text{-SPS}})$ is convex.}
% \textcolor{red}{when $\delta = 0$, $\boldsymbol{\theta}_{\delta\text{-SPS}}$ denotes a stable point of the SGD-GD update.}
As shown by \citet{li2024stochastic}, SGD-GD can converge to $\boldsymbol{\theta}_{\delta\text{-SPS}}$ at a rate of $\mathcal{O}(1/\sqrt{T})$ for smooth non-convex loss function $\ell$. However, this result relies on an additional assumption on \textbf{bounded-variance of stochastic gradient estimates}. Moreover, the resulting convergence bound contains a {\em non-vanishing} error neighborhood that {\em scales with the variance}. This motivates the following question: 

\begin{tcolorbox}[left=1.2pt,right=1.2pt,top=1.2pt,bottom=1.2pt]
 %\begin{center}
\textbf{(Q)}: Can we control the variance of stochastic gradient estimates in PP to design an SGD-GD-type algorithm that converges to $\boldsymbol{\theta}_{\delta\text{-SPS}}$ with (i) a faster rate and (ii) an error neighborhood {\em independent} of the gradient variance? If yes, how?
%\end{center}
\end{tcolorbox}

In conventional (non-PP) stochastic optimization, variance reduction has proved effective for accelerating SGD and improving convergence guarantees. Representative methods include stochastic average gradient (SAG) \citep{roux2012stochastic}, stochastic average gradient augmented (SAGA) \citep{defazio2014saga}, stochastic variance reduced gradient (SVRG)  \citep{reddi2016stochastic}, and stochastic path-integrated differential estimator (SPIDER) \citep{fang2018spider}. 
However, whether such techniques can be adapted to PP remains largely \emph{unknown}. This paper addresses this gap by developing a variance-reduced SGD-GD-type framework for PP.

As a first step, we propose \ul{\textbf{s}}tochastic \ul{\textbf{p}}erformative p\ul{\textbf{r}}ediction w\ul{\textbf{i}}th varia\ul{\textbf{n}}ce reduc\ul{\textbf{t}}ion (\algns), a variance-reduced stochastic optimization framework for PP inspired by stochastic variance reduced gradient (SVRG) \citep{reddi2016stochastic}. In standard non-PP settings, SVRG periodically computes the full gradient at a past iterate (referred to as the snapshot point) and uses it to correct the stochastic gradient estimator in subsequent iterations. Specifically, SVRG combines the current stochastic gradient with a gradient-difference correction between the current iterate and the snapshot point, together with the stored full gradient at the snapshot point. This correction reduces the variance of stochastic updates relative to vanilla SGD, particularly when the current iterate remains close to the snapshot point.

Building on this idea, \alg divides the SGD-GD updates into multiple epochs. At the beginning of each epoch $s$, the algorithm stores a snapshot point $\widetilde{\boldsymbol{\theta}}^{s}$ and computes the corresponding full gradient of the decoupled performative risk $\nabla \mathcal{J}(\boldsymbol{\widetilde{\theta}}^s; \boldsymbol{\widetilde{\theta}}^s)$. During the inner iterations, this snapshot information is used to construct variance-reduced stochastic gradient estimators for updating the iterates $\boldsymbol{\theta}_k^{s}$. 

%Unlike standard SVRG, however, the changing performative distribution introduces bias into the corrected gradient estimator, since the stored snapshot and current iterate may correspond to different data distributions.

%As a first step in this direction, we propose \ul{\textbf{s}}tochastic \ul{\textbf{p}}erformative p\ul{\textbf{r}}ediction w\ul{\textbf{i}}th varia\ul{\textbf{n}}ce reduc\ul{\textbf{t}}ion (\algns) method, the first variance-reduced framework for PP, inspired by SVRG \citep{reddi2016stochastic}. \com{need more details here: e.g., how SVRG works in non-PP setting, why it reduces variance } The key idea is to divide the updating iterations of SGD-GD into multiple epochs and store a snapshot of the full gradient $\nabla \mathcal{J}$ at the end of each epoch. At each iteration of repeated optimization, the most recent snapshot can reduce the variance of the current stochastic gradient descent with a slight bias. \com{don't need discussion? you only store the snapshot of last gredient in each iteration? this is a good place to introduce $\boldsymbol{\theta}_k^{s}$, which you used directly in the next paragraph without definition }

We emphasize that \alg is \textbf{not} a straightforward extension of SVRG to PP setting. Unlike conventional stochastic optimization, the performative nature of the problem introduces two unique challenges for both algorithm design and theoretical analysis. First, because performative effects continuously modify the data distribution across iterations, the gradient of the decoupled performative risk also evolves over time. Consequently, the full-gradient snapshot computed at $\widetilde{\boldsymbol{\theta}}^{s}$ under a previous distribution becomes a biased estimate of the current gradient, even when evaluated at the same parameter value. Second, the evolving data distribution introduces additional analytical difficulties absent in standard SVRG analysis. In particular, the update from $\boldsymbol{\theta}_k^{s}$ to $\boldsymbol{\theta}_{k+1}^{s}$ induces extra distribution shifts, making it substantially more challenging to construct a suitable Lyapunov function and establish a negative Lyapunov drift for proving convergence.

%(i) The effects change the data distribution and the resulting decoupled performative loss gradient across iterations, thereby rendering the full gradient snapshot taken in the previous iteration a {\em biased} estimate of the current stochastic gradient evaluated at the same $\boldsymbol{\theta}$ on a different data distribution; and (ii) The effects introduce new difficulties, absent in the standard SVRG theoretical analysis, in constructing a suitable Lyapunov function and establishing a negative Lyapunov drift, since the change from $\boldsymbol{\theta}_k^{s}$ to $\boldsymbol{\theta}_{k+1}^{s}$ incurs extra distribution shifts unseen in the non-PP settings. 

The main contribution of this paper is to overcome the above technical challenges and establish convergence of the proposed \alg method to an SPS solution with a fast {\em variance-independent} convergence rate. Our main results and key contributions are summarized as follows:

\begin{list}{\labelitemi}{\leftmargin=1em \topsep=-0.03cm \itemsep=-0.06cm}

%\vspace{-.1in}

\item We propose the \alg~algorithm to reduce the variance of SGD-GD-based approaches in non-convex PP settings. Compared to existing methods, \alg~converges to an SPS solution at an {\em accelerated} rate of $\mathcal{O}(1/T)$ and, more importantly, achieves an error neighborhood that is \textbf{independent} of the stochastic gradient variance. To the best of our knowledge, these results are novel in the literature. %We propose the \alg algorithm to reduce the variance of the SGD-GD-Based approach in non-convex PP settings. Compared to existing works, \alg converges to an SPS solution at an {\em accelerated} rate of $\mathcal{O}(1/T)$, and, more importantly, features an error neighborhood that is \textbf{independent} of the stochastic gradient variance. To our knowledge, all these results are first in the literature.

%thereby accelerating convergence to an SPS solution.
    
\item To establish the $\mathcal{O}(1/T)$ convergence rate of \alg, we develop a series of new Lyapunov function construction techniques, which may be of independent interest in the PP literature. Moreover, we derive the incremental first-order oracle (IFO) complexity (i.e., the number of stochastic gradient evaluations required to achieve a $\mathcal{O}(\delta)$-accurate SPS solution \citep{reddi2016stochastic}).

%To establish the aforementioned $\mathcal{O}(1/T)$ convergence rate of our \alg method, we propose a series of new Lyapunov function construction techniques, which are of general and independent interests in the PP literature. Moreover, we derive the incremental first-order oracle (IFO) complexity (the number of operations of taking a sample and calculating its gradient to achieve a $\mathcal{O}(\delta)$-SPS \citep{reddi2016stochastic}) of the algorithm.
    
\item We conduct extensive numerical experiments to validate our theoretical results{, focusing on the SPS stationarity measure $\| \nabla\mathcal{J} (\boldsymbol{\widetilde{\theta}}^s; \boldsymbol{\widetilde{\theta}}^s)\|^2$ that Theorem~\ref{theorem:converge} directly bounds,} on four real-world datasets. In addition to credit data \citep{creditdata} and MNIST \citep{deng2012mnist}, which are commonly used in PP evaluations, we include additional experiments training MLP/CNN models using both SGD-GD-based and \alg methods on the full CIFAR-10 \citep{krizhevsky2009learning} dataset. These experiments demonstrate and further support the practical effectiveness of optimization-based approaches for PP.
%We conduct extensive numerical experiments to validate the theoretical results on three real-world datasets.  In addition to the credit data \citep{creditdata} and MNIST \citep{deng2012mnist}, we present additional experimental results for training MLP/CNN models using SGD-GD-Based and \alg approaches on the full CIFAR-10 \citep{krizhevsky2009learning} dataset. These broad experiments verify and enrich the practical applicability of optimization-based approaches for PP.

\end{list}

The remainder of this paper is organized as follows: Sec. \ref{sec:related} reviews related work and Sec.~\ref{subsec:prelim} presents prior results. We then introduce the \alg algorithm in Sec.~\ref{subsec:algo} and establish its convergence and IFO complexity in Sec.~\ref{sec:main}. 
Finally, we present numerical results in Sec.~\ref{sec:exp} and conclude the paper in Sec.~\ref{sec:conclusion}.

\section{Related work}\label{sec:related}

We review four closely related lines of work: (i) algorithms for finding PS/SPS solutions in PP; (ii) alternative solution concepts and metrics in PP; (iii) variance-reduction techniques for SGD-based optimization; (iv) connections between PP, reinforcement learning, and optimal control.%In this section, we provide a brief overview on three research areas that are closely related to our work: 1) finding SPS solutions for PP; 2) other related solution metrics in PP; and 3) variance reduction (VR) techniques for SGD-based optimization methods.

\textbf{1) Finding PS/SPS solutions for PP.} Performative prediction (PP) was first introduced by \citet{perdomo_performative_2021} as an optimization framework for learning under endogenous data distribution shifts. They proposed \textit{repeated risk minimization} (RRM), an iterative optimization procedure for finding performatively stable (PS) solutions, and established bounds on the distance between PS and performatively optimal (PO) solutions. They also studied repeated gradient descent (RGD), which updates the model by taking full-gradient steps under the currently induced distribution. However, RGD requires access to a full gradient at every iteration, which can be computationally expensive in large-scale settings. To improve scalability, \citet{Mendler2020Stochastic} developed the first stochastic algorithm for finding PS solutions in an online setting. Later, \citet{mendler2022anticipating} established convergence rates for greedy and lazy deployment schemes under smoothness and strong convexity assumptions. Other works have extended PP to more structured learning settings. For example, \citet{mofakhami2023performative} studied neural-network training in PP settings, assuming a predictor of the form $\hat{y}=f_{\boldsymbol{\theta}}(x)$, a loss $\ell(\hat{y},y)$ that is strongly convex in $\hat{y}$, and a distribution map $D(\hat{y})$ that is $\epsilon$-sensitive with respect to predictions. From a complexity-theoretic perspective, \citet{anagnostides2026computationalcomplexityperformativeprediction} showed that finding a performatively stable point is PPAD-complete even for Lipschitz losses, highlighting the need for efficient stochastic methods that converge to approximate SPS solutions.

Most closely related to our work, \citet{li2024stochastic} provided the first convergence guarantee for PP with smooth non-convex losses without imposing a particular solution structure. Their result shows that SGD-GD converges to a stationary performatively stable (SPS) solution. However, the guarantee relies on a bounded-variance assumption on stochastic gradient estimates and yields a non-vanishing error neighborhood that scales with the variance. Our work overcomes this limitation by introducing a variance-reduced SGD-GD-type method that accelerates convergence while removing variance-dependent residual error.

\textbf{2) Alternative solution concepts and metrics in PP.} Beyond PS and SPS solutions, several works have studied alternative solution concepts for PP. One line of work directly optimizes the performative risk (PR) to find PO solutions. For instance, \citet{miller2021outside} studied direct optimization of PR under restricted classes of distribution maps, while \citet{ray2022decision,liu2023two} used derivative-free optimization methods to find PO solutions. \citet{izzo2021performative} proposed performative gradient descent for finding PO solutions under a convex PR assumption; however, convexity of the performative risk can be difficult to verify, even when the loss itself is strongly convex. Other works have studied PP under weak convexity assumptions \citep{zhu2023online,zheng2024profl} or in state-dependent settings, where algorithms are designed to converge to PS solutions \citep{brown2020performative,li2022state}.

In addition to optimization-based solution concepts, recent work has examined broader statistical and societal impacts of performative effects. For example, \citet{prm} analyzed the impact of performative risk minimization on binary random variables, showing that PRM can induce non-vanishing prediction bias and distributional mean shift. Other studies have investigated performative prediction through the lens of fairness \citep{jin2024addressingpolarizationunfairnessperformative}, social welfare \citep{kim2022making}, and privacy \citep{li2024clipped}. These works highlight the broader importance of understanding model-induced distribution shifts, while our work focuses specifically on improving stochastic optimization guarantees for finding SPS solutions.

%It is worth noting that, besides finding SPS solutions, there also exist other solution metrics for PP. \citet{miller2021outside} tackled PP problems by directly optimizing performance risk (PR) to find the PO solution for a restricted set of the distribution maps. \citet{ray2022decision, liu2023two} utilized derivative-free optimization to find PO solutions.  \citet{izzo2021performative} also designed an algorithm based on performative gradient descent for finding PO solutions under a convex PR assumption, which, however, is hard to verify even with a strongly convex loss function. \citet{zhu2023online,zheng2024profl} also studied PP under weakly convex conditions. \citet{brown2020performative, li2022state} focused on state-dependent PP settings and proposed algorithms that converge to PS solutions. \citet{prm} analyze the broader impact of PRM on binary random variables, showing that PRM can introduce non-vanishing prediction bias and distribution mean shift as side effects of performative optimization. There also exist other works on PP that focused on fairness\citep{jin2024addressingpolarizationunfairnessperformative}, social welfare \citep{kim2022making} and privacy \citep{li2024clipped} metrics.

\textbf{3) Variance-reduction techniques for SGD-based optimization.} Variance reduction (VR) has been extensively studied in conventional stochastic optimization as a means of accelerating SGD and improving convergence guarantees \citep{roux2012stochastic,defazio2014saga,reddi2016stochastic,nguyen2017sarah,fang2018spider}. Representative methods include stochastic average gradient (SAG) \citep{roux2012stochastic} and stochastic average gradient augmented (SAGA) \citep{defazio2014saga}, which maintain and update a table of per-sample gradients to reduce the variance of stochastic updates. While effective, these methods can incur high memory costs in large-scale regimes.

Stochastic variance-reduced gradient (SVRG) \citep{reddi2016stochastic} addresses this issue by avoiding the storage of per-sample gradients. Instead, SVRG periodically computes a full gradient at a snapshot point and uses it to construct variance-reduced stochastic updates throughout an epoch. Under strong convexity, SVRG achieves linear convergence with modest memory requirements. Subsequent methods, including stochastic recursive gradient methods such as SARAH \citep{nguyen2017sarah} and stochastic path-integrated differential estimator (SPIDER) \citep{fang2018spider}, further improve variance reduction by using recursive gradient estimators based on fresher iterates, achieving optimal convergence rates in terms of stationary gap and dataset-size dependence.

However, existing VR techniques are designed for standard stochastic optimization under fixed data distributions. They do not directly apply to PP, where the data distribution changes as the model is updated. This model-induced distribution shift makes it nontrivial to construct unbiased or low-variance gradient estimators and to establish convergence guarantees. To the best of our knowledge, variance-reduced stochastic optimization for performative prediction remains largely underexplored. Our work addresses this gap by developing a variance-reduced SGD-GD-type framework tailored to PP.

%In the literature, variance reduction (VR) techniques have been developed to accelerate the convergence of the SGD method \citep{roux2012stochastic, reddi2016stochastic, nguyen2017sarah, defazio2014saga, fang2018spider}. For example, stochastic average gradient (SAG) \citep{roux2012stochastic} and stochastic average gradient augmented (SAGA) \citep{defazio2014saga} maintain an average of the stochastic gradients of all data samples with infrequent updates. However, these approaches could incur high memory costs in the large dataset regime. In contrast, the stochastic variance-reduced gradient (SVRG) \citep{reddi2016stochastic} eliminates the need for storing per-sample stochastic gradients by periodically computing a full gradient at the beginning of each epoch and using it to construct a variance-reduced update. SVRG enjoys a linear convergence rate with a modest memory requirement under the strong convexity setting. Subsequently, the stochastic recursive gradient (SARAH) \citep{nguyen2017sarah} and the stochastic path-integrated differential estimator (SPIDER)\citep{fang2018spider} methods further improve SVRG by incorporating fresher recursive iterates to achieve optimal convergence rate in terms of both stationary gap and dataset size dependencies. We note, however, that all these VR techniques were only designed for non-PP settings.  To date, the development of VR techniques for PP remains largely underexplored in the literature.

\textbf{4) Connections between PP, reinforcement learning, and optimal control.} Performative prediction is also related to reinforcement learning (RL) and optimal control (OC), as all three frameworks involve some form of decision-dependent data or state evolution. In RL \citep{kaelbling1996reinforcement}, a policy $\pi$ is optimized to maximize an expected reward objective of the form $J(\pi)=\mathbb{E}_{\tau\sim p_\pi}[R(\tau)]$, where the trajectory distribution $p_\pi$ depends on the policy. However, RL is typically reward-driven and focuses on sequential decision-making, rather than on finding an equilibrium between a deployed model and the distribution it induces. For convergence analyses of various optimization algorithms in RL, we refer readers to \citet{xiao2022convergence,beggs2005convergence,yuan2022general}.

Optimal control \citep{todorov2006optimal} studies dynamic systems governed by temporal state transitions, such as $x_{t+1}=f(x_t,u_t)$, and seeks to minimize cumulative costs of the form $\sum_t c(x_t,u_t)$. The dependency in OC is primarily temporal and state-based, whereas PP focuses on model-dependent distribution shift: the deployed model changes the data distribution on which future learning is performed. Thus, while PP shares conceptual connections with RL and OC, its equilibrium-based optimization structure and model-induced distribution shifts create distinct algorithmic and analytical challenges.

%While Reinforcement Learning (RL) \citep{kaelbling1996reinforcement} optimizes the reward brought by policy $\pi$ $J(\pi) = \mathbb{E}_{\tau \sim p_\pi}[R(\tau)]$ directly under a distribution $p_\pi$, it is reward-based and does not seek an explicit equilibrium between the policy and the induced distribution. For the convergence of various optimization algorithms in RL settings, we refer to \citet{xiao2022convergence, beggs2005convergence, agarwal2015consistent, yuan2022general}. For Optimal Control (OC) \citep{todorov2006optimal}, it studies dynamic state transitions $x_{t+1}=f(x_t,u_t)$ and minimizes cumulative cost $\sum_t c(x_t,u_t)$. The dependency is primarily temporal, while in PP the dependency is model-dependent.
%\section{Problem Formulation}

\section{Preliminaries}\label{subsec:prelim}

% Following the convention in the variance reduction literature \citep{reddi2016stochastic, roux2012stochastic}, we assume the full population consists of
%  $n$ samples $\{z_1,...,z_n\} \subset \mathcal{Z}$, \footnote{This finite-population setting is standard in the variance reduction literature  \citep{ roux2012stochastic, reddi2016stochastic, nguyen2017sarah} and is natural in PP, where typical applications such as credit scoring \citep{perdomo_performative_2021} and college admission \citep{xie2024algorithmic}) often involve a finite population. }
%  and that the ML model is parameterized by $\boldsymbol{\theta} \in \Theta \subseteq \mathbb{R}^d$. The expected gradient of the decoupled performative risk over the full population is $$\nabla \mathcal{J}(\boldsymbol{\theta}; \boldsymbol{\theta}) = \mathbb{E}_{z \sim \mathcal{D}(\boldsymbol{\theta})}[\nabla \ell(z;\boldsymbol{\theta})] = \frac{1}{n}\sum_{i=1}^{n}\nabla \ell(z_i;\boldsymbol{\theta}),$$ where $\nabla \ell(z;\boldsymbol{\theta})$ denotes the gradient of $\ell(z;\boldsymbol{\theta})$ with respect to $\boldsymbol{\theta}$. 

 Following the convention in the variance reduction literature \citep{reddi2016stochastic, roux2012stochastic}, {we consider a fixed population of $n$ units indexed by $i \in [n]$. Upon deployment of $\theta$, unit $i$ produces an observation $z_i(\theta) \in \mathcal{Z}$, i.e., $z_i:\Theta \rightarrow \mathcal{Z}$\footnote{This finite-population setting is standard in the variance reduction literature  \citep{ roux2012stochastic, reddi2016stochastic, nguyen2017sarah} and is natural in PP, where typical applications such as credit scoring \citep{perdomo_performative_2021} and college admission \citep{xie2024algorithmic} often involve a finite population. }. This induces the empirical distribution $\mathcal{D} (\theta)\coloneqq \frac{1}{n} \sum_{i=1}^n \delta_{z_i(\theta)}$. The population identities $\{1,\ldots,n\}$ are fixed throughout, while the realized data $\{z_i(\theta)\}_{i=1}^n$ change as $\theta$ changes. The  decoupled performative risk gradient is therefore} 
 $${\nabla \mathcal{J}(\boldsymbol{\theta}; \boldsymbol{\theta}') = \mathbb{E}_{z \sim \mathcal{D}(\boldsymbol{\theta}')}[\nabla \ell(z;\boldsymbol{\theta})] := \frac{1}{n}\sum_{i=1}^{n}\nabla \ell(z_i(\boldsymbol{\theta} ');\boldsymbol{\theta})}$$ where $\nabla \ell(z;\boldsymbol{\theta})$ denotes the gradient of $\ell(z;\boldsymbol{\theta})$ with respect to $\boldsymbol{\theta}$. 
 
 We adopt the \textit{non-convex} problem setup of \citet{li2024stochastic}, in contrast to most prior PP work that assumes loss function $\ell$ is strongly convex \citep{Mendler2020Stochastic, izzo2021performative, mofakhami2023performative}. To set up the background, we first present the technical assumptions and then the convergence result of \citet{li2024stochastic}. Throughout the paper, we use $\|\cdot\|$ to denote $\ell_2$ norm.  
 
 %Unlike most prior works that studied the convergence of SGD-GD (Eqn. \ref{eq:sgdgd}) in performative prediction (PP) under the assumption that the loss function $\ell$ is strongly convex \citep{Mendler2020Stochastic, izzo2021performative, mofakhami2023performative}, we consider {\em non-convex} $\ell$ and focus on the problem setup studied by \citet{li2024stochastic}, a recent work that showed the convergence of SGD-GD to an SPS solution in non-convex PP settings. To familiarize the readers with necessary background and facilitate the subsequent discussions, we first present the technical assumptions and then the theoretical results of \citep{li2024stochastic}.  Throughout the paper, we use $\|\cdot\|$ to denote $\ell_2$ norm.  

\begin{assumption}[Smoothness and lower bound of the loss]\label{assumption: smoothness} There exist constants $L, L_0 \geq 0$ such that %for any $z,z' \in \mathcal{Z}$ and any $\boldsymbol{\theta}, \boldsymbol{\theta'} \in \mathbb{R}^d $, %there exists a constant $L, L_0 \geq 0$ such that
\begin{align*}
    \|\nabla \ell(z; \boldsymbol{\theta}) - \nabla \ell(z;\boldsymbol{\theta'})\| &\leq L \|\boldsymbol{\theta} - \boldsymbol{\theta'}\|, \quad \forall \boldsymbol{\theta}, \boldsymbol{\theta'} \in \mathbb{R}^d \\
    \|\nabla \ell(z; \boldsymbol{\theta}) - \nabla \ell(z';\boldsymbol{\theta})\| &\leq L_0 \|z - z'\|, \quad \forall z, z' \in \mathcal{Z}
\end{align*}
Moreover, there exists a constant $\ell^\star > -\infty$ such that $\ell(z;\boldsymbol{\theta}) \ge \ell^\star$ for all $\boldsymbol{\theta}, z$.
\end{assumption}

\begin{definition}[$\epsilon$-sensitivity]\label{assumption:sensitivity}
 Distribution map $\mathcal{D}(\theta)$ is $\epsilon$-sensitive if there exists $\epsilon \geq 0$ such that 
\begin{align*}
    W_1(\mathcal{D}(\boldsymbol{\theta}), \mathcal{D}(\boldsymbol{\theta'})) \leq \epsilon \|\boldsymbol{\theta} - \boldsymbol{\theta'}\|, ~\forall \boldsymbol{\theta}, \boldsymbol{\theta'} \in \mathbb{R}^d
\end{align*}
where $W_1(\cdot, \cdot)$ is the Wasserstein-1 distance. 
\end{definition}

\begin{assumption}[Lipschitzness of the loss]\label{assumption:lipschitzness} The loss function $\ell$ is $L_1$-Lipschitz in  $z$:
\begin{align*}
   |\ell(z;\boldsymbol{\theta}) - \ell(z';\boldsymbol{\theta})| \leq L_1  \|z - z'\|, ~~~\forall z, z' \in \mathcal{Z}, \boldsymbol{\theta}\in \mathbb{R}^d.
\end{align*} 
\end{assumption}

\begin{assumption}[Variance of stochastic gradient estimates]\label{assumption:variance}
There exist constants $\sigma_0, \sigma_1 \ge 0$, such that
\begin{equation*} {\mathbb{E}}_{Z \sim \mathcal{D}(\boldsymbol{\theta}_2)} \left\| \nabla \ell(Z; \boldsymbol{\theta}_1) - \nabla \mathcal{J}(\boldsymbol{\theta}_1; \boldsymbol{\theta}_2) \right\|^2 \leq \sigma_0^2 + \sigma_1^2 \left\| \nabla \mathcal{J}(\boldsymbol{\theta}_1; \boldsymbol{\theta}_2) \right\|^2
\end{equation*}
\end{assumption}

% \begin{assumption}[Variance of stochastic gradient estimates]\label{assumption:variance}
% The stochastic gradient is unbiased, i.e., $\nabla \mathcal{J}(\boldsymbol{\theta_1}; \boldsymbol{\theta_2}) = \mathbb{E}_{Z \sim \mathcal{D}(\boldsymbol{\theta_2})}[\nabla \ell(Z;\boldsymbol{\theta_1})]$. 
% Moreover, there exist constants $\sigma_0, \sigma_1 \ge 0$, such that
% \begin{equation*} {\mathbb{E}}_{Z \sim \mathcal{D}(\boldsymbol{\theta}_2)} \left\| \nabla \ell(Z; \boldsymbol{\theta}_1) - \nabla \mathcal{J}(\boldsymbol{\theta}_1; \boldsymbol{\theta}_2) \right\|^2 \leq \sigma_0^2 + \sigma_1^2 \left\| \nabla \mathcal{J}(\boldsymbol{\theta}_1; \boldsymbol{\theta}_2) \right\|^2
% \end{equation*}
% \end{assumption}

Note that smoothness (Assumption \ref{assumption: smoothness}) and $\epsilon$-sensitivity (Def.~\ref{assumption:sensitivity}) are standard in the PP optimization literature \citep{perdomo_performative_2021,Mendler2020Stochastic}, while Lipschitzness of loss (Assumption \ref{assumption:lipschitzness}) and the bounded-variance condition (Assumption~\ref{assumption:variance}) are specific to non-convex PP but remain common in optimization. Note that Assumption~\ref{assumption:lipschitzness} is only about Lipschitzness with respect to $z$ and does not imply the bounded-variance assumption.

Under the above assumptions, \citet{li2024stochastic} proved the following convergence result:

\begin{proposition}[Convergence of SGD-GD \citep{li2024stochastic}]\label{proposition:converge}
{Under Assumptions~\ref{assumption: smoothness},~\ref{assumption:lipschitzness},~\ref{assumption:variance} and Definition~\ref{assumption:sensitivity}}, for the SGD-GD iterates given in Eqn.~\ref{eq:sgdgd}, the following holds for any $T > 1$:
\begin{equation}
\label{eq:liconverge1}
    \sum_{t=0}^{T-1} \frac{\gamma_{t+1}}{4} \mathbb{E} \left[ \|\nabla \mathcal{J}(\boldsymbol{\theta}_t; \boldsymbol{\theta}_t)\|^2 \right] \leq \Delta_0 + \frac{L}{2} \sigma_0^2 \sum_{t=0}^{T-1} \gamma_{t+1}^2 
    + L_1\epsilon \left(\sigma_0 + (1 + \sigma_1^2) L_1 \epsilon\right) \sum_{t=0}^{T-1} \gamma_{t+1},
\end{equation}
where $\gamma_{t}$ is the learning rate at iteration $t$ and  $\Delta_0 = \mathcal{J}(\boldsymbol{\theta_0}; \boldsymbol{\theta_0}) - \ell_{*}$ is an upper bound of the initial optimality gap between the initial performative risk and a lower bound on the performative risk $\ell_{*}$. {Assuming $\tau$ is sampled uniformly at random from $\{0,\ldots,T-1\}$, independent of the algorithmic randomness.}
If one sets $\gamma_t = 1/\sqrt{T}$ for all $t$, the bound in Eqn.~\ref{eq:liconverge1} can be reduced to: 
\begin{equation}\label{eq:liconverge2}
\mathbb{E} \left[ \left\| \nabla \mathcal{J}(\boldsymbol{\theta}_\tau; \boldsymbol{\theta}_\tau) \right\|^2 \right] 
\leq 4 \underbrace{ \left( \Delta_0 + \frac{L}{2} \sigma_0^2 \right) }_{\mathrm{Variance-Dependent}} \frac{1}{\sqrt{T}}
 + 4 L_1\epsilon \underbrace{ \left( \sigma_0 + (1 + \sigma_1^2) L_1 \epsilon \right)  }_{\mathrm{Variance-Dependent}}.
\end{equation}
\end{proposition}
Proposition~\ref{proposition:converge} shows that SGD-GD converges at a rate of $\mathcal{O}(1/\sqrt{T})$ to an error neighborhood of size $\mathcal{O}(4L_1\epsilon(\sigma_0 + (1+\sigma_1^2)L_1\epsilon)$ \footnote{This rate also holds in the finite-population setting.}. However, this error neighborhood is directly {\em influenced by the variance} of the stochastic gradient estimate. Additionally, the convergence rate is affected by the second %third 
term on the right-hand side of Eqn.~\ref{eq:liconverge1}, which is also {\em variance-dependent}. These limitations motivate us to propose a new algorithm for \textbf{non-convex} PP that achieves faster convergence to SPS solutions and to develop new variance reduction techniques that enable the SGD-GD-based PP approach to be \textbf{variance-independent}.

\begin{table*}[t]
    \centering
\resizebox{\textwidth}{!}{
\begin{tabular}{cccccc}%{p{2.5cm}p{1.5cm}p{1.4cm}p{2.7cm}p{1.8cm}p{3.7cm}}
    \toprule
    \textbf{Literature} & \textbf{Loss function $\ell$} & \textbf{Convergence rate} & \textbf{Error neighborhood} & \textbf{Comments}\\ \midrule
    \citet{Mendler2020Stochastic} & Strongly convex &  $\mathcal{O}(\frac{1}{T})$ & None & No error due to strong convexity \\ 
    \citet{li2024stochastic} &  Nonconvex & $\mathcal{O}(\frac{1}{\sqrt{T}})$ & $\mathcal{O}(\sigma_0 \epsilon + \sigma_1^2\epsilon^2)$ & Variance-dependent error neighborhood \\ %\hline
 This work: \textbf{\alg} &  Nonconvex & $\mathcal{O}(\frac{1}{T})$ & $\mathcal{O}(\epsilon^2 + \epsilon^4)$ & \textbf{Variance-independent} error neighborhood\\ \bottomrule
\end{tabular}}
    \caption{\alg compared to prior SGD-GD-based methods for PP. Since our goal is to improve existing SGD-GD-based PP methods, we do not include algorithms targeting PO solutions  \citep{izzo2021performative, miller2021outside} or those based on repeated risk minimization (RRM) \citep{perdomo_performative_2021, mofakhami2023performative}; we refer readers to Table~1 of \citet{li2024stochastic} for a detailed comparison with these methods.}
    \label{tab:comparision}
\end{table*}

\section{The proposed \alg algorithm}\label{subsec:algo}

To address the limitations of Proposition~\ref{proposition:converge}, we propose \alg, a new algorithm for non-convex PP inspired by stochastic variance reduced gradient (SVRG) \citep{reddi2016stochastic}. As shown in Alg.~\ref{alg:one}, the total $T$ update iterations are divided into $S$ epochs, each consisting of $m$ iterations. At the beginning of each epoch $s+1$, the algorithm sets a snapshot point $\boldsymbol{\widetilde{\theta}}^s$, {observes all $n$ population units under deployment $\boldsymbol{\widetilde{\theta}}^s$, and} computes the corresponding full gradient of the decoupled performative risk, $\nabla\mathcal{J}(\boldsymbol{\widetilde{\theta}}^s; \boldsymbol{\widetilde{\theta}}^s)$ (Line 4).
%At the end of each epoch $s+1$, the algorithm stores a snapshot of the expected full gradient $\mathcal{J}(\boldsymbol{\widetilde{\theta}}^s; \boldsymbol{\widetilde{\theta}}^s)$ (Line 4). 
During each inner iteration $k$ of the epoch, the algorithm {samples only an index $i_k$ uniformly from $[n]$ and observes that unit under the current deployment $z_{i_k}(\boldsymbol{\theta}_k^{s+1})$ (Line 6), which is used to} construct a variance-reduced estimator $\boldsymbol{v}_k^{s+1}$ (Line 7) by subtracting the control variate $\nabla \ell(z_{i_k}; \widetilde{\boldsymbol{\theta}}^{s})$ from the current stochastic gradient $\nabla \ell(z_{i_k}; \boldsymbol{\theta}_k^{s+1})$  and adding back the full gradient $\nabla\mathcal{J}(\boldsymbol{\widetilde{\theta}}^s; \boldsymbol{\widetilde{\theta}}^s)$ at the snapshot.

\begin{minipage}{0.301\textwidth}
    Note that, unlike variance reduction in the non-PP setting, $\boldsymbol{v}^{s+1}_k$ is {\em not} an unbiased estimator of  $\nabla \mathcal{J}(\boldsymbol{\theta}^{s+1}_k; \boldsymbol{\theta}^{s+1}_k)= \mathbb{E}_{z \sim \mathcal{D}(\boldsymbol{\theta}^{s+1}_k)}[\nabla \ell(z; \boldsymbol{\theta}^{s+1}_k)] $ in PP setting. This is because the snapshot gradient is computed under $\mathcal{D}(\widetilde{\boldsymbol{\theta}}^{s})$, whereas the current stochastic gradient is sampled from $\mathcal{D}(\boldsymbol{\theta}^{s+1}_k)$, and generally $\mathcal{D}(\widetilde{\boldsymbol{\theta}}^{s})\neq \mathcal{D}(\boldsymbol{\theta}^{s+1}_k)$.
This {\bf new bias}, absent from the non-PP setting, significantly complicates the convergence analysis, as detailed in Sec.~\ref{sec:main}. Table~\ref{tab:comparision} compares the proposed \alg with existing SGD-GD-based PP methods.
\end{minipage}
\hfill 
\begin{minipage}{0.65\textwidth}
\begin{algorithm}[H]
\caption{Proposed Algorithm: Stochastic Performative Prediction with Variance Reduction (\algns)}
\label{alg:one}
\begin{algorithmic}[1]
\REQUIRE Number of total rounds $T$, number of iterations $m$ each epoch, learning rates $\gamma_0,\cdots,\gamma_{m-1}$, initialization $\widetilde{\boldsymbol{\theta}}^0 = \boldsymbol{\theta}_m^0 = \boldsymbol{\theta}_0$.
\STATE $S =  \frac{T}{m} $\footnote{We assume $T$ is divisible by $m$, i.e., $T = Sm$}
    \FOR{$s = 0$ to $S-1$}
        \STATE $\boldsymbol{\theta}_0^{s+1} = \boldsymbol{\theta}_m^s = \widetilde{\boldsymbol{\theta}}^s$
        \STATE {$\nabla \mathcal{J}(\widetilde{\boldsymbol{\theta}}^s; \widetilde{\boldsymbol{\theta}}^s) = \frac{1}{n}\sum_{i=1}^{n}[\nabla\ell(z_i(\widetilde{\boldsymbol{\theta}}^s); \widetilde{\boldsymbol{\theta}}^s)]$} 
        % \textcolor{blue}{each $z_i \in \mathcal{D}(\widetilde{\boldsymbol{\theta}}^s)$}
        \FOR{$k = 0$ to $m-1$}
            \STATE {Sample index $i_k \sim \mathrm{Unif}([n])$ and observe $z_{i_k}(\boldsymbol{\theta}_k^{s+1})$}
            % \STATE Pick a sample $z_{i_k}$ where \textcolor{blue}{$z_{i_k} \in \mathcal{D}(\boldsymbol{\theta}_k^{s+1})$}
            \STATE {$\boldsymbol{v}_k^{s+1} = \nabla\ell(z_{i_k}(\boldsymbol{\theta}_k^{s+1}); \boldsymbol{ \theta}_k^{s+1}) - \nabla\ell(z_{i_k}(\boldsymbol{\theta}_k^{s+1}); \widetilde{\boldsymbol{\theta}}^s) + \nabla \mathcal{J}(\widetilde{\boldsymbol{\theta}}^s; \widetilde{\boldsymbol{\theta}}^s)$}% \STATE $\boldsymbol{v}_k^{s+1} = \nabla\ell(z_{i_k}; \boldsymbol{ \theta}_k^{s+1}) - \nabla\ell(z_{i_k}; \widetilde{\boldsymbol{\theta}}^s) + \nabla \mathcal{J}(\widetilde{\boldsymbol{\theta}}^s; \widetilde{\boldsymbol{\theta}}^s)$
            \STATE ${\boldsymbol{\theta}}^{s+1}_{k+1} = {\boldsymbol{\theta}}^{s+1}_{k} - \gamma_k \boldsymbol{v}_k^{s+1}$
        \ENDFOR
        \STATE $\widetilde{\boldsymbol{\theta}}^{s+1} = \boldsymbol{\theta}_m^{s+1}$
\ENDFOR
% \OUTPUT $\boldsymbol{\theta}_m^S$
\OUTPUT {$\boldsymbol{\theta}_\tau$, where $\theta_\tau \sim \mathrm{Unif}(\{\boldsymbol{\theta}_k^{s+1}\}_{s=0,\ldots,S-1;k=0,\ldots,m-1})$ }
\end{algorithmic}
\end{algorithm}
\end{minipage}

Notably, only \citet{li2024stochastic} and our work address the non-convex PP setting. Among the two, \alg achieves a faster convergence rate and yields an error neighborhood  independent of the stochastic-gradient variance, as established in Sec.~\ref{sec:main}. 
%the variance reduction step computes $\boldsymbol{v}_k^{s+1}$ (Line 8), where the variance of the current stochastic gradient estimate $\nabla \ell(z_{ik}; \boldsymbol{\theta}_k^{s+1})$ is offset by the variance of the previous stochastic gradient estimate $\nabla \ell(z_{ik}; \widetilde{\boldsymbol{\theta}}^{s})$. Meanwhile, this step adds the expected gradient $\nabla \mathcal{J}(\widetilde{\boldsymbol{\theta}}^s; \widetilde{\boldsymbol{\theta}}^s)$ back to ensure that the overall expectation remains stable. 

%We characterize the convergence results of Alg. \ref{alg:one} in Sec. \ref{sec:main}.

\section{Convergence analysis of \alg }\label{sec:main}

In this section, we analyze the convergence of \alg.  

\paragraph{Assumptions.} Our analysis relies only on smoothness of the loss (Assumption~\ref{assumption: smoothness}), $\epsilon$-sensitivity of the distribution map (Def.~\ref{assumption:sensitivity}), and Lipschitzness of the loss (Assumption~\ref{assumption:lipschitzness}). Crucially, unlike \citet{li2024stochastic}, we \textbf{do not} require the variance of the stochastic gradient estimate to be bounded (Assumption~\ref{assumption:variance}). This contrasts directly with Proposition~\ref{proposition:converge}, which requires Assumption~\ref{assumption:variance} to hold.

%\textbf{1) Assumptions:} Our theoretical analysis only relies on the standard smoothness of the loss function (Assumption~\ref{assumption: smoothness}), the $\epsilon$-sensitivity of the distribution map (Def.~\ref{assumption:sensitivity}), and the Lipschitzness of loss function (Assumption \ref{assumption:lipschitzness}). Most notably, unlike \citet{li2024stochastic}, we \textbf{do not} assume bounded variance of the stochastic gradient estimate (Assumption~\ref{assumption:variance}).

%\textbf{2) Overview of the Proofs of the Main Theoretical Results:} 

\paragraph{Overview of the proof.} Following \citet{reddi2016stochastic} for SVRG in non-PP setting, we establish convergence of \alg by constructing a suitable \textit{Lyapunov function} and showing that the corresponding function sequence decreases over iterations. However, the analysis of \citep{reddi2016stochastic} does not extend directly to PP: in non-PP, loss $\ell(z;\boldsymbol{\theta})$ depends only on the current parameter $\boldsymbol{\theta}$, whereas decoupled performative risk $\mathcal{J}(\boldsymbol{\theta}; \boldsymbol{\theta}')=\mathbb{E}_{z \sim \mathcal{D}(\boldsymbol{\theta}')}[\ell(z;\boldsymbol{\theta})]$ depends on two arguments: $\boldsymbol{\theta}$, which determines the model prediction, and $\boldsymbol{\theta}'$, which controls the data distribution. Because each parameter update simultaneously shifts the data distribution, bounding consecutive Lyapunov terms is much harder and calls for new analysis techniques.

%Motivated by the analysis in \citep{reddi2016stochastic}, which solely focused on SVRG in the non-PP settings, we aim to establish the convergence of \alg by constructing a suitable {\em Lyapunov function} and showing that the corresponding function sequence decreases over iterations. However, it turns out that one {\em cannot} trivially extend   \citep{reddi2016stochastic} because the loss $\ell(\boldsymbol{\theta}; z)$ in the non-PP setting solely depends on the current model parameter $\boldsymbol{\theta}$.In stark contrast, the decoupled performative risk $\mathcal{J}(\boldsymbol{\theta}; \boldsymbol{\theta}')=\mathbb{E}_{z \sim \mathcal{D}(\boldsymbol{\theta}')}[\ell(\boldsymbol{\theta}; z)]$ in the PP setting depends on {\em two variables}: 1) $\boldsymbol{\theta}$, which determines the model prediction, and 2) $\boldsymbol{\theta}'$, which controls the data distribution. Because updates to the model parameter $\boldsymbol{\theta}$ simultaneously affect the data distribution, bounding the difference between consecutive terms in the Lyapunov sequence becomes significantly more challenging, which necessitates new proof and analysis techniques.

We overcome these challenges and establish the SPS convergence of \alg in three steps: (i) \textbf{Construct an intermediate function sequence} $R_k^{s+1}(\boldsymbol{\theta})$ (Lemma~\ref{lemma:r}), in which consecutive terms differ only in the \textit{second} argument of decoupled performative risk $\mathcal{J}$ and show that $R_k^{s+1} (\boldsymbol{\theta}_k^{s+1})$ is smaller than $R_{k+1}^{s+1} (\boldsymbol{\theta}_k^{s+1})$; (ii) \textbf{Construct the final Lyapunov function} $\widetilde{R}_k^{s+1}$ (Lemma \ref{lemma:r2}) in which consecutive terms differ in \textit{both} arguments of $\mathcal{J}$; (iii) \textbf{Establish convergence} by proving that $\widetilde{R}_k^{s+1}$ is decreasing, leading to the main convergence result (Theorem~\ref{theorem:converge}) and the IFO complexity bound (Corollary~\ref{corollary:complexity}). 

\textbf{Step 1: Construct an intermediate function sequence.} For the $(s+1)$-th epoch, we identify an intermediate function evaluated at the iterations $k$ and $k+1$ as:
\begin{align*}
R_k^{s+1} (\boldsymbol{\theta}) \overset{\triangle}{=} \mathbb{E}\left[\mathcal{J}(\boldsymbol{\theta}_k^{s+1}; \boldsymbol{\theta}) + c_k \| \boldsymbol{\theta}_k^{s+1} - \widetilde{\boldsymbol{\theta}}^s\|^2\right], \quad
R_{k+1}^{s+1} (\boldsymbol{\theta}) \overset{\triangle}{=} \mathbb{E}\left[\mathcal{J}(\boldsymbol{\theta}_{k+1}^{s+1}; \boldsymbol{\theta}) + c_{k+1} \| \boldsymbol{\theta}_{k+1}^{s+1} - \widetilde{\boldsymbol{\theta}}^s\|^2\right]. 
\end{align*}

Then, the following Lemma~\ref{lemma:r} characterizes the relation between $R_k^{s+1} (\boldsymbol{\theta}^{s+1}_k) $ and $R_{k+1}^{s+1} (\boldsymbol{\theta}^{s+1}_{k})$:

\begin{lemma}\label{lemma:r}
Let $\beta_k$ be some positive constant, and let $c_k,c_{k+1}$ be constants defined in intermediate functions $R_k^{s+1}, R_{k+1}^{s+1}$ that satisfy the following: 
\begin{equation}
c_k - \frac{1}{2}\gamma_k = c_{k+1}(1+ \gamma_k \beta_k + 2  L_0 \epsilon \gamma_k) + (2L^2+ L_0^2 \epsilon^2 ) (L\gamma_k^2 + 2c_{k+1}\gamma_k^2) + \frac{\beta_k}{2}L_0\epsilon\gamma_k 
\end{equation}
Then, we have 
\begin{equation}
R_{k+1}^{s+1}(\boldsymbol{\theta}_{k}^{s+1}) \leq R_k^{s+1}(\boldsymbol{\theta}_k^{s+1}) - \Gamma_k \mathbb{E} \left[ \|\nabla \mathcal{J}(\boldsymbol{\theta}_k^{s+1}; \boldsymbol{\theta}_k^{s+1})\|^2 \right] - \frac{1}{2} \gamma_k \|\boldsymbol{\theta}_k^{s+1} - \widetilde{\boldsymbol{\theta}}^s\|^2,
\end{equation}
where $\Gamma_k = \left( \gamma_k - \frac{c_{k+1} \gamma_k + \frac{1}{2} L_0 \epsilon \gamma_k}{\beta_k} - 2L \gamma_k^2 - 4 c_{k+1} \gamma_k^2 \right)$.
\end{lemma}

Note that $R_{k+1}^{s+1}(\boldsymbol{\theta}_{k}^{s+1})$ includes a term $\mathbb{E}[\mathcal{J}(\boldsymbol{\theta}_{k+1}^{s+1}; \boldsymbol{\theta}_k^{s+1})]$, which represents the expected performative risk when the model parameter is $\boldsymbol{\theta}_{k+1}^{s+1}$ but the data distribution is induced by $\boldsymbol{\theta}_k^{s+1}$. This enables us to focus on bounding the influence of $\boldsymbol{\theta}$ through the first argument of $\mathcal{J}$, similar to analyses in the non-PP settings. 
However, since the snapshot full gradient $\nabla \mathcal{J}(\widetilde{\boldsymbol{\theta}}^s; \widetilde{\boldsymbol{\theta}}^s)$ is taken under the {\em previous} distribution, we must additionally bound the effect of the distribution shift from $\mathcal{D}(\widetilde{\boldsymbol{\theta}}^s)$ to $\mathcal{D}(\boldsymbol{\theta}^{s+1}_{k})$. 
To this end, we provide the following auxiliary lemma.

\begin{lemma}[\citet{li2024stochastic}]\label{lemma:aux}
    Under Assumption~\ref{assumption:lipschitzness} and $\epsilon$-sensitivity in Def.~\ref{assumption:sensitivity}, we have:
    \begin{align*}
        |\mathcal{J}(\boldsymbol{\theta}; \boldsymbol{\theta}_1) - \mathcal{J}(\boldsymbol{\theta}; \boldsymbol{\theta}_2)| \le L_1 \epsilon\|\boldsymbol{\theta}_1 - \boldsymbol{\theta}_2\|, ~~\forall \boldsymbol{\theta}.
    \end{align*}
\end{lemma}
The proof is in App.~B of \citet{li2024stochastic}. %We refer readers to App. B of \citet{li2024stochastic} for the proof of Lemma~\ref{lemma:aux} which relies on $\epsilon$-sensitivity and Assumption \ref{assumption:lipschitzness}. 
To prove Lemma~\ref{lemma:r}, we first apply the smoothness conditions of $\ell$ on $\mathbb{E}[\mathcal{J}(\boldsymbol{\theta}_{k+1}^{s+1}; \boldsymbol{\theta}_{k}^{s+1})]$ and leverage Lemma~\ref{lemma:aux} to bound $\|\boldsymbol{v}_k^{s+1}\|$ and $\|\boldsymbol{v}_k^{s+1}\|^2$.  The full proof is in App.~\ref{app:lemmar}.

\textbf{Step 2: Construct the final Lyapunov function.} We then construct the final Lyapunov function: at epoch $s+1$, we define Lyapunov function evaluated at iterations $k$ and $k+1$ as:
\begin{align*}
    \widetilde{R}_k^{s+1} \overset{\triangle}{=} \mathbb{E}[\mathcal{J}(\boldsymbol{\theta}_k^{s+1}; \boldsymbol{\theta}_k^{s+1})] + c_k \| \boldsymbol{\theta}_k^{s+1} - \widetilde{\boldsymbol{\theta}}^s\|^2, \quad
    \widetilde{R}_{k+1}^{s+1} \overset{\triangle}{=} \mathbb{E}[\mathcal{J}(\boldsymbol{\theta}_{k+1}^{s+1}; \boldsymbol{\theta}_{k+1}^{s+1})] + c_{k+1} \| \boldsymbol{\theta}_{k+1}^{s+1} - \widetilde{\boldsymbol{\theta}}^s\|^2.
\end{align*}
The following Lemma~\ref{lemma:r2} characterizes the relation between $    \widetilde{R}_{k}^{s+1}$ and $    \widetilde{R}_{k+1}^{s+1}$:

\begin{lemma}\label{lemma:r2}
Let $c_k, c_{k+1}, \beta_k, \Gamma_k$ be defined the same as in Lemma~\ref{lemma:r}, we have: 
%\begin{align*}
%    \widetilde{R}_k^{s+1} \overset{\triangle}{=} \mathbb{E}[\mathcal{J}(\boldsymbol{\theta}_k^{s+1}, \boldsymbol{\theta}_k^{s+1})] + c_k \| \boldsymbol{\theta}_k^{s+1} - \widetilde{\boldsymbol{\theta}}^s\|^2\\
 %   \widetilde{R}_{k+1}^{s+1} \overset{\triangle}{=} \mathbb{E}[\mathcal{J}(\boldsymbol{\theta}_{k+1}^{s+1}, \boldsymbol{\theta}_{k+1}^{s+1})] + c_{k+1} \| \boldsymbol{\theta}_{k+1}^{s+1} - \widetilde{\boldsymbol{\theta}}^s\|^2
%\end{align*}
%Then we have:
\begin{equation}\label{eq:r2}
    \widetilde{R}_{k+1}^{s+1} \le \widetilde{R}_k^{s+1} + (L_0^2 L_1^2\epsilon^4 + 2L_1^2\epsilon^2L^2 + 2L_1^2\epsilon^2)\gamma_k 
    - \left(\Gamma_k - \frac{\gamma_k}{4}\right)\cdot \mathbb{E}\left[\|\nabla \mathcal{J}(\boldsymbol{\theta}_k^{s+1}; \boldsymbol{\theta}_k^{s+1})\|^2\right]
\end{equation}
\end{lemma}
Lemma~\ref{lemma:r2} can be proved by leveraging Lemmas~\ref{lemma:r} and~\ref{lemma:aux}. 
The proof also establishes a bound on the difference between $\mathcal{J}(\boldsymbol{\theta}_k^{s+1}; \boldsymbol{\theta}_k^{s+1})$ and $\mathcal{J}(\boldsymbol{\theta}_{k+1}^{s+1}; \boldsymbol{\theta}_{k+1}^{s+1})$. The full proof is provided in App.~\ref{app:lemmar2}.

%where both arguments in $\mathcal{J}$ change as illustrated in Lemma \ref{lemma:r2}. We can prove Lemma \ref{lemma:r2} relatively at ease based on Lemma \ref{lemma:r} and then additionally apply the auxiliary Lemma \ref{lemma:aux}. The proof of Lemma \ref{lemma:r2} demonstrate we bound the difference between $\mathcal{J}(\boldsymbol{\theta}_k^{s+1}; \boldsymbol{\theta}_k^{s+1})$ and $\mathcal{J}(\boldsymbol{\theta}_{k+1}^{s+1}; \boldsymbol{\theta}_{k+1}^{s+1})$. The complete proof is in App. \ref{app:lemmar2}.

\textbf{Step 3: Establish the convergence of \alg} Building on Lemma~\ref{lemma:r2}, we are ready to show the convergence of \alg algorithm.

\begin{theorem}[Convergence to an SPS point]
\label{theorem:converge}
{Under Assumptions~\ref{assumption: smoothness},~\ref{assumption:lipschitzness} and Definition~\ref{assumption:sensitivity}, consider \alg(Algorithm~\ref{alg:one}) run for $T$  iterations
over $S=T/m$ epochs, and set $c_m=0$ in the Lyapunov
function at the final iteration of each epoch. Let $\gamma_k\equiv\gamma=
    \min\left\{
        1,\,
        \frac{1}{8L+\frac{8}{3}L_0\epsilon},\,
        \frac{\omega L_0\epsilon}{C_0+C_1}
    \right\},$ for all $k=0,\ldots,m-1$. Here, $C_0$, $C_1$, and $\omega$ are
constants determined by $L$, $L_0$, $\epsilon$, $m$, and
$\beta_k$ as specified in Appendix~\ref{app:thmconverge}. For this choice of $\gamma$, let
$\Gamma>0$ be a uniform constant satisfying $\Gamma_k-\frac{\gamma}{4}\geq \Gamma,
 k=0,\ldots,m-1.$ Then
%Let $\gamma_k\equiv\gamma\triangleq \min\left\{ 1,\,\frac{1}{8L+\frac{8}{3}L_0\epsilon},\,\frac{\omega L_0\epsilon}{C_0+C_1}\right\}$, for all $k=0,\ldots,m-1$, and $C_0$, $C_1$, $\omega$, $m$, and $\beta$ are specified in Appendix~\ref{app:thmconverge}. As shown in Appendix~\ref{app:thmconverge}, this constant choice of stepsize guarantees a uniform lower bound $\Gamma_k - \frac{\gamma}{4} \;\geq\; \Gamma \;\triangleq\; \frac{\gamma}{8} \;>\;0, \text{for all } k=0,\ldots,m-1,$ where $\Gamma_k$ is defined in Lemma~\ref{lemma:r}. Telescoping the per-step descent in Lemma~\ref{lemma:r2} over $k$ and $s$ and applying the lower bound then yields
\begin{align}
\label{eq:mainconverge}
\frac{1}{T}
\sum_{s=0}^{S-1}\sum_{k=0}^{m-1}
\mathbb{E}\left[
\left\|
\nabla \mathcal{J}
(\boldsymbol{\theta}_k^{s+1};
 \boldsymbol{\theta}_k^{s+1})
\right\|^2
\right]
\leq
\frac{\Delta_0}{T\Gamma}+\Delta_1,
\end{align}
where
\begin{align*}
\Delta_0
\triangleq
\mathcal{J}(\boldsymbol{\theta}_0^0;
            \boldsymbol{\theta}_0^0)-\ell_*, \quad
\Delta_1
\triangleq
\frac{
\left(
L_0^2L_1^2\epsilon^4
+2L_1^2\epsilon^2L^2
+2L_1^2\epsilon^2
\right)
\displaystyle\sum_{s=0}^{S-1}\sum_{k=0}^{m-1}\gamma_k
}{T\Gamma}.
\end{align*}
Moreover, this choice of $\gamma$ ensures
that $\Gamma=\Theta(\gamma)$. Consequently, \alg achieves an
$\mathcal{O}(1/T)$ convergence rate to a variance-independent
$\mathcal{O}(\epsilon^2+\epsilon^4)$ neighborhood of an SPS point.
%using $\gamma_k\equiv\gamma$ and $\Gamma=\gamma/8$.Consequently, \alg~converges at a rate of $\mathcal{O}(1/T)$ to a variance-independent neighborhood $\mathcal{O}(\epsilon^2+\epsilon^4)$.
}
\end{theorem}

{Eqn. \ref{eq:mainconverge} is established by summing the bound in Lemma \ref{lemma:r2} over $k$ and $s$ and proving the existence of $\Gamma$. It mainly provides a bound on the expectation of the squared gradient norm over the iterates. This is the standard convergence metric in non-convex stochastic optimization \citep{reddi2016stochastic, fang2018spider}, where one typically certifies convergence by showing that a uniformly sampled iterate has a small expected gradient norm}. The full proof is in App.~\ref{app:thmconverge}. Theorem~\ref{theorem:converge} improves on previous results in three aspects: (i) the convergence rate  $\mathcal{O}(1/T)$ is faster than the $\mathcal{O}(1/\sqrt{T})$ rate %that 
of \citet{li2024stochastic}; 
(ii) we eliminate Assumption~\ref{assumption:variance}, and the error neighborhood $\Delta_1$ {\bf no longer depends} on the variance of the stochastic gradient; and (iii) the error neighborhood is $\mathcal{O}(\epsilon^2 + \epsilon^4)$, which is guaranteed to be smaller than that of the SGD-GD-Based PP algorithm in \citep{li2024stochastic} when the distribution map is mildly sensitive ($\epsilon < 1$). 

\paragraph{IFO complexity.} We next derive the incremental first-order oracle (IFO) complexity \citep{reddi2016stochastic} of \alg---the number of per-sample gradient evaluations required to %reach operations required to sample and compute the gradient to achieve 
reach a $\mathcal{O}(\delta)$-SPS. 

% We provide more details in App. \ref{app:adv}.

\begin{corollary}[IFO Complexity of \algns]\label{corollary:complexity}
Consider the finite-population setting with \(n\) samples and \(\alpha\in(0,1)\). With fixed parameters
\[
\gamma_k=\gamma=\Theta(n^{-\alpha}),\qquad
\beta_k=\beta=\Theta(n^{\alpha/2}),\qquad
m=\Theta(n^{\alpha/2}),
\]
where the implied constants are chosen to satisfy the conditions in Appendix D, we have
\[
\Gamma=\Theta(n^{-\alpha}),
\]
and the IFO complexity required to reach an \(O(\delta)\)-SPS solution beyond the error neighborhood is
\[
O\!\left(
\frac{(n^\alpha+n^{1+\alpha/2})\Delta_0}{\delta}
\right).
\]
% Consider the finite-population setting with $n$ samples and $\alpha \in (0,1)$. With fixed parameters $\gamma_k = \gamma = \mathcal{O}(\frac{1}{n^{\alpha}}), \beta_k = \beta = \mathcal{O}(n^{\frac{\alpha}{2}}), m = \mathcal{O}(n^{\frac{\alpha}{2}})$, we have $\Gamma = \mathcal{O}(\frac{1}{n^{\alpha}})$, and the IFO complexity  required to reach a $\mathcal{O}(\delta)$-SPS solution (beyond the error neighborhood) is $\mathcal{O}\left(\frac{(n^{\alpha} + n^{1+\frac{\alpha}{2}})\Delta_0}{\delta}\right)$. 
\end{corollary}
\color{black}
By Theorem~1 of \citet{li2024stochastic}, the IFO complexity of SGD-GD is $\mathcal{O}\left(\frac{(\Delta_0 + \frac{L}{2}\sigma_0^2)^2}{\delta^2}\right)$. Comparing the two, \alg has a clear advantage when (i) the population size $n$ is moderate and $\delta$ is small; or (ii) the variance parameter $\sigma_0^2$ is large or even intractable (e.g., heavy-tailed noise \citep{gurbuzbalaban2021heavy}), in which case the SGD-GD bound can become arbitrarily large or even infinite. By contrast, \alg's complexity is independent of $\sigma_0$. When the population is very large and $\delta$ is not too small, however, \alg can be slower than SGD-GD.

%Compared to the IFO Complexity of SGD-GD ($\mathcal{O}(\frac{(\Delta_0 + \frac{L}{2}\sigma_0^2)^2}{\delta^2})$), \alg has a clear advantage when the population size $n$ is not too large or when the variance parameter $\sigma_0^2$ is large or even does not exist. However, when the population size is very large while the error $\delta$ is not too small, the computation of SPRINT can be slower than SGD-GD. 

%\noindent \textbf{Comparison of IFO Complexity between \alg and SGD-GD}\label{app:adv} According to Theorem 1 of \citet{li2024stochastic}, we can directly get the IFO Complexity of SGD-GD. When $n$ is not too large and when we require $\delta$ to be small, \alg will have much smaller complexity. Meanwhile, when $\sigma_0$ is large or even intractable (e.g., heavy-tailed/fat-tailed noise settings \citep{gurbuzbalaban2021heavy}), SGD-GD can also have extremely large or even infinite IFO Complexity, while the complexity of \alg is independent of the variance parameter $\sigma_0$. 

%In Corollary 5.5 and Appendix H, we have derived the IFO complexity of SPRINT to be $\mathcal{O}(\frac{(n^{\alpha} + n^{1+\frac{\alpha}{2}})\Delta_0}{\delta})$. Compared to the complexity of SGD-GD is $\mathcal{O}(\frac{1}{\delta^2})$, we know there is a trade-off between population size $n$ and the error bound in additional to the error neighborhood $\delta$. When the population size is very large while the error $\delta$ is not too small, the computation of SPRINT can be slower than SGD-GD. 

\paragraph{Extension to infinite populations.}\label{app:inf}
Our finite-sum analysis avoids the bounded-variance assumption by computing the full gradient at each snapshot. In the infinite-population setting this is no longer possible, and some form of tail control is needed to estimate the mean gradient. One option is to reinstate Assumption~\ref{assumption:variance}, or a milder bounded-$(1+\kappa)$-moment condition:
%Our current finite-sum analysis avoids bounded variance by using a full gradient. However, when it comes to the setting where we have an infinite population, we cannot compute the full gradient, and some tail control is necessary to estimate the mean gradient. We may still use the bounded variance assumption \ref{assumption:variance}, or we can also use milder conditions such as bounded $(1+\kappa)$-moment. Then we can use the following strategies to reduce the variance and accelerate convergence:
\begin{enumerate}[leftmargin=*,topsep=0.1cm,itemsep=0cm]
    \item Use a minibatch as the gradient snapshot, which attains SVRG-type rates but requires a bounded-variance assumption. 
    
   % We may use a minibatch of samples as the gradient snapshot, but we do need the bounded variance assumption to obtain SVRG-type rates.

    \item Use a growing-minibatch control-variate scheme similar to \citet{jothimurugesan2018variance}, which keeps the gradient variance bounded without an explicit bounded-variance assumption.
    
   % We can consider methods similar to \cite{jothimurugesan2018variance}, where the minibatch is growing and a control variate method is used to keep the gradient variance bounded. In this way, we do not need an explicit bounded variance assumption.
\end{enumerate}
We emphasize that even in the finite-sum setting, SGD-GD still requires Assumption~\ref{assumption:variance} because it samples a random component gradient at every iteration, so the per-iterate variance does not vanish. Thus the improvement of \alg over SGD-GD is robust.

%However, even in the finite-sum setting, SGD-GD still needs the bounded variance assumption \ref{assumption:variance} because it samples a random component gradient at every iteration, so the variance at a fixed point does not vanish. Thus, the improvement of \alg compared to SGD-GD is solid.

\section{Experiments}\label{sec:exp}

In this section, we conduct comprehensive experiments to verify the effectiveness of \alg in practice. 
We compare the convergence behavior of the training loss and training accuracy of \alg with SGD-GD \citep{li2024stochastic} on $3$ real-world datasets with non-convex models and study the performative effects. 
The results of Credit dataset \citep{creditdata} and CIFAR-10 dataset \citep{krizhevsky2009learning} are included in the main paper, while results of MNIST dataset \citep{deng2012mnist} and additional experiments on different random seeds are relegated in App.~\ref{app:exp} due to space limitation. 
We also show the magnitude of the squared gradient norm $\|\nabla \mathcal{J}(\widetilde{\boldsymbol{\theta}}^s; \widetilde{\boldsymbol{\theta}}^s)\|^2$ in App.~\ref{app:exp} to validate Thm.~\ref{theorem:converge}. 

Besides the credit and MNIST datasets, we are the first to provide results using CNN models on the CIFAR-10 \citep{krizhevsky2009learning} dataset in the PP settings. 
As neural network architectures are widely used, our experiments go one step further for the practical applications of PP algorithms. 

\noindent \textbf{Credit dataset.} We use the \textit{Give me some credit} data \citep{creditdata} with $10$ features $X \in \mathbb{R}^{10}$ to predict creditworthiness $Y \in \{0,1\}$ \citep{perdomo_performative_2021, jin2024addressingpolarizationunfairnessperformative}. 
We preprocess the dataset by selecting a balanced subset of $5,000$ examples and then normalizing features similar to \citet{perdomo_performative_2021}. We then assume a credit scoring agency uses a two-layer MLP model to predict $Y$ given the individuals' features $X$ with cross entropy loss. 
To simulate performative effects, we assume individuals are strategic \citep{perdomo_performative_2021, hardt_strategic_2015}, i.e., there is a subset of features $X_s\in X$ which individuals can change to $X_s'$ based on the current model. 
Given the current model $f_{\boldsymbol{\theta}}$, individuals with feature $x_s$ will best respond to the model in the direction of $\nabla_xf_{\boldsymbol{\theta}}(x_s)$ and change $x_s$ to $x_s' = x_s + \alpha \nabla_xf_{\boldsymbol{\theta}}(x_s)$ to improve their scoring \citep{rosenfeld2020predictions, xie2024nonlinear}.  

In this setting, we train the MLP model using \alg and SGD-GD for $400$ gradient calls and illustrate the squared gradient norm $\| \nabla\mathcal{J} (\boldsymbol{\widetilde{\theta}}^s; \boldsymbol{\widetilde{\theta}}^s)\|^2$ in Fig.~\ref{fig:grad}. 
We simulate different performative effects using $\alpha \in \{0.01, 0.2, 0.4\}$. 
In Fig.~\ref{fig:grad}, we show that the \alg algorithm achieves lower squared gradient norm compared with the SGD-GD in all $3$ cases and the advantages are quite prominent when performative effects are larger (i.e., when $\alpha = 0.2$ and $0.4$). To assess generalization, we additionally report out-of-sample evaluation accuracy on a held-out test set in Fig.~\ref{fig:credit_eval_1}, Fig.~\ref{fig:credit_eval_2} and Fig.~\ref{fig:credit_eval_3}. We show that \alg achieves competitive or superior evaluation accuracy compared to SGD-GD across all $\alpha$ settings, indicating that faster convergence to SPS does not come at the cost of generalization. 

{\noindent \textbf{Spambase Dataset.} We use the Spambase dataset~\citep{spambase_94} with 57 features $X \in \mathbb{R}^{57}$ to predict whether an email is spam. We preprocess the dataset by using the full set of 4,601 examples and normalizing features similar to~\citet{li2024stochastic}. We then assume a spam filtering service uses a two-layer MLP model to predict Y given the emails' features X with cross entropy loss. To simulate performative effects, we assume senders are strategic as we do in the credit dataset. We also train the MLP model using \algns and SGD-GD for $300$ gradient calls and also illustrate the squared gradient norm in Fig.~\ref{fig:grad}. We simulate different performative effects using $\alpha \in \lbrace 40, 60, 80\rbrace$. We also see a lower squared gradient norm compared with the SGD-GD in all 3 cases. We also additionally report out-of-sample evaluation accuracy on a held-out test set in Fig.~\ref{fig:spambase_eval_1}, Fig.~\ref{fig:spambase_eval_2} and Fig.~\ref{fig:spambase_eval_3}.}

% \textcolor{blue}{Since Theorem~\ref{theorem:converge} and Proposition~\ref{proposition:converge} characterize convergence in terms of the  squared gradient norm $\|\nabla\mathcal{J}(\boldsymbol{\theta}, \boldsymbol{\theta})\|^2$ rather than the training loss value, we additionally track $\|\nabla\mathcal{J}(\widetilde{\boldsymbol{\theta}_s}, \widetilde{\boldsymbol{\theta}_s})\|^2$ for the Credit dataset in Fig~\ref{fig:grad}}

\begin{figure}[h]
    \centering

    \begin{subfigure}[b]{0.32\textwidth}
        \includegraphics[width=\textwidth]{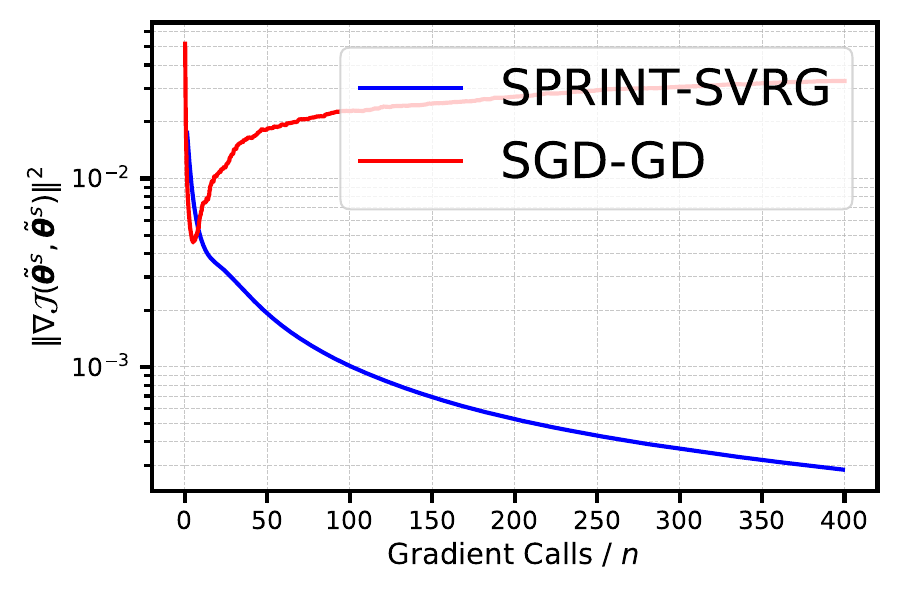}
        \caption{$\alpha=0.01$}
        \label{fig:grads_sub1}
    \end{subfigure}
    \hfill
    \begin{subfigure}[b]{0.32\textwidth}
        \includegraphics[width=\textwidth]{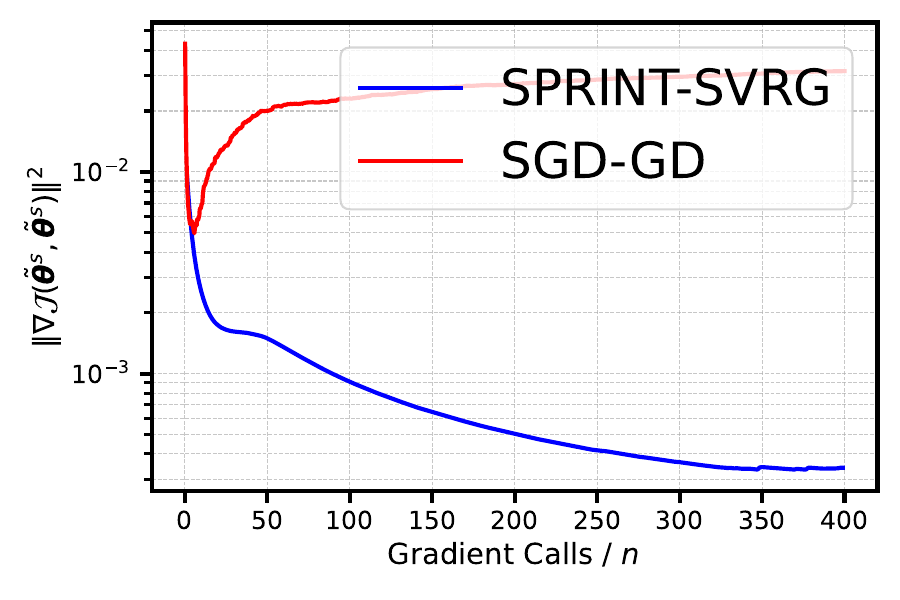}
        \caption{$\alpha = 0.2$}
        \label{fig:grads_sub2}
        
    \end{subfigure}
    \hfill
    \begin{subfigure}[b]{0.32\textwidth}
        \includegraphics[width=\textwidth]{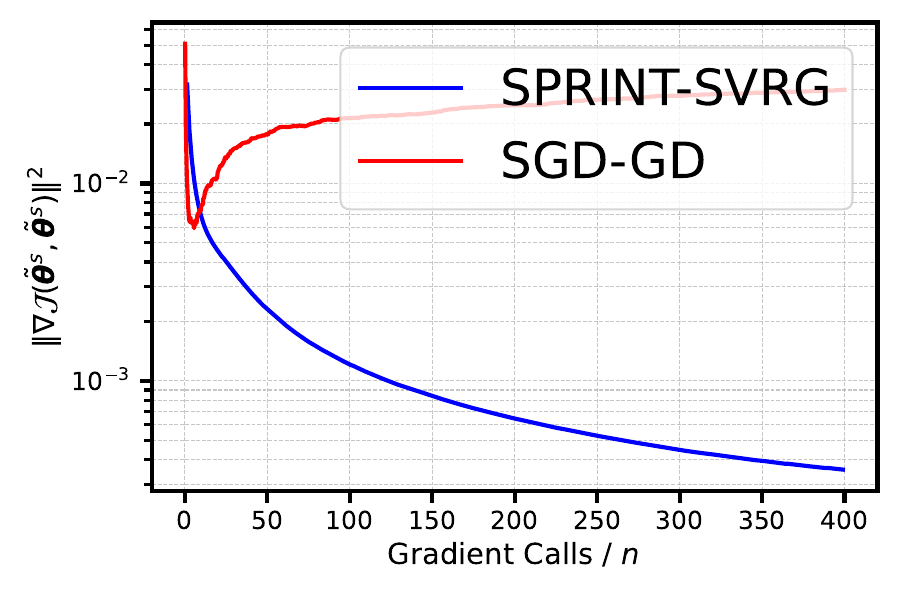}
        \caption{$\alpha = 0.4$}
        \label{fig:grads_sub3}
    \end{subfigure}

    % \vspace{-0.5cm}

    \begin{subfigure}[b]{0.32\textwidth}
        \includegraphics[width=\textwidth]{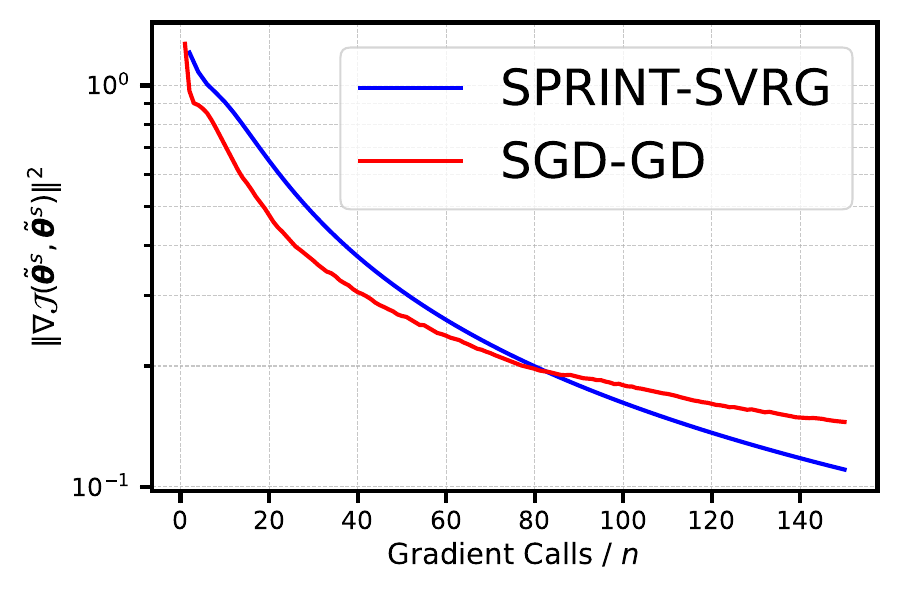}
        \caption{$\alpha = 20$}
        \label{fig:grads_sub4}
    \end{subfigure}
    \hfill
    \begin{subfigure}[b]{0.32\textwidth}
        \includegraphics[width=\textwidth]{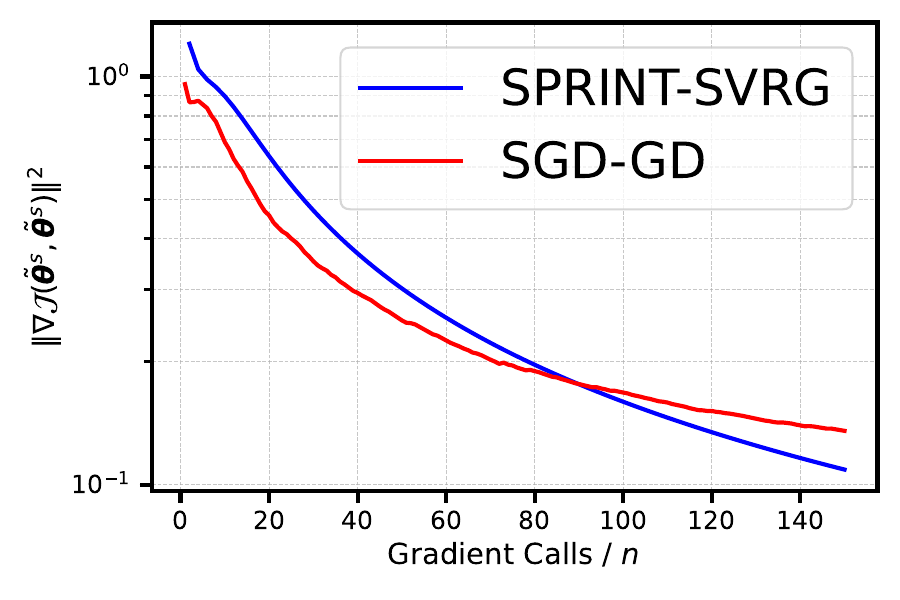}
        \caption{$\alpha = 50$}
        \label{fig:grads_sub5}
    \end{subfigure}
    \hfill
    \begin{subfigure}[b]{0.32\textwidth}
        \includegraphics[width=\textwidth]{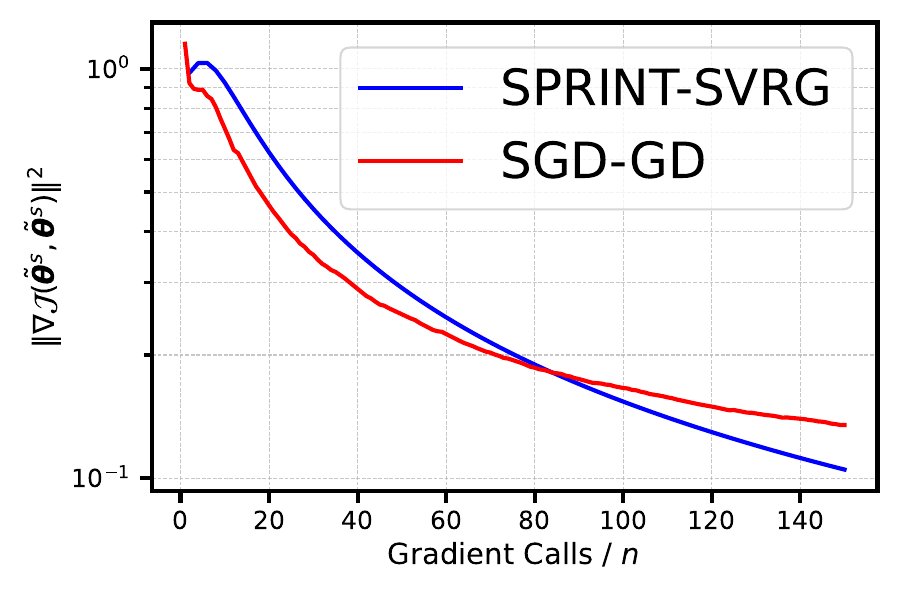}
        \caption{$\alpha = 80$}
        \label{fig:grads_sub6}
    \end{subfigure}

    % \vspace{-0.5cm}

    \begin{subfigure}[b]{0.32\textwidth}
    \includegraphics[width=\textwidth]{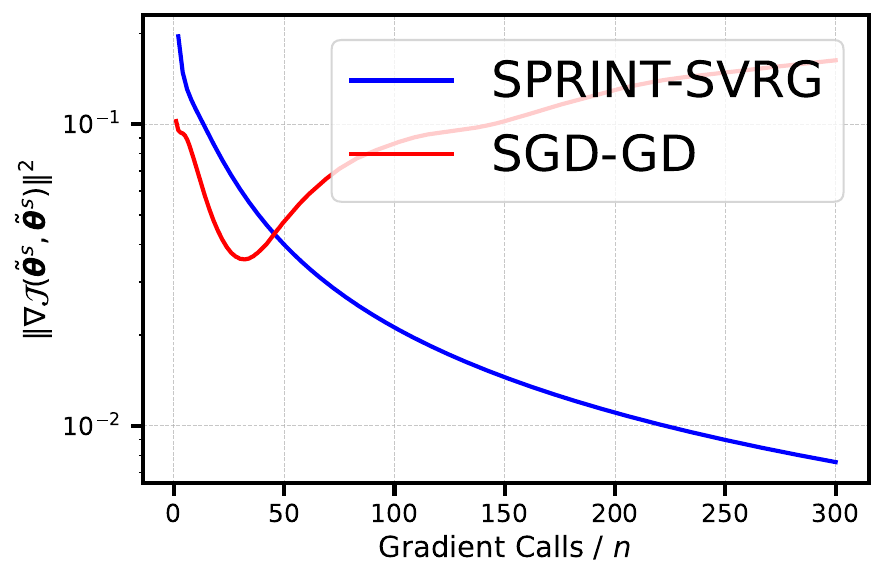}
    \caption{$\alpha = 40$}
    \label{fig:grads_sub7}
    \end{subfigure}
    \hfill
    \begin{subfigure}[b]{0.32\textwidth}
    \includegraphics[width=\textwidth]{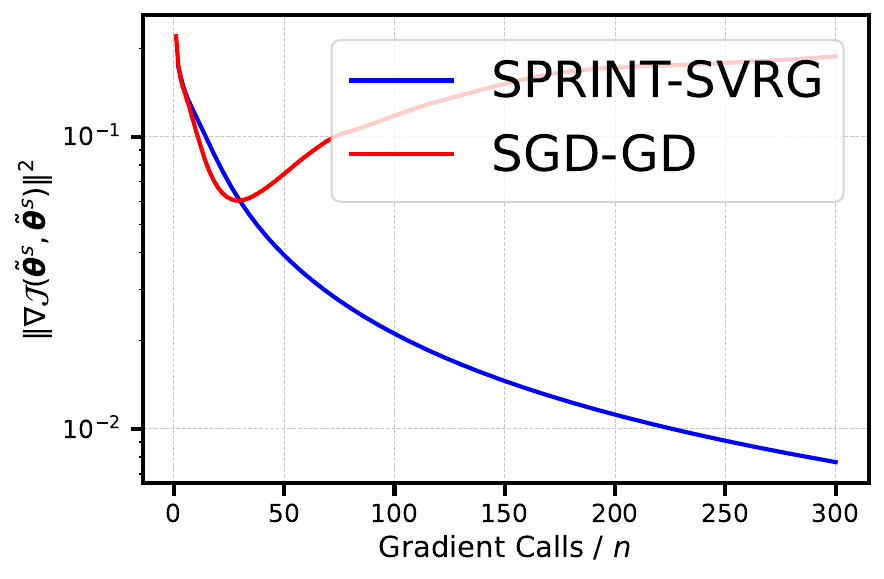}
    \caption{$\alpha = 60$}
    \label{fig:grads_sub8}
    \end{subfigure}
    \hfill
    \begin{subfigure}[b]{0.32\textwidth}
    \includegraphics[width=\textwidth]{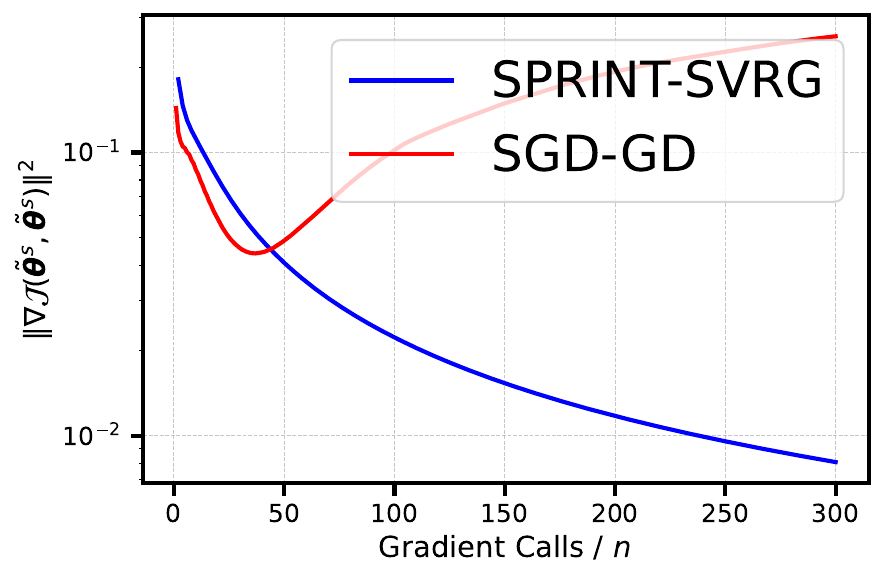}
    \caption{$\alpha = 80$}
    \label{fig:grads_sub9}
    \end{subfigure}

    \caption{Squared gradient norm $\|\nabla \mathcal{J}(\widetilde{\boldsymbol{\theta}}^s; \widetilde{\boldsymbol{\theta}}^s)\|^2$. From top to bottom are \textbf{Credit} dataset, \textbf{MNIST} dataset and \textbf{Spambase} dataset.}
    \label{fig:grad}
\end{figure}

\begin{figure}[h]
    \centering
    \begin{subfigure}[b]{0.48\textwidth}
        \includegraphics[width=0.8\textwidth]{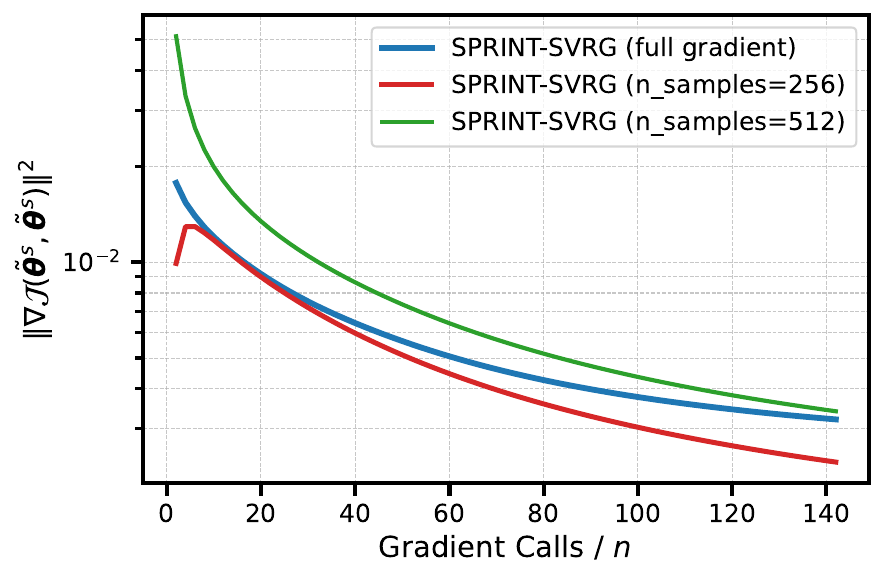}
        \caption{Credit dataset}
        \label{fig:minibatch_credit}
    \end{subfigure}
    \hfill
    \begin{subfigure}[b]{0.48\textwidth}
        \includegraphics[width=0.8\textwidth]{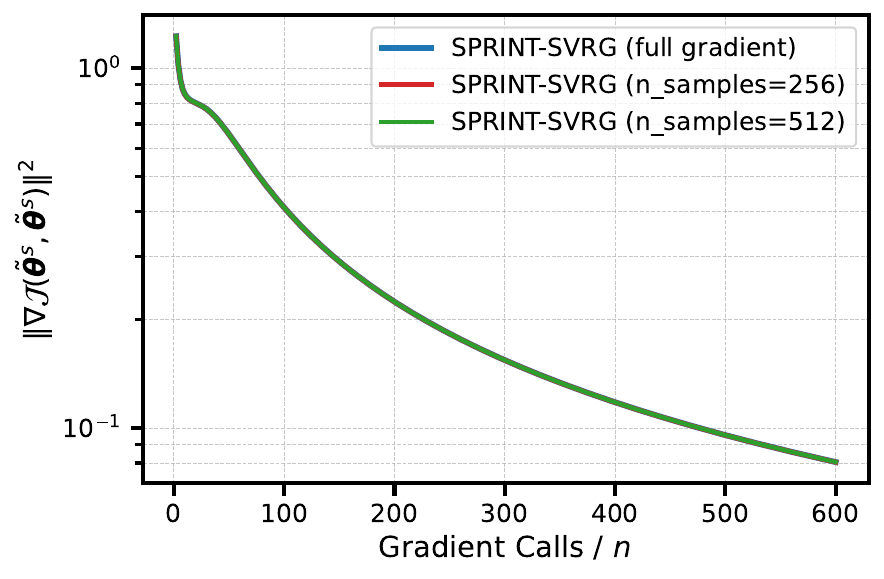}
        \caption{MNIST dataset}
        \label{fig:minibatch_mnist}
    \end{subfigure}
    \caption{Mini-batch gradient snapshot experiments on the \textbf{Credit} and \textbf{MNIST} datasets. We plot the squared gradient norm $\|\nabla \mathcal{J}(\widetilde{\boldsymbol{\theta}}^s; \widetilde{\boldsymbol{\theta}}^s)\|^2$ over epochs for \alg with a full gradient snapshot and with mini-batch snapshot sizes $B\in \{256, 512\}$. A larger batch size yields a snapshot estimate closer to the true full gradient and thus a convergence curve nearer to the full-gradient baseline. The results demonstrate that \alg  degrades only mildly under moderate snapshot approximation, suggesting practical deployability in large-population settings.
    % We compare SGD-GD, \textsc{SPRINT} with full gradient snapshot, and \textsc{SPRINT} with mini-batch snapshot sizes $B \in \{128, 256\}$. \textbf{Note:} the theoretical guarantee of Theorem~\ref{theorem:converge} applies only to the full-gradient variant; the mini-batch results are provided solely to demonstrate empirical robustness under practical approximations.
    }
    \label{fig:minibatch}
\end{figure}

{\noindent \textbf{MNIST Dataset.} We use the MNIST dataset \citep{deng2012mnist} and train a two-layer MLP to predict the $10$ possible digits. We randomly sample approximately $12000$ images and train the model using a learning rate set at $0.003$. Similar to retention dynamics \citep{hashimoto2018fairness, jin2024addressingpolarizationunfairnessperformative, Zhang_2019_Retention}, we assume the class distribution will change based on the performance of the current model. The fraction in class $c$ at iteration $k$ of $p^{k+1}_c = \frac{e^{-\alpha \ell^k_c}}{\sum_{c'}e^{-\alpha \ell^k_{c'}}}$, where $\ell^k_c$ is the loss expectation of class $c$ in iteration $k$. This means the image class with larger loss now will be less likely to appear at the next iteration \citep{Zhang_2019_Retention}. We train the model for $80$ epochs with cross entropy loss and visualize the squared gradient norm in Fig.~\ref{fig:grad}. The $\alpha$ is set at $20, 50, 80$ to simulate different magnitudes of performative effects. In Fig.~\ref{fig:grads_sub4}, Fig.~\ref{fig:grads_sub5} and Fig.~\ref{fig:grads_sub6}, \alg consistently converges to a smaller squared gradient norm}.

% Moreover, we also iterate through different learning rates and verify that \alg consistently achieves a comparable or lower training loss and higher accuracy across all tested learning rates (Fig. 7), consistent with, though not on its own establishing, the advantage shown in Fig~\ref{fig:grad}. In Fig. \ref{fig:mnist_lr}, we plot the dynamics of training loss and accuracy when learning rate in $\{0.001, 0.003, 0.01\}$ and the results demonstrate the competence of \alg.

\noindent \textbf{Mini-batch gradient snapshot experiments.} 
{We emphasize that the full-gradient snapshot is used throughout all convergence comparisons reported in the main text as well as the additional results in Appx.~\ref{app:exp}; the mini-batch approximation is examined only in this paragraph, as a separate ablation isolating the effect of snapshot fidelity.} While the theoretical guarantee of \alg requires a full gradient snapshot computed over the entire population, this requirement may be impractical in large-population settings where iterating over all $n$ samples at the beginning of each epoch incurs prohibitive computational cost. A natural remedy is to approximate the full gradient snapshot with a mini-batch of size $B \ll n$. Although Theorem~\ref{theorem:converge} does not apply to this setting, it is important to assess whether \alg degrades gracefully under such approximations in practice.

To this end, we investigate the empirical behavior of \alg when the snapshot gradient $\nabla \mathcal{J}(\widetilde{\boldsymbol{\theta}}^s; \widetilde{\boldsymbol{\theta}}^s)$ is approximated by mini-batches of size $B \in \{256, 512\}$, and compare against the full-gradient implementation. We report the squared gradient norm $\|\nabla \mathcal{J}(\widetilde{\boldsymbol{\theta}}^s; \widetilde{\boldsymbol{\theta}}^s)\|^2$ across gradient calls on both the Credit and MNIST datasets in Figures~\ref{fig:minibatch_credit} and \ref{fig:minibatch_mnist}.

Several observations emerge from the results. First, across both datasets, the convergence curves under $B=256$ and $B=512$ remain close to that of the full-gradient variant throughout training, suggesting that the variance introduced by the mini-batch snapshot is well-controlled in practice. Second, the gap between different batch sizes is small and does not widen appreciably as training progresses, indicating that \alg is not overly sensitive to the fidelity of the gradient snapshot. Third, the ordering of curves is consistent with intuition: a larger batch size yields a curve closer to the full-gradient baseline, as the snapshot estimate is more accurate. These results suggest that \alg is practically deployable even in large-population settings where a full gradient snapshot is too costly, and that a modest batch size such as $B=512$ may suffice to retain most of the convergence benefits.

\begin{figure}[h]
    \centering
    \begin{subfigure}[b]{0.48\textwidth}
        \includegraphics[width=0.8\textwidth]{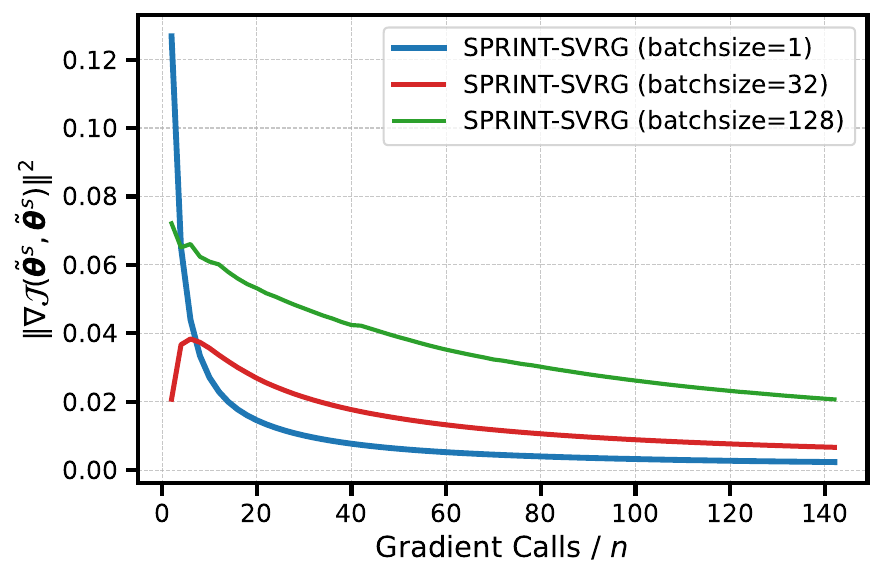}
        \caption{Credit dataset}
        \label{fig:batch_credit}
    \end{subfigure}
    \hfill
    \begin{subfigure}[b]{0.48\textwidth}
        \includegraphics[width=0.8\textwidth]{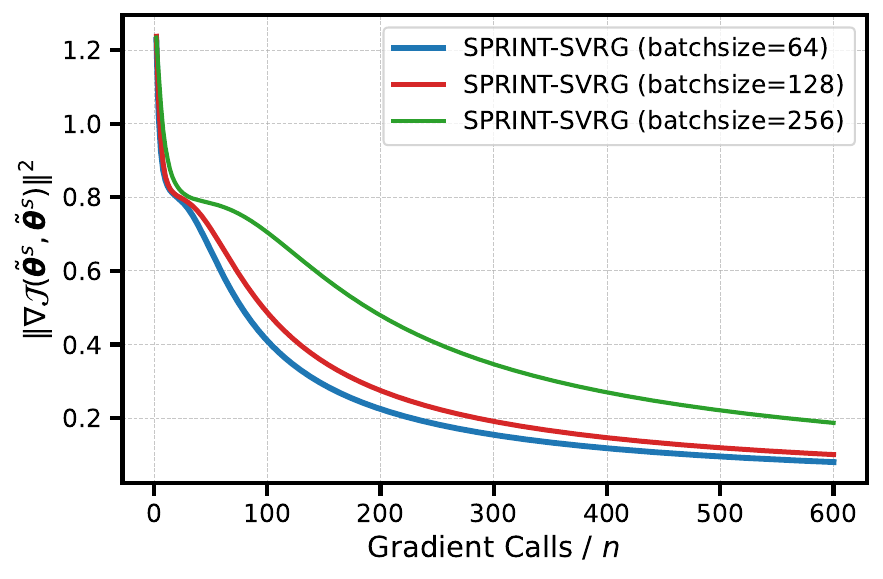}
        \caption{MNIST dataset}
        \label{fig:batch_mnist}
    \end{subfigure}
    \caption{Empirical validation of variance-independent error neighborhood on the Credit dataset (left) and MNIST dataset (right). We vary the batch size $B \in \{1, 32, 128\}$ (Credit) and $B \in \{64,128,256\}$(MNIST) to control the stochastic gradient variance $\sigma_0^2$. A larger batch size corresponds to smaller variance. If the error neighborhood were variance-dependent (as in SGD-GD), the converged gradient norm $\|\nabla \mathcal{J}(\tilde{\theta}^s, \tilde{\theta}^s)\|^2$ would decrease as batch size increases. The results show that the converged gradient norms of SPRINT remain approximately constant across different batch sizes, providing empirical evidence consistent with the variance-independent error neighborhood established in Theorem~\ref{theorem:converge}.
    }
    \label{fig:batchsize}
\end{figure}

\section{Conclusion and future work} \label{sec:conclusion}

This paper focused on improving the convergence of stochastic optimization methods in non-convex PP settings by adapting a variance reduction method. 
We proposed \alg as a new algorithm to achieve better convergence results without the need for the bounded variance assumption on stochastic gradient descent estimates. 
We also presented extensive experiments with different non-convex neural network models to compare \alg with SGD-GD, which demonstrated that \alg converges to an SPS solution with a smaller stationarity gap than SGD-GD under a matched gradient-evaluation budget, with comparable or better training loss, accuracy, and out-of-sample generalization observed as secondary evidence. However, \alg only converges to a non-vanishing error neighborhood. Future work may include developing algorithms to converge to performative optimal solutions in the non-convex PP settings, or adapting other more advanced variance reduction methods (SARAH or SPIDER) to the PP setting and trying to completely eliminate the non-vanishing error neighborhood.

% \section*{Impact Statement}

% This paper is mainly theoretical and focuses on developing more advanced optimization methods in performative prediction (PP) settings. Since PP is often human-related, we believe our work is beneficial for trustworthy machine learning.

\section*{Acknowledgments}

This material is based upon work supported in part by the U.S. National Science Foundation (NSF) under Grant Nos. IIS-2202699, IIS-2416895, IIS-2301599, CMMI-2301601, DMS-2529302, OIA-2616279, CNS-2110259, CNS-2112471, and ECCS-233104; the Defense Advanced Research Projects Agency (DARPA) under Award Nos. D24AP00265 (Young Faculty Award) and HR0011-25-2-0019; the Office of Naval Research (ONR) under Grant No. N00014-24-1-2729; and Meta through Research Award INB3267726.

{\small
\bibliographystyle{plainnat}
\bibliography{bib/sample}
}

\clearpage
\appendix
\newpage

\section{The Use of Large Language Models (LLMs)}

We only use LLMs for polishing the writing without using them to generate any single sentence solely on their own.

\section{Proof of Lemma \ref{lemma:r}}\label{app:lemmar}

\begin{lemma*}
Let $\beta_k$ be some positive constant, and $c_k,c_{k+1}$ are some constants defined in intermediate functions $R_k^{s+1}, R_{k+1}^{s+1}$ that satisfy the following 
\begin{align*}
c_k - \frac{1}{2}\gamma_k = c_{k+1}(1+ \gamma_k \beta_k + 2  L_0 \epsilon \gamma_k) + (2L^2+ L_0^2 \epsilon^2 ) (L\gamma_k^2 + 2c_{k+1}\gamma_k^2) + \frac{\beta_k}{2}L_0\epsilon\gamma_k
\end{align*}
Then, we have 
\begin{align*}
R_{k+1}^{s+1}(\boldsymbol{\theta}_{k}^{s+1}) \le R_k^{s+1}(\boldsymbol{\theta}_k^{s+1}) - \Gamma_k \cdot \mathbb{E}[\|\nabla \mathcal{J}(\boldsymbol{\theta}_k^{s+1}; \boldsymbol{\theta}_k^{s+1})\|^2] - \frac{1}{2} \gamma_k \|\boldsymbol{\theta}_k^{s+1} - \widetilde{\boldsymbol{\theta}}^s\|^2
\end{align*}
where $\Gamma_k = \left( \gamma_k - \frac{c_{k+1} \gamma_k + \frac{1}{2} L_0 \epsilon \gamma_k}{\beta_k} - 2L \gamma_k^2 - 4 c_{k+1} \gamma_k^2 \right) $
\end{lemma*}

\begin{proof}

This proof mainly relies on Assumptions \ref{assumption: smoothness}, \ref{assumption:lipschitzness}, and $\epsilon$-sensitivity.

We first prove another auxiliary Lemma which is very similar to Lemma \ref{lemma:aux}.

\begin{lemma}\label{lemma:aux2}
Under Assumption \ref{assumption: smoothness} and the $\epsilon$-sensitivity condition in Def. \ref{assumption:sensitivity},
for any $\boldsymbol{\theta},\boldsymbol{\theta}_1,\boldsymbol{\theta}_2\in\mathbb{R}^d$, we have:
\begin{align}
\big\|
\nabla J(\boldsymbol{\theta};\boldsymbol{\theta}_1)
-
\nabla J(\boldsymbol{\theta};\boldsymbol{\theta}_2)
\big\| \le
L_0\,\epsilon\,\|\boldsymbol{\theta}_1-\boldsymbol{\theta}_2\|.
\end{align}
\end{lemma}

\begin{proof}
Let $\Pi(\mathcal{D}(\boldsymbol{\theta}_1),D(\boldsymbol{\theta}_2))$ denote the set of all
couplings between $\mathcal{D}(\boldsymbol{\theta}_1)$ and $\mathcal{D}(\boldsymbol{\theta}_2)$.
For any $\pi\in\Pi(\mathcal{D}(\boldsymbol{\theta}_1),\mathcal{D}(\boldsymbol{\theta}_2))$,
let $(Z_1,Z_2)\sim\pi$. Then
\[
\begin{aligned}
\big\|
\nabla \mathcal{J}(\boldsymbol{\theta};\boldsymbol{\theta}_1)
-
\nabla \mathcal{J}(\boldsymbol{\theta};\boldsymbol{\theta}_2)
\big\|
=
\Big\|
\mathbb{E}\big[
\nabla \ell(Z_1;\boldsymbol{\theta})
-
\nabla \ell(Z_2;\boldsymbol{\theta})
\big]
\Big\| \le
\mathbb{E}\Big[
\big\|
\nabla \ell(Z_1;\boldsymbol{\theta})
-
\nabla \ell(Z_2;\boldsymbol{\theta})
\big\|
\Big]  \le
L_0\,\mathbb{E}\big[\|Z_1-Z_2\|\big],
\end{aligned}
\]
where the last inequality follows from Assumption~3.1, i.e.,
\(
\|\nabla \ell(z;\boldsymbol{\theta})-\nabla \ell(z';\boldsymbol{\theta})\|
\le L_0\|z-z'\|
\)
for all $z,z'$.
Thus,
$$\big\|
\nabla \mathcal{J}(\boldsymbol{\theta};\boldsymbol{\theta}_1)
-
\nabla \mathcal{J}(\boldsymbol{\theta};\boldsymbol{\theta}_2)
\big\|
\le
L_0\, W_1\!\big(\mathcal{D}(\boldsymbol{\theta}_1),\mathcal{D}(\boldsymbol{\theta}_2)\big).
$$
Finally, applying the $\epsilon$-sensitivity condition
$W_1(\mathcal{D}(\boldsymbol{\theta}_1),\mathcal{D}(\boldsymbol{\theta}_2))
\le
\epsilon\|\boldsymbol{\theta}_1-\boldsymbol{\theta}_2\|$
completes the proof.
\end{proof}
% --- End of lemma ---

Now we begin to prove Lemma \ref{lemma:r}. Since $\ell(z; \boldsymbol{\theta})$ is $L$-smooth in $\boldsymbol{\theta}$ and $v_k^{s+1} = \frac{1}{\gamma_k}
\bigl(\boldsymbol{\theta}_{k}^{s+1}-\boldsymbol{\theta}_{k+1}^{s+1}\bigr)$, we can apply descent lemma to get the following result:
\begin{align}\label{eq:proof1_1}
        \mathbb{E}[\mathcal{J}(\boldsymbol{\theta}_{k+1}^{s+1}; \boldsymbol{\theta}_{k}^{s+1})] &\le \mathbb{E}[\mathcal{J}(\boldsymbol{\theta}_{k}^{s+1}; \boldsymbol{\theta}_{k}^{s+1}) + \frac{L}{2} \|\boldsymbol{\theta}_{k+1}^{s+1} - \boldsymbol{\theta}_k^{s+1}\|^2 + \nabla\mathcal{J}(\boldsymbol{\theta}_k^{s+1}; \boldsymbol{\theta}_k^{s+1})^T(\boldsymbol{\theta}_{k+1}^{s+1} - \boldsymbol{\theta}_k^{s+1})]\\
       \nonumber &= \mathbb{E}[\mathcal{J}(\boldsymbol{\theta}_{k}^{s+1}; \boldsymbol{\theta}_{k}^{s+1})] + \frac{L\gamma_k^2}{2} \|v_k^{s+1}\|^2 - \gamma_k\nabla\mathcal{J}(\boldsymbol{\theta}_k^{s+1}; \boldsymbol{\theta}_k^{s+1})^Tv_k^{s+1}
\end{align}

By definition of $v_k^{s+1}$, we have:
% \begin{align}\label{eq:vk}
%     \mathbb{E}[v_k^{s+1}] &=  \mathbb{E}[\nabla \ell(z_{i_k}; \boldsymbol{\theta}_k^{s+1}) - \nabla \ell(z_{i_k}; \widetilde{\boldsymbol{\theta}}^s) +\nabla\mathcal{J}(\widetilde{\boldsymbol{\theta}}^s; \widetilde{\boldsymbol{\theta}}^s)]\\
%    \nonumber &= \mathbb{E}[\nabla \ell(z_{i_k}; \boldsymbol{\theta}_k^{s+1}) - \nabla \ell(z_{i_k}; \widetilde{\boldsymbol{\theta}}^s) + \nabla \mathcal{J}(\widetilde{\boldsymbol{\theta}}^s; \boldsymbol{\theta}_k^{s+1})] - \mathbb{E}[\nabla \mathcal{J}(\widetilde{\boldsymbol{\theta}}^s; \boldsymbol{\theta}_k^{s+1}) - \nabla\mathcal{J}(\widetilde{\boldsymbol{\theta}}^s; \widetilde{\boldsymbol{\theta}}^s)] \\
%    \nonumber & = \mathbb{E}[\nabla \mathcal{J}(\boldsymbol{\theta}_k^{s+1}; \boldsymbol{\theta}_k^{s+1})] - \mathbb{E}[\nabla \mathcal{J}(\widetilde{\boldsymbol{\theta}}^s; \boldsymbol{\theta}_k^{s+1}) - \nabla\mathcal{J}(\widetilde{\boldsymbol{\theta}}^s; \widetilde{\boldsymbol{\theta}}^s)]
% \end{align}
{\begin{align}\label{eq:vk}
    \mathbb{E}[v_k^{s+1}] &=  \mathbb{E}[\nabla \ell(z_{i_k}(\boldsymbol{\theta}_k^{s+1}); \boldsymbol{\theta}_k^{s+1}) - \nabla \ell(z_{i_k}(\boldsymbol{\theta}_k^{s+1}); \widetilde{\boldsymbol{\theta}}^s) +\nabla\mathcal{J}(\widetilde{\boldsymbol{\theta}}^s; \widetilde{\boldsymbol{\theta}}^s)]\\
   \nonumber &= \mathbb{E}[\nabla \ell(z_{i_k}(\boldsymbol{\theta}_k^{s+1}); \boldsymbol{\theta}_k^{s+1}) - \nabla \ell(z_{i_k}(\boldsymbol{\theta}_k^{s+1}); \widetilde{\boldsymbol{\theta}}^s) + \nabla \mathcal{J}(\widetilde{\boldsymbol{\theta}}^s; \boldsymbol{\theta}_k^{s+1})] - \mathbb{E}[\nabla \mathcal{J}(\widetilde{\boldsymbol{\theta}}^s; \boldsymbol{\theta}_k^{s+1}) - \nabla\mathcal{J}(\widetilde{\boldsymbol{\theta}}^s; \widetilde{\boldsymbol{\theta}}^s)] \\
   \nonumber & = \mathbb{E}[\nabla \mathcal{J}(\boldsymbol{\theta}_k^{s+1}; \boldsymbol{\theta}_k^{s+1})] - \mathbb{E}[\nabla \mathcal{J}(\widetilde{\boldsymbol{\theta}}^s; \boldsymbol{\theta}_k^{s+1}) - \nabla\mathcal{J}(\widetilde{\boldsymbol{\theta}}^s; \widetilde{\boldsymbol{\theta}}^s)]
\end{align}}

The last equality holds since $\mathbb{E}[\nabla \ell(z_{i_k}; \widetilde{\boldsymbol{\theta}}^s)] =  \nabla \mathcal{J}(\widetilde{\boldsymbol{\theta}}^s, \boldsymbol{\theta}_k^{s+1})$. Then applying Lemma \ref{lemma:aux2} on the second term above to derive the following:
\begin{align*}
    \gamma_k\nabla\mathcal{J}(\boldsymbol{\theta}_k^{s+1}; \boldsymbol{\theta}_k^{s+1})^Tv_k^{s+1} &= \gamma_k\nabla\mathcal{J}(\boldsymbol{\theta}_k^{s+1}; \boldsymbol{\theta}_k^{s+1})^T \bigl(\nabla \mathcal{J}(\boldsymbol{\theta}_k^{s+1}; \boldsymbol{\theta}_k^{s+1}) - \mathbb{E}[\nabla \mathcal{J}(\widetilde{\boldsymbol{\theta}}^s; \boldsymbol{\theta}_k^{s+1}) - \nabla\mathcal{J}(\widetilde{\boldsymbol{\theta}}^s; \widetilde{\boldsymbol{\theta}}^s)] \bigr)\\
    &\ge \gamma_k \|\nabla\mathcal{J}(\boldsymbol{\theta}_k^{s+1}; \boldsymbol{\theta}_k^{s+1})\|^2 - L_0\epsilon\gamma_k \|\nabla\mathcal{J}(\boldsymbol{\theta}_k^{s+1}; \boldsymbol{\theta}_k^{s+1})\|\|\widetilde{\boldsymbol{\theta}}^s - \boldsymbol{\theta}_k^{s+1}\|
\end{align*}
Next, applying Young's Inequality, we can get:
\begin{align} \label{eq:boundJ}
\text{Eqn.}~\ref{eq:proof1_1} &\le \mathbb{E}[\mathcal{J}(\boldsymbol{\theta}_{k+1}^{s+1}; \boldsymbol{\theta}_{k}^{s+1}) - \gamma_k \|\nabla\mathcal{J}(\boldsymbol{\theta}_k^{s+1}; \boldsymbol{\theta}_k^{s+1})\|^2 + L_0\epsilon\gamma_k(\frac{1}{2\beta_k}\|\nabla\mathcal{J}(\boldsymbol{\theta}_k^{s+1}; \boldsymbol{\theta}_k^{s+1})\|^2 \\
\nonumber&~~+ \frac{\beta_k}{2}\|\boldsymbol{\theta}_k^{s+1} - \widetilde{\boldsymbol{\theta}}^s\|^2) + \frac{L\gamma_k^2}{2}\|v_k^{s+1}\|^2]
\end{align}

% Minimal corrected replacement for Eq. (14).
% Here \mathbb{E}_k denotes conditional expectation with respect to the
% history before drawing z_{i_k}.

Next, bound
$\mathbb{E}[\|\boldsymbol{\theta}^{s+1}_{k+1}
-\widetilde{\boldsymbol{\theta}}^s\|^2]$
by adding and subtracting as follows:
\begin{align}\label{eq:proof1_2}
\nonumber
\mathbb{E}\!\left[
\left\|\boldsymbol{\theta}^{s+1}_{k+1}
-\widetilde{\boldsymbol{\theta}}^s\right\|^2
\right]
&=
\mathbb{E}\!\left[
\left\|\boldsymbol{\theta}^{s+1}_{k+1}
-\boldsymbol{\theta}_k^{s+1}
+\boldsymbol{\theta}_k^{s+1}
-\widetilde{\boldsymbol{\theta}}^s\right\|^2
\right] \\
\nonumber
&=
\mathbb{E}\!\left[
\gamma_k^2\|v_k^{s+1}\|^2
+\|\boldsymbol{\theta}_k^{s+1}
-\widetilde{\boldsymbol{\theta}}^s\|^2
\right] \\
\nonumber
&\quad
-2\gamma_k\mathbb{E}\!\left[
\left\langle
\boldsymbol{\theta}_k^{s+1}
-\widetilde{\boldsymbol{\theta}}^s,
v_k^{s+1}
\right\rangle
\right] \\
&\le
\mathbb{E}\!\left[
\gamma_k^2\|v_k^{s+1}\|^2
+\|\boldsymbol{\theta}_k^{s+1}
-\widetilde{\boldsymbol{\theta}}^s\|^2
\right] \\
\nonumber
&\quad
+2\gamma_k\mathbb{E}\!\left[
\frac{1}{2\beta_k}
\|\nabla\mathcal{J}(
\boldsymbol{\theta}_{k}^{s+1};
\boldsymbol{\theta}_{k}^{s+1})\|^2
+\frac{\beta_k}{2}
\|\boldsymbol{\theta}_k^{s+1}
-\widetilde{\boldsymbol{\theta}}^s\|^2
\right] \\
\nonumber
&\quad
+2L_0\epsilon\gamma_k
\mathbb{E}\!\left[
\|\widetilde{\boldsymbol{\theta}}^s
-\boldsymbol{\theta}_k^{s+1}\|^2
\right].
\end{align}

The last inequality follows from the tower property and the conditional
form of Eq.~(12). Since
$\boldsymbol{\theta}_k^{s+1}-\widetilde{\boldsymbol{\theta}}^s$
is measurable with respect to the history before drawing the index $i_k$,
\begin{align*}
&-2\gamma_k\mathbb{E}\!\left[
\left\langle
\boldsymbol{\theta}_k^{s+1}-\widetilde{\boldsymbol{\theta}}^s,
v_k^{s+1}
\right\rangle
\right]\\
&=
-2\gamma_k\mathbb{E}\!\left[
\left\langle
\boldsymbol{\theta}_k^{s+1}-\widetilde{\boldsymbol{\theta}}^s,
\mathbb{E}_k[v_k^{s+1}]
\right\rangle
\right]\\
&=
-2\gamma_k\mathbb{E}\!\left[
\left\langle
\boldsymbol{\theta}_k^{s+1}-\widetilde{\boldsymbol{\theta}}^s,
\nabla\mathcal{J}(
\boldsymbol{\theta}_k^{s+1};
\boldsymbol{\theta}_k^{s+1})
\right\rangle
\right]\\
&\quad+
2\gamma_k\mathbb{E}\!\left[
\left\langle
\boldsymbol{\theta}_k^{s+1}-\widetilde{\boldsymbol{\theta}}^s,
\nabla\mathcal{J}(
\widetilde{\boldsymbol{\theta}}^s;
\boldsymbol{\theta}_k^{s+1})
-
\nabla\mathcal{J}(
\widetilde{\boldsymbol{\theta}}^s;
\widetilde{\boldsymbol{\theta}}^s)
\right\rangle
\right]\\
&\le
2\gamma_k\mathbb{E}\!\left[
\frac{1}{2\beta_k}
\left\|
\nabla\mathcal{J}(
\boldsymbol{\theta}_k^{s+1};
\boldsymbol{\theta}_k^{s+1})
\right\|^2
+
\frac{\beta_k}{2}
\left\|
\boldsymbol{\theta}_k^{s+1}
-\widetilde{\boldsymbol{\theta}}^s
\right\|^2
\right] +
2L_0\epsilon\gamma_k\mathbb{E}\!\left[
\left\|
\boldsymbol{\theta}_k^{s+1}
-\widetilde{\boldsymbol{\theta}}^s
\right\|^2
\right],
\end{align*}
where the last line follows from Young's inequality and Lemma~\ref{lemma:aux2}.
\color{black}

% Next, bound $\mathbb{E}[\|\boldsymbol{\theta}^{s+1}_{k+1} - \widetilde{\boldsymbol{\theta}}^s\|^2]$ by adding and subtracting as follows:
% \begin{align}\label{eq:proof1_2}
%     \nonumber\mathbb{E}[\|\boldsymbol{\theta}^{s+1}_{k+1} - \widetilde{\boldsymbol{\theta}}^s\|^2] &= \mathbb{E}[\|\boldsymbol{\theta}^{s+1}_{k+1} - \boldsymbol{\theta}_k^{s+1} + \boldsymbol{\theta}_k^{s+1} - \widetilde{\boldsymbol{\theta}}^s\|^2] \\
%     \nonumber &= \mathbb{E}[\gamma_k^2\|v_k^{s+1}\|^2 + \|\boldsymbol{\theta}_k^{s+1} - \widetilde{\boldsymbol{\theta}}^s\|^2] + 2\gamma_k\mathbb{E}[\|\boldsymbol{\theta}_k^{s+1} - \widetilde{\boldsymbol{\theta}}^s\|\|v_k^{s+1}\|]\\
%     &\le \mathbb{E}[\gamma_k^2\|v_k^{s+1}\|^2 + \|\boldsymbol{\theta}_k^{s+1} - \widetilde{\boldsymbol{\theta}}^s\|^2] + 2\gamma_k\mathbb{E}[\frac{1}{2\beta_k}\|\nabla\mathcal{J}(\boldsymbol{\theta}_{k}^{s+1}; \boldsymbol{\theta}_{k}^{s+1})\|^2 \\
%     \nonumber & ~~+\frac{\beta_k}{2}\|\boldsymbol{\theta}_k^{s+1} - \widetilde{\boldsymbol{\theta}}^s\|^2] + 2L_0\epsilon\gamma_k\mathbb{E}\|\widetilde{\boldsymbol{\theta}}^s - \boldsymbol{\theta}_k^{s+1}\|^2
% \end{align}

Similarly, $\mathbb{E}[\|v_k^{s+1}\|^2]$ can also be bounded:
{\begin{align}\label{eq:proof1_3}
 \mathbb{E}[\|v_k^{s+1}\|^2] =& \mathbb{E} \left[ \|\nabla \ell(z_{i_k}(\boldsymbol{\theta}_k^{s+1}); \boldsymbol{\theta}^{s+1}_k) - \nabla \ell(z_{i_k}(\boldsymbol{\theta}_k^{s+1}); \widetilde{\boldsymbol{\theta}}^s) + \nabla \mathcal{J}(\widetilde{\boldsymbol{\theta}}^s; \widetilde{\boldsymbol{\theta}}^s) \| ^2\right] \\
    \nonumber \le& 2 \mathbb{E} \left[\| \nabla \ell(z_{i_k}(\boldsymbol{\theta}_k^{s+1}); \boldsymbol{\theta}^{s+1}_k) - \nabla \ell(z_{i_k}(\boldsymbol{\theta}_k^{s+1}); \widetilde{\boldsymbol{\theta}}^s) + \nabla \mathcal{J}(\widetilde{\boldsymbol{\theta}}^s; \boldsymbol{\theta}_k^{s+1})\|^2 \right] \\ \nonumber &+ 2 \mathbb{E} \left[\| \nabla \mathcal{J}(\widetilde{\boldsymbol{\theta}}^s; \boldsymbol{\theta}_k^{s+1})- \nabla \mathcal{J}(\widetilde{\boldsymbol{\theta}}^s; \widetilde{\boldsymbol{\theta}}^s) \|^2\right]
\end{align}}
% \begin{align}\label{eq:proof1_3}
%  \mathbb{E}[\|v_k^{s+1}\|^2] &= \mathbb{E} \left[ \|\nabla \ell(z_{i_k}; \boldsymbol{\theta}^{s+1}_k) - \nabla \ell(z_{i_k}; \widetilde{\boldsymbol{\theta}}^s) + \nabla \mathcal{J}(\widetilde{\boldsymbol{\theta}}^s; \widetilde{\boldsymbol{\theta}}^s) \| ^2\right] \\
%     \nonumber &\le 2 \mathbb{E} \left[\| \nabla \ell(z_{i_k}; \boldsymbol{\theta}^{s+1}_k) - \nabla \ell(z_{i_k}; \widetilde{\boldsymbol{\theta}}^s) + \nabla \mathcal{J}(\widetilde{\boldsymbol{\theta}}^s; \boldsymbol{\theta}_k^{s+1})\|^2 \right] + 2 \mathbb{E} \left[\| \nabla \mathcal{J}(\widetilde{\boldsymbol{\theta}}^s; \boldsymbol{\theta}_k^{s+1})- \nabla \mathcal{J}(\widetilde{\boldsymbol{\theta}}^s; \widetilde{\boldsymbol{\theta}}^s) \|^2\right]
% \end{align}
The above inequality is obtained by adding and subtracting $\nabla \mathcal{J}(\widetilde{\boldsymbol{\theta}}^s; \boldsymbol{\theta}_k^{s+1})$ and then applying Young's Inequality. For the first term, it is worth noting that the term already gets rid of the performative effects, i.e., the distribution is fixed as $\mathcal{D}(\boldsymbol{\theta}_{k}^{s+1})$. Then bounding this term is equivalent to bound the norm of update in \citet{reddi2016stochastic}. Specifically, let {$\delta_k^{s+1} =  \nabla \ell(z_{i_k}(\boldsymbol{\theta}_k^{s+1}); \boldsymbol{\theta}^{s+1}_k) - \nabla \ell(z_{i_k}(\boldsymbol{\theta}_k^{s+1}); \widetilde{\boldsymbol{\theta}}^s)$}, we know {$\mathbb{E}[\delta_k^{s+1}] = \mathbb{E}[\nabla\ell(z_{i_k}(\boldsymbol{\theta}_k^{s+1}); \boldsymbol{\theta}^{s+1}_k)] - \mathbb{E}[\nabla \ell(z_{i_k}(\boldsymbol{\theta}_k^{s+1}); \widetilde{\boldsymbol{\theta}}^s)] =  \nabla \mathcal{J}(\boldsymbol{\theta}_k^{s+1}; \boldsymbol{\theta}_k^{s+1}) -  \nabla \mathcal{J}(\widetilde{\boldsymbol{\theta}}^s; \boldsymbol{\theta}_k^{s+1})$}. Thus, we have:
{\begin{align}\label{eq:term1bound}
    &\mathbb{E} \left[ \|\nabla \ell(z_{i_k}(\boldsymbol{\theta}_k^{s+1}); \boldsymbol{\theta}^{s+1}_k) - \nabla \ell(z_{i_k}(\boldsymbol{\theta}_k^{s+1}); \widetilde{\boldsymbol{\theta}}^s) + \nabla \mathcal{J}(\widetilde{\boldsymbol{\theta}}^s; \boldsymbol{\theta}_k^{s+1})\|^2 \right]  \\
    \nonumber &= \mathbb{E}[\|\delta_k^{s+1} +  \nabla \mathcal{J}(\widetilde{\boldsymbol{\theta}}^s; \boldsymbol{\theta}_k^{s+1})\|^2]\\
    \nonumber &=  \mathbb{E}[\|\delta_k^{s+1} +  \nabla \mathcal{J}(\widetilde{\boldsymbol{\theta}}^s; \boldsymbol{\theta}_k^{s+1}) - \nabla \mathcal{J}(\boldsymbol{\theta}_k^{s+1}; \boldsymbol{\theta}_k^{s+1}) + \nabla \mathcal{J}(\boldsymbol{\theta}_k^{s+1}; \boldsymbol{\theta}_k^{s+1})\|^2]\\
    \nonumber &= \mathbb{E}[\|\delta_k^{s+1} - \mathbb{E}[\delta_k^{s+1}] + \nabla \mathcal{J}(\boldsymbol{\theta}_k^{s+1}; \boldsymbol{\theta}_k^{s+1})\|^2] \\
    \nonumber &\le 2 \mathbb{E}[\|\delta_k^{s+1} - \mathbb{E}[\delta_k^{s+1}]\|^2] + 2 \mathbb{E}[\|\nabla \mathcal{J}(\boldsymbol{\theta}_k^{s+1}; \boldsymbol{\theta}_k^{s+1})\|^2]\\
    \nonumber & \le 2\mathbb{E}[\|\delta_k^{s+1}\|^2] + 2 \mathbb{E}[\|\nabla \mathcal{J}(\boldsymbol{\theta}_k^{s+1}; \boldsymbol{\theta}_k^{s+1})\|^2] \\
    \nonumber & \le 2 L^2 \mathbb{E} [\|\widetilde{\boldsymbol{\theta}}_k^{s+1} - \widetilde{\boldsymbol{\theta}}^s\|^2] + 2 \mathbb{E}[\|\nabla \mathcal{J}(\boldsymbol{\theta}_k^{s+1}; \boldsymbol{\theta}_k^{s+1})\|^2]
\end{align}}
% \begin{align}\label{eq:term1bound}
%     &\mathbb{E} \left[ \|\nabla \ell(z_{i_k}; \boldsymbol{\theta}^{s+1}_k) - \nabla \ell(z_{i_k}; \widetilde{\boldsymbol{\theta}}^s) + \nabla \mathcal{J}(\widetilde{\boldsymbol{\theta}}^s; \boldsymbol{\theta}_k^{s+1})\|^2 \right]  \\
%     \nonumber &= \mathbb{E}[\|\delta_k^{s+1} +  \nabla \mathcal{J}(\widetilde{\boldsymbol{\theta}}^s; \boldsymbol{\theta}_k^{s+1})\|^2]\\
%     \nonumber &=  \mathbb{E}[\|\delta_k^{s+1} +  \nabla \mathcal{J}(\widetilde{\boldsymbol{\theta}}^s; \boldsymbol{\theta}_k^{s+1}) - \nabla \mathcal{J}(\boldsymbol{\theta}_k^{s+1}; \boldsymbol{\theta}_k^{s+1}) + \nabla \mathcal{J}(\boldsymbol{\theta}_k^{s+1}; \boldsymbol{\theta}_k^{s+1})\|^2]\\
%     \nonumber &= \mathbb{E}[\|\delta_k^{s+1} - \mathbb{E}[\delta_k^{s+1}] + \nabla \mathcal{J}(\boldsymbol{\theta}_k^{s+1}; \boldsymbol{\theta}_k^{s+1})\|^2] \\
%     \nonumber &\le 2 \mathbb{E}[\|\delta_k^{s+1} - \mathbb{E}[\delta_k^{s+1}]\|^2] + 2 \mathbb{E}[\|\nabla \mathcal{J}(\boldsymbol{\theta}_k^{s+1}; \boldsymbol{\theta}_k^{s+1})\|^2]\\
%     \nonumber & \le 2\mathbb{E}[\|\delta_k^{s+1}\|^2] + 2 \mathbb{E}[\|\nabla \mathcal{J}(\boldsymbol{\theta}_k^{s+1}; \boldsymbol{\theta}_k^{s+1})\|^2] \\
%     \nonumber & \le 2 L^2 \mathbb{E} [\|\widetilde{\boldsymbol{\theta}}_k^{s+1} - \widetilde{\boldsymbol{\theta}}^s\|^2] + 2 \mathbb{E}[\|\nabla \mathcal{J}(\boldsymbol{\theta}_k^{s+1}; \boldsymbol{\theta}_k^{s+1})\|^2]
% \end{align}
By Lemma \ref{lemma:aux2}, we know $\| \nabla \mathcal{J}(\widetilde{\boldsymbol{\theta}}^s; \boldsymbol{\theta}_k^{s+1})- \nabla \mathcal{J}(\widetilde{\boldsymbol{\theta}}^s; \widetilde{\boldsymbol{\theta}}^s)\| \le L_0\epsilon\|\widetilde{\boldsymbol{\theta}}_k^{s+1} - \widetilde{\boldsymbol{\theta}}^s\|$ always holds. So we have:
\begin{align}
    \mathbb{E} \left[ \|\nabla \mathcal{J}(\widetilde{\boldsymbol{\theta}}^s; \boldsymbol{\theta}_k^{s+1})- \nabla \mathcal{J}(\widetilde{\boldsymbol{\theta}}^s; \widetilde{\boldsymbol{\theta}}^s)\|^2 \right] \le L_0^2\epsilon^2 \mathbb{E} \|\boldsymbol{\theta}_k^{s+1} - \widetilde{\boldsymbol{\theta}}^s\|^2 
\end{align}
Sum them together, we have:
\begin{align}\label{eq:vk2bound}
     \mathbb{E}[\|v_k^{s+1}\|^2]  \le 4 \mathbb{E}[\|\nabla \mathcal{J}(\boldsymbol{\theta}_k^{s+1}; \boldsymbol{\theta}_k^{s+1})\|^2] + (4L^2 + 2L_0^2\epsilon^2) \mathbb{E} \|\boldsymbol{\theta}_k^{s+1} - \widetilde{\boldsymbol{\theta}}^s\|^2
\end{align}
Now consider $R_{k+1}^{s+1}$ and plug Eqn.~\ref{eq:proof1_1} and Eqn.~\ref{eq:proof1_2} into its formula, we have:
\begin{align}\label{eq:proof1_4}
R_{k+1}^{s+1}(\boldsymbol{\theta}_k^{s+1}) &= \mathbb{E} \left[ \mathcal{J}(\boldsymbol{\theta}_{k+1}^{s+1}; \boldsymbol{\theta}_k^{s+1}) + c_{k+1} \|\widetilde{\boldsymbol{\theta}}_{k+1}^{s+1} - \widetilde{\boldsymbol{\theta}}^s\|^2 \right] \nonumber \\
\nonumber &\le \mathbb{E}[\mathcal{J}(\boldsymbol{\theta}_k^{s+1}; \boldsymbol{\theta}_{k}^{s+1}) - \gamma_k \|\nabla\mathcal{J}(\boldsymbol{\theta}_k^{s+1}, \boldsymbol{\theta}_k^{s+1})\|^2 + L_0\epsilon\gamma_k(\frac{1}{2\beta_k}\|\nabla\mathcal{J}(\boldsymbol{\theta}_k^{s+1}; \boldsymbol{\theta}_k^{s+1})\|^2 + \frac{\beta_k}{2}\|\boldsymbol{\theta}_k^{s+1} - \widetilde{\boldsymbol{\theta}}^s\|^2) \\
\nonumber &~~+ \frac{L\gamma_k^2}{2}\|v_k^{s+1}\|^2] + \mathbb{E} [ c_{k+1} \gamma_k^2 \|v_k^{s+1}\|^2 + c_{k+1} \|\boldsymbol{\theta}_{k+1}^{s+1} - \widetilde{\boldsymbol{\theta}}^s\|^2 + 2 c_{k+1} \gamma_k \mathbb{E} ( \frac{1}{2 \beta_k} \|\nabla \mathcal{J}(\boldsymbol{\theta}_k^{s+1}; \boldsymbol{\theta}_k^{s+1})\|^2  \\ 
\nonumber &~~ + \frac{\beta_k}{2} \|\boldsymbol{\theta}_k^{s+1}
- \widetilde{\boldsymbol{\theta}}^s\|^2 ) ] + 2 c_{k+1} L_0 \epsilon \gamma_k \mathbb{E} \left\| \widetilde{\boldsymbol{\theta}}^s - \boldsymbol{\theta}_k^{s+1} \right\|^2\\
\nonumber &= \mathbb{E} \left[ \mathcal{J}(\boldsymbol{\theta}_k^{s+1}; \boldsymbol{\theta}_k^{s+1}) \right] 
 + \left( \frac{L \gamma_k^2}{2} + c_{k+1} \gamma_k^2 \right)\mathbb{E} \left[ \|v_k^{s+1}\|^2 \right] \\
 \nonumber &~~-\left( \gamma_k - \frac{c_{k+1} \gamma_k + \frac{1}{2} L_0 \epsilon \gamma_k}{\beta_k} \right) 
\mathbb{E} \left[ \|\nabla \mathcal{J}(\boldsymbol{\theta}_k^{s+1}; \boldsymbol{\theta}_k^{s+1})\|^2 \right] \\
\nonumber &~~+ \left( c_{k+1}(1+ \gamma_k \beta_k + 2 L_0 \epsilon \gamma_k) + \frac{\beta_k L_0 \epsilon \gamma_k}{2} \right) 
\mathbb{E} \left[ \|\widetilde{\boldsymbol{\theta}}^s - \boldsymbol{\theta}_k^{s+1}\|^2 \right]
\end{align}
Regarding the $\|v_k^{s+1}\|^2$, we plug Eqn.~\ref{eq:vk2bound} in Eqn.~\ref{eq:proof1_4} to get the expression we need:
\begin{align*}
    R_{k+1}^{s+1}(\boldsymbol{\theta}_k^{s+1}) &\leq \mathbb{E} \left[ \mathcal{J}(\boldsymbol{\theta}_k^{s+1}; \boldsymbol{\theta}_k^{s+1}) \right] - \Gamma_k \mathbb{E} \left[ \|\nabla \mathcal{J}(\boldsymbol{\theta}_k^{s+1}; \boldsymbol{\theta}_k^{s+1})\|^2 \right]  + (c_k - \frac{1}{2}\gamma_k) \mathbb{E}\|\boldsymbol{\theta}_k^{s+1} - \widetilde{\boldsymbol{\theta}}^s\|^2\\
    \nonumber&= R_k^{s+1} - \Gamma_k \mathbb{E} \left[ \|\nabla \mathcal{J}(\boldsymbol{\theta}_k^{s+1}; \boldsymbol{\theta}_k^{s+1})\|^2 \right]  - \frac{1}{2}\gamma_k \mathbb{E}\|\boldsymbol{\theta}_k^{s+1} - \widetilde{\boldsymbol{\theta}}^s\|^2
\end{align*}
\normalsize
where
\begin{align*}
   c_k - \frac{1}{2}\gamma_k &= c_{k+1}(1+ \gamma_k \beta_k + 2  L_0 \epsilon \gamma_k) + (2L^2+ L_0^2 \epsilon^2 ) (L\gamma_k^2 + 2c_{k+1}\gamma_k^2) + \frac{\beta_k}{2}L_0\epsilon\gamma_k\\
     \Gamma_k &= \left( \gamma_k - \frac{c_{k+1} \gamma_k + \frac{1}{2} L_0 \epsilon \gamma_k}{\beta_k} - 2L \gamma_k^2 - 4 c_{k+1} \gamma_k^2 \right)
\end{align*}

\end{proof}

\section{Proof of Lemma \ref{lemma:r2}}\label{app:lemmar2}

\begin{lemma*}

With $c_k, c_{k+1}, \beta_k, \Gamma_k$ same as Lemma \ref{lemma:r}, denote 
\begin{align*}
    \widetilde{R}_k^{s+1} \overset{\triangle}{=} \mathbb{E}[\mathcal{J}(\boldsymbol{\theta}_k^{s+1}; \boldsymbol{\theta}_k^{s+1}) + c_k \| \boldsymbol{\theta}_k^{s+1} - \widetilde{\boldsymbol{\theta}}^s\|^2]\\
    \widetilde{R}_{k+1}^{s+1} \overset{\triangle}{=} \mathbb{E}[\mathcal{J}(\boldsymbol{\theta}_{k+1}^{s+1}; \boldsymbol{\theta}_{k+1}^{s+1}) + c_{k+1} \| \boldsymbol{\theta}_{k+1}^{s+1} - \widetilde{\boldsymbol{\theta}}^s\|^2]
\end{align*}
Then we have:
\begin{align}
    \widetilde{R}_{k+1}^{s+1} &\le \widetilde{R}_k^{s+1} + (L_0^2L_1^2\epsilon^4 +2 L_1^2\epsilon^2L^2 + 2L_1^2\epsilon^2)\gamma_k - (\Gamma_k - \frac{\gamma_k}{4})\cdot \mathbb{E}[\|\nabla \mathcal{J}(\boldsymbol{\theta}_k^{s+1}; \boldsymbol{\theta}_k^{s+1})\|^2]
\end{align}
\end{lemma*}

\begin{proof}

This proof mainly relies on Assumptions \ref{assumption: smoothness}, \ref{assumption:lipschitzness}, and $\epsilon$-sensitivity.

According to Lemma \ref{lemma:r}, we have:
\begin{align}\label{eq:lemmaend}
\mathbb{E} [ \mathcal{J}(\boldsymbol{\theta}_{k+1}^{s+1}; \boldsymbol{\theta}_k^{s+1})
+ c_{k+1} \|\boldsymbol{\theta}_{k+1}^{s+1} - \widetilde{\boldsymbol{\theta}}^s\|^2 ]
&\leq \mathbb{E}[\mathcal{J}(\boldsymbol{\theta}_k^{s+1}; \boldsymbol{\theta}_k^{s+1})] + c_k \|\boldsymbol{\theta}_k^{s+1} - \widetilde{\boldsymbol{\theta}}^s\|^2 \\
\nonumber &- \Gamma_k \|\nabla \mathcal{J}(\boldsymbol{\theta}_k^{s+1}; \boldsymbol{\theta}_k^{s+1})\|^2 
- \frac{\gamma_k}{2} \|\boldsymbol{\theta}_k^{s+1} - \widetilde{\boldsymbol{\theta}}^s\|^2 ]
\end{align}
By aggregating terms, we will have:
\begin{align}\label{eq:agg}
&~~~~c_{k+1} \mathbb{E}\|\boldsymbol{\theta}_{k+1}^{s+1} - \widetilde{\boldsymbol{\theta}}^s\|^2 
- \left( c_k - \frac{\gamma_k}{2} \right) \mathbb{E} \|\boldsymbol{\theta}_k^{s+1} - \widetilde{\boldsymbol{\theta}}^s\|^2 + \Gamma_k \mathbb{E} \|\nabla \mathcal{J}(\boldsymbol{\theta}_k^{s+1}; \boldsymbol{\theta}_k^{s+1})\|^2 \\
\nonumber & \le \mathbb{E}[\mathcal{J}(\boldsymbol{\theta}_k^{s+1}; \boldsymbol{\theta}_k^{s+1}) - \mathcal{J}(\boldsymbol{\theta}_{k+1}^{s+1}; \boldsymbol{\theta}_k^{s+1})]
\end{align}

We can further bound the above expression by:
\begin{align}\label{eq:bound_second}
 \mathbb{E}[\mathcal{J}(\boldsymbol{\theta}_k^{s+1}; \boldsymbol{\theta}_k^{s+1}) - \mathcal{J}(\boldsymbol{\theta}_{k+1}^{s+1}; \boldsymbol{\theta}_k^{s+1})] &= \mathbb{E} \left[ \mathcal{J}(\boldsymbol{\theta}_k^{s+1}; \boldsymbol{\theta}_k^{s+1}) - \mathcal{J}(\boldsymbol{\theta}_{k+1}^{s+1}; \boldsymbol{\theta}_{k+1}^{s+1}) \right] \\
\nonumber &~~ + \mathbb{E} \left[ \mathcal{J}(\boldsymbol{\theta}_{k+1}^{s+1}; \boldsymbol{\theta}_{k+1}^{s+1}) - \mathcal{J}(\boldsymbol{\theta}_{k+1}^{s+1}; \boldsymbol{\theta}_{k}^{s+1}) \right] \\
\nonumber &\le \mathbb{E} \left[ \mathcal{J}(\boldsymbol{\theta}_k^{s+1}; \boldsymbol{\theta}_k^{s+1}) - \mathcal{J}(\boldsymbol{\theta}_{k+1}^{s+1}; \boldsymbol{\theta}_{k+1}^{s+1}) \right]  \\
\nonumber &~~ + {L}_1\epsilon \mathbb{E}\|\boldsymbol{\theta}_{k+1}^{s+1} - \boldsymbol{\theta}_k^{s+1}\|
\end{align}
Regarding $\mathbb{E}\|\boldsymbol{\theta}_{k+1}^{s+1} - \boldsymbol{\theta}_k^{s+1}\| = \gamma_k \mathbb{E}\|v_k^{s+1}\|$, we can bound it according to Eqn. \ref{eq:vk} and Eqn. \ref{eq:vk2bound}:
{\begin{align}
    \mathbb{E}[\|v_k^{s+1}\|] &=  \mathbb{E}[\|\nabla \ell(z_{i_k}(\boldsymbol{\theta}_k^{s+1}); \boldsymbol{\theta}_k^{s+1}) - \nabla \ell(z_{i_k}(\boldsymbol{\theta}_k^{s+1}); \widetilde{\boldsymbol{\theta}}^s) +\nabla\mathcal{J}(\widetilde{\boldsymbol{\theta}}^s; \widetilde{\boldsymbol{\theta}}^s)\|]\\
   \nonumber &\le \mathbb{E}[\|\nabla \ell(z_{i_k}(\boldsymbol{\theta}_k^{s+1}); \boldsymbol{\theta}_k^{s+1}) - \nabla \ell(z_{i_k}(\boldsymbol{\theta}_k^{s+1}); \widetilde{\boldsymbol{\theta}}^s) + \nabla \mathcal{J}(\widetilde{\boldsymbol{\theta}}^s; \boldsymbol{\theta}_k^{s+1})\|] \\
    \nonumber&~~~~+ \mathbb{E}[\|\nabla \mathcal{J}(\widetilde{\boldsymbol{\theta}}^s; \boldsymbol{\theta}_k^{s+1}) - \nabla\mathcal{J}(\widetilde{\boldsymbol{\theta}}^s; \widetilde{\boldsymbol{\theta}}^s)\|] 
\end{align}}
Let $\nabla \ell(z_{i_k}; \boldsymbol{\theta}_k^{s+1}) - \nabla \ell(z_{i_k}; \widetilde{\boldsymbol{\theta}}^s) + \nabla \mathcal{J}(\widetilde{\boldsymbol{\theta}}^s; \boldsymbol{\theta}_k^{s+1})$ be $A$ and $\nabla \mathcal{J}(\widetilde{\boldsymbol{\theta}}^s; \boldsymbol{\theta}_k^{s+1}) - \nabla\mathcal{J}(\widetilde{\boldsymbol{\theta}}^s; \widetilde{\boldsymbol{\theta}}^s)$ be $B$. 
According to Eqn. \ref{eq:term1bound}, we have:
\begin{align*}
    \mathbb{E}[\|A\|^2] \le 2 \mathbb{E} \left[ \|\nabla \mathcal{J}(\boldsymbol{\theta}_k^{s+1}; \boldsymbol{\theta}_k^{s+1})\|^2 \right] + 2 L^2 \mathbb{E} \|\boldsymbol{\theta}_k^{s+1} - \widetilde{\boldsymbol{\theta}}^s\|^2
\end{align*}
which means
\begin{align*}
    \mathbb{E}[\|A\|] \le \sqrt{\mathbb{E}[\|A\|^2]}  \le \sqrt{2 \mathbb{E} \left[ \|\nabla \mathcal{J}(\boldsymbol{\theta}_k^{s+1}; \boldsymbol{\theta}_k^{s+1})\|^2 \right] + 2L^2 \mathbb{E} \|\boldsymbol{\theta}_k^{s+1} - \widetilde{\boldsymbol{\theta}}^s\|^2}
\end{align*}
Applying Young's Inequality on rhs, for any positive $\lambda$ we have:
\begin{align}
   \mathbb{E}\|A\| \le \frac{1}{\lambda} \left( \mathbb{E} \left[ \|\nabla \mathcal{J}(\boldsymbol{\theta}_k^{s+1}; \boldsymbol{\theta}_k^{s+1})\|^2 \right] + L^2 \mathbb{E} \left\| \boldsymbol{\theta}_k^{s+1} - \widetilde{\boldsymbol{\theta}}^s \right\|^2 \right) + \frac{\lambda}{2}
\end{align}
Let $\lambda_0 = 4L_1\epsilon(L^2 + 1)$, we have:
\begin{align}
 \mathbb{E}\|A\| &\le \frac{1}{\lambda_0} \left( \mathbb{E} \left[ \|\nabla \mathcal{J}(\boldsymbol{\theta}_k^{s+1}; \boldsymbol{\theta}_k^{s+1})\|^2 \right] + L^2\mathbb{E} \left\| \boldsymbol{\theta}_k^{s+1} - \widetilde{\boldsymbol{\theta}}^s \right\|^2 \right) + \frac{\lambda_0}{2} \\
 \nonumber & \le \frac{1}{4L_1\epsilon} \mathbb{E} \left[ \|\nabla \mathcal{J}(\boldsymbol{\theta}_k^{s+1}; \boldsymbol{\theta}_k^{s+1})\|^2 \right] + \frac{1}{4L_1\epsilon} \mathbb{E} \left\| \boldsymbol{\theta}_k^{s+1} - \widetilde{\boldsymbol{\theta}}^s \right\|^2 + \frac{\lambda_0}{2}
\end{align}
According to the point-wise Lipschitzness in Lemma \ref{lemma:aux2}, we have:
\begin{align}
  \mathbb{E}[\|B\|] \le \sqrt{\mathbb{E}[\|B\|^2]} \le \sqrt{L_0^2\epsilon^2 \mathbb{E} \left\| \boldsymbol{\theta}_k^{s+1} - \boldsymbol{\theta}^s \right\|^2} \le \frac{\mathbb{E} \left\| \boldsymbol{\theta}_k^{s+1} - \widetilde{\boldsymbol{\theta}}^s \right\|^2}{4L_1\epsilon} + L_0^2 L_1\epsilon^3  
\end{align}
Thus, we have
\begin{align}\label{eq:bound_last}
    \mathbb{E}[\|v_k^{s+1}\|] &\le \mathbb{E}\|A\|  + \mathbb{E} \|B\| \\
    \nonumber &\le \frac{\mathbb{E} \|\nabla \mathcal{J}(\boldsymbol{\theta}_k^{s+1}; \boldsymbol{\theta}_k^{s+1})\|^2}{4L_1\epsilon} + (\frac{1}{4L_1\epsilon} + \frac{1}{4L_1\epsilon}) \, \mathbb{E} \left\| \boldsymbol{\theta}_k^{s+1} - \widetilde{\boldsymbol{\theta}}^s \right\|^2  + \frac{\lambda_0}{2} +  L_0^2 L_1\epsilon^3 \\
    \nonumber &= \frac{\mathbb{E} \|\nabla \mathcal{J}(\boldsymbol{\theta}_k^{s+1}; \boldsymbol{\theta}_k^{s+1})\|^2}{4L_1\epsilon} + \frac{1}{2L_1\epsilon}\mathbb{E} \left\| \boldsymbol{\theta}_k^{s+1} - \widetilde{\boldsymbol{\theta}}^s \right\|^2 + (\frac{\lambda_0}{2} +  L_0^2 L_1\epsilon^3) 
\end{align}

Finally, we plug Eqn. \ref{eq:bound_last} into Eqn. \ref{eq:bound_second}, and then plug Eqn. \ref{eq:bound_second} into Eqn. \ref{eq:agg}. This leads to the inequality relationship between $\mathcal{J}(\boldsymbol{\theta}_{k+1}^{s+1}; \boldsymbol{\theta}_{k+1}^{s+1})$ and $\mathcal{J}(\boldsymbol{\theta}_k^{s+1}; \boldsymbol{\theta}_k^{s+1})$.
\begin{align*}
    &\mathbb{E} \left[ \mathcal{J}(\boldsymbol{\theta}_{k+1}^{s+1}; \boldsymbol{\theta}_{k+1}^{s+1})+ c_{k+1}\|\boldsymbol{\theta}_{k+1}^{s+1} - \widetilde{\boldsymbol{\theta}}^s\|^2 \right]  \le \mathbb{E} \left[ \mathcal{J}(\boldsymbol{\theta}_{k}^{s+1}; \boldsymbol{\theta}_{k}^{s+1}) + c_k \|\boldsymbol{\theta}_k^{s+1} - \widetilde{\boldsymbol{\theta}}^s\|^2 \right] \\
    &~~-(\Gamma_k - \frac{\gamma_k}{4})\cdot \mathbb{E}[\|\nabla \mathcal{J}(\boldsymbol{\theta}_k^{s+1}; \boldsymbol{\theta}_k^{s+1})\|^2] + (L_0^2 L_1^2\epsilon^4 + 2L_1^2\epsilon^2L^2 + 2L_1^2\epsilon^2)\gamma_k
\end{align*}
which is exactly what we want.
\end{proof}

\section{Proof of Thm. \ref{theorem:converge}}\label{app:thmconverge}

\begin{theorem*}
Let the total rounds be $T$ and the epochs be $S = \lceil T/m \rceil$. {Assume a constant step size $\gamma_k=\gamma$ for all $k$.}
For the last epoch, $c_m = 0$. Define $\Delta_0$ and $\Delta_1$ as follows:

\begin{align*}
    \Delta_0 \overset{\triangle}{=} \mathcal{J}(\boldsymbol{\theta}_0^0; \boldsymbol{\theta}_0^0) - \ell_*; ~~~~~
    \Delta_1 \overset{\triangle}{=}  \frac{(L_0^2L_1^2\epsilon^4 + 2L_1^2\epsilon^2L^2 + 2L_1^2\epsilon^2)\sum_{s=0}^{S}\sum_{k=0}^{m}\gamma_k}{T\Gamma}
\end{align*}
where $\Gamma > 0$ is guaranteed to exist as a constant and is the lower bound of $\Gamma_k - \frac{\gamma_k}{4}$. Then we have:
\begin{align}
    \frac{1}{T} \sum_{s=0}^{S-1} \sum_{k=0}^{m-1} \mathbb{E}\|\nabla \mathcal{J}(\boldsymbol{\theta}_k^{s+1}; \boldsymbol{\theta}_k^{s+1})\|^2 \le \frac{\Delta_0}{T\Gamma} + \Delta_1
\end{align}
Thus, Alg. \ref{alg:one} can achieve $\mathcal{O}(\frac{1}{T})$ convergence rate and the error neighborhood $\Delta_1$ is $\mathcal{O}(\epsilon^2 + \epsilon^4)$. 
\end{theorem*}

\begin{proof}
Note the $\gamma_t$ is just an abbreviation of the learning rate at global round $s$ and local round $k$. Using Lemma \ref{lemma:r2} (Eqn. \ref{eq:r2}) and telescoping:
\begin{align}
    \sum_{s=0}^{S-1}\sum_{k=0}^{m-1}(\Gamma_k - \frac{\gamma_k}{4})\cdot \mathbb{E}[\|\nabla \mathcal{J}(\boldsymbol{\theta}_k^{s+1}; \boldsymbol{\theta}_k^{s+1})\|^2] &\le (\widetilde{R}_0^1 - \widetilde{R}_m^S) + (L_0^2 L_1^2\epsilon^4 + 2L_1^2\epsilon^2L^2 + 2L_1^2\epsilon^2)\sum_{t=0}^{T-1}\gamma_t
\end{align}
Thus,
\begin{align*}
    &\frac{1}{T}\cdot \sum_{s=0}^{S-1}\sum_{k=0}^{m-1}\mathbb{E}[\|\nabla \mathcal{J}(\boldsymbol{\theta}_k^{s+1}; \boldsymbol{\theta}_k^{s+1})\|^2] \le \frac{\Delta_0}{T\Gamma} + \Delta_1 \\
\end{align*}
There is one remaining step: prove there exists values of $\gamma_k, \beta_k$ to let $\Gamma > 0$. We first begin with a lemma:
\begin{lemma}\label{lemma:c_k}
    For any $\eta > 0$ and $\beta_k$ is some constant $\beta$, there exists a sequence of learning rate $\{\gamma_k\}_{k=1}^{m}$ to let $c_k \le \eta L_0\epsilon$ hold for any $k$.
\end{lemma}

\begin{proof}
We can start from $c_m = 0$ and are given the relationship between $c_k$ and $c_{k+1}$:

\begin{align*}
c_k - \frac{1}{2}\gamma_k = c_{k+1}(1+ \gamma_k \beta_k + 2  L_0 \epsilon \gamma_k) + (2L^2+ L_0^2 \epsilon^2 ) (L\gamma_k^2 + 2c_{k+1}\gamma_k^2) + \frac{\beta_k}{2}L_0\epsilon\gamma_k
\end{align*}

For convenience, define $A_k = 1+\gamma_k\beta+2L_0\epsilon\,\gamma_k + (4L^2+2L_0^2\epsilon^2)\gamma_k^2$, $B_k = (2L^2+L_0^2\epsilon^2)L\gamma_k^2+\frac{\beta}{2}L_0\epsilon\,\gamma_k$. Then, the recurrence can be rewritten as $c_k = \frac{1}{2}\gamma_k + A_k\,c_{k+1} + B_k$. Next, we use induction to prove that for any $\eta>0$, there exists a sequence of $\{\gamma_k\}_{k=1}^{m}$ (depending only on $L$, $L_0$, $\epsilon$, and $\beta$) such that for all indices $k$ we have $c_k \le \eta L_0\epsilon$.

Since we start from $c_m = 0$, we have $c_m = 0 \le \eta_0\,L_0\epsilon$ holding for any positive $\eta_0 < \eta$. Next, assume for some $k+1$, $c_{k+1} \le \eta_{k+1}\, L_0\epsilon$ where $\eta_0 <= \eta_k < \eta$. Then we can work out $c_k$ as:

\begin{align*}
c_k = \frac{1}{2}\gamma_k + A_k\,c_{k+1} + B_k \le \frac{1}{2}\gamma_k + A_k\eta_{k+1}\,L_0\epsilon + B_k.
\end{align*}
Since we can adjust $\gamma_k$ to be small, we can expand $A_k = 1 + \gamma_k\beta + 2L_0\epsilon\,\gamma_k + \mathcal{O}(\gamma_k^2) $ and $B_k = \frac{\beta}{2}L_0\epsilon\,\gamma_k + \mathcal{O}(\gamma_k^2)$. Thus,
\begin{align*}
c_k &\le \frac{1}{2}\gamma_k + \Bigl(1 + \gamma_k\beta + 2L_0\epsilon\,\gamma_k\Bigr)\eta_{k+1}\,L_0\epsilon
+ \frac{\beta}{2}L_0\epsilon\,\gamma_k + C_1\gamma_k^2 \\
&= \eta_{k+1}\,L_0\epsilon + \Biggl\{ \frac{1}{2}  +\eta_{k+1}\,L_0\epsilon(\beta+2L_0\epsilon) + \frac{\beta}{2}L_0\epsilon \Biggr\}\gamma_k + C_1 \gamma_k^2\\
&= \eta_{k+1}\,L_0\epsilon +  C_0 \gamma_k + C_1 \gamma_k^2.
\end{align*}
where $C_0, C_1$ are constants depending on $L, L_0, \epsilon, \beta$. Fix any sequence $\{\omega_k\}_{k=0}^{m-1}$ such that $\sum_{k=0}^{m-1}\omega_k < \eta-\eta_0$.
Then by choosing $\gamma_k = \min\Bigl\{1,\; \frac{\omega_kL_0\epsilon}{C_0+C_1}\Bigr\}$, we obtain $c_k \le (\eta_{k+1}+\omega_k)L_0\epsilon$. This completes the proof.

\end{proof}

Next, we prove that with the learning rates in Lemma \ref{lemma:c_k}, then $\Gamma > 0$. We first assume $\gamma_k \leq \frac{1}{8L+\frac{8}{3}L_0\epsilon}$.

\begin{align*}
    \Gamma_k &= \gamma_k - \frac{c_{k+1} \gamma_k + \frac{1}{2} L_0 \epsilon \gamma_k}{\beta} - 2L \gamma_k^2 - 4 c_{k+1} \gamma_k^2.
\end{align*}

Using the bound $c_{k+1} \leq \eta L_0 \epsilon$ and taking $\beta \geq 4L + \frac{4}{3}L_0\epsilon$, we obtain
\begin{align*}
    \Gamma_k &\geq \gamma_k - \frac{(\eta L_0 \epsilon + \frac{1}{2} L_0 \epsilon) \gamma_k}{4L+ \frac{4}{3}L_0\epsilon} - 2L \gamma_k^2 - 4 \eta L_0 \epsilon \gamma_k^2 \\
    &= \gamma_k - \frac{(\eta + \frac{1}{2}) L_0 \epsilon \gamma_k}{4L+ \frac{4}{3}L_0\epsilon} - 2L \gamma_k^2 - 4\eta L_0 \epsilon \gamma_k^2.
\end{align*}

Then we can bound the quadratic terms as follows:
\begin{itemize}
    \item \(2L\gamma_k^2 \leq 2L \left(\frac{1}{8L+\frac{8}{3}L_0\epsilon}\right)\gamma_k < 2L \cdot \frac{1}{8L}\gamma_k = \frac{\gamma_k}{4} \)
    \item \( 4\eta L_0 \epsilon \gamma_k^2 \leq \frac{2\eta L_0 \epsilon \gamma_k}{4L+\frac{4}{3}L_0\epsilon} \) (since \( \gamma_k \leq \frac{1}{8L+\frac{8}{3}L_0\epsilon} \)).
\end{itemize}

Thus,
\begin{align*}
    \Gamma_k &\geq \gamma_k - \frac{(\eta + \frac{1}{2}) L_0 \epsilon \gamma_k}{4L+\frac{4}{3}L_0\epsilon} - \frac{\gamma_k}{4} - \frac{2\eta L_0 \epsilon \gamma_k}{4L+\frac{4}{3}L_0\epsilon}.
\end{align*}

Combine the two fractions involving $L_0 \epsilon$:

\begin{align*}
    \Gamma_k &\geq \gamma_k \left[ 1 - \frac{1}{4} - \frac{(3\eta + \frac{1}{2}) L_0 \epsilon}{4L+\frac{4}{3}L_0\epsilon} \right].
\end{align*}

To ensure $\Gamma_k-\frac{\gamma_k}{4}>\frac{\gamma_k}{8}$, it suffices that:

\begin{align*}
    1 - \frac{1}{4} - \frac{(3\eta + \frac{1}{2}) L_0 \epsilon}{4L+\frac{4}{3}L_0\epsilon} > \frac{3}{8}.
\end{align*}

This means that if $\eta < \frac{L}{2L_0\epsilon}$, then $\Gamma_k > \frac{3\gamma_k}{8}$ and hence $\Gamma_k-\frac{\gamma_k}{4}>\frac{\gamma_k}{8}$. Finally, denote $\frac{L}{2L_0\epsilon}$ by $\eta_0$ and choose $\eta=\frac{\eta_0}{2}$ in Lemma \ref{lemma:c_k}. We set $\omega_k=\omega=\frac{\eta_0}{4m}$, so that $\sum_{k=0}^{m-1}\omega_k=\frac{\eta_0}{4}<\frac{\eta_0}{2}=\eta$. Then we can set $\gamma = \min \left\{1,\frac{1}{8L+\frac{8}{3}L_0\epsilon},\frac{\omega L_0\epsilon}{C_0+C_1}\right\}$ and set $\gamma_k = \gamma$ so that Lemma \ref{lemma:c_k} also holds. Aggregating all these together, we may take $\Gamma=\frac{\gamma}{8}$, since $\Gamma_k-\frac{\gamma}{4}>\frac{\gamma}{8}$ for every $k$, resulting in the $\mathcal{O}(\frac{1}{T})$ convergence given that $\gamma$ is a constant.

Finally, consider $\Delta_1$ with $\Gamma=\frac{\gamma}{8}$, it is obvious it reduces to $L_1^2\epsilon^2+L_0^2 L_1^2\epsilon^4$ multiplied by some constant. When $\epsilon$ is smaller than $1$, the error neighborhood is $\mathcal{O}(\epsilon^2 + \epsilon^4)$.

To elaborate on the selection of $\gamma$ in practice, Corollary~\ref{corollary:complexity} provides a concrete scaling regime in the finite-population setting: one may choose a fixed learning rate $\gamma=\Theta(1/n^\alpha)$ together with an epoch length $m=\Theta(n^{\alpha/2})$, where $\alpha\in(0,1)$. The auxiliary parameter $\beta=\Theta(n^{\alpha/2})$ is used in the proof to establish the Lyapunov bound. This result gives an order-level guideline for choosing $\gamma$, while the constant factor can be selected by standard tuning procedures, such as grid search on a held-out validation set under the induced distribution, or by decreasing the learning rate until the empirical training dynamics and stationarity measure become stable.

\color{black}

\end{proof}

\section{Proof of Corollary \ref{corollary:complexity}}\label{app:proofcorollary}

\begin{corollary*}
Consider the finite-population setting with \(n\) samples and \(\alpha\in(0,1)\). With fixed parameters
\[
\gamma_k=\gamma=\Theta(n^{-\alpha}),\qquad
\beta_k=\beta=\Theta(n^{\alpha/2}),\qquad
m=\Theta(n^{\alpha/2}),
\]
where the implied constants are chosen to satisfy the conditions in Appendix D, we have
\[
\Gamma=\Theta(n^{-\alpha}),
\]
and the IFO complexity required to reach an \(O(\delta)\)-SPS solution beyond the error neighborhood is
\[
O\!\left(
\frac{(n^\alpha+n^{1+\alpha/2})\Delta_0}{\delta}
\right).
\]
\end{corollary*}

\begin{proof}
    From Lemma \ref{lemma:c_k} and the proof in App. \ref{app:thmconverge}, when $m = \Theta(n^{\frac{\alpha}{2}})$, $\beta=\Theta(n^{\frac{\alpha}{2}})$, and $\omega=\frac{\eta_0}{4m}$, where the implied constant in $\beta$ is chosen so that $\beta\geq 4L+\frac{4}{3}L_0\epsilon$, we have $C_0+C_1=\Theta(\beta)=\Theta(n^{\frac{\alpha}{2}})$. Thus, we choose $\gamma_k=\gamma=\min\{1,\frac{1}{8L+\frac{8}{3}L_0\epsilon},\frac{\omega L_0 \epsilon}{C_0+C_1}\}=\Theta(\frac{1}{n^\alpha})$, which satisfies Lemma \ref{lemma:c_k}. By the proof in App. \ref{app:thmconverge}, we may take $\Gamma=\frac{\gamma}{8}=\Theta(\frac{1}{n^\alpha})$. With this fact, we let $T=\lceil \frac{\Delta_0}{\delta\Gamma}\rceil$ and, if necessary, round $T$ up to the nearest multiple of $m$. In the usual accuracy regime $0<\delta\leq\Delta_0$, this gives $T=\mathcal{O}(\frac{n^\alpha\Delta_0}{\delta})$. Since for each $m$ rounds ($1$ epoch), we would have $2m+n$ evaluations of gradients, the average sample complexity at each round is $2+\frac{n}{m}$. Then, the final IFO complexity is:
    \begin{align*}
        T(2+\frac{n}{m}) = \mathcal{O}\left(\frac{\Delta_0n^{\alpha}(2+n^{1-\frac{\alpha}{2}})}{\delta}\right) = \mathcal{O}(\frac{(n^{\alpha} + n^{1+\frac{\alpha}{2}})\Delta_0}{\delta}).
    \end{align*}
\end{proof}
\color{black}

\section{Additional experimental results}
\label{app:exp}

\textbf{Additional experimental settings.} We conduct all the experiments on a server which has two Intel Xeon 6326 CPUs and an Nvidia A6000 GPU. We implement our code using Pytorch 2.4.1 and seed $2024$. In the PP settings, each training requires extensive steps to model distributional shifts, which significantly increases the computational cost, especially on large-scale datasets like CIFAR-10. We are only able to run one experiment on CIFAR-10, but we include experiments using $3$ random seeds for the MNIST dataset as follows.

\textbf{Additional experiments on CIFAR-10 and Credit dataset.} We additionally report error bars for Credit and CIFAR-10 (with 10\% of the data randomly selected) by rerunning each experiment with three different random seeds. The results are shown in Fig. \ref{fig:cifar_error} and Fig. \ref{fig:credit_error}. In Fig. \ref{fig:credit_error}, the error bars are relatively large compared to the gap between \alg and SGD-GD, which is expected since the Credit dataset is relatively small and SGD-GD can already converge fairly quickly. Nevertheless, \alg consistently maintains an advantage over SGD-GD across seeds.

\begin{figure}[h]
    \centering
    \begin{subfigure}[b]{0.32\textwidth}
        \includegraphics[width=\textwidth]{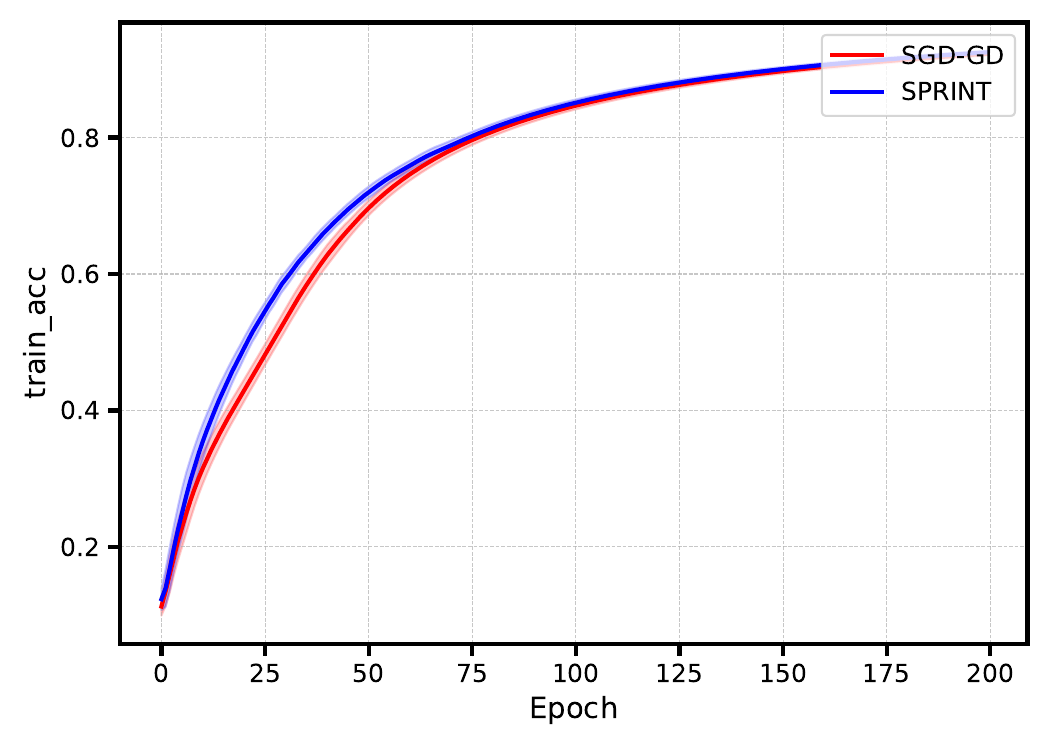}
        \label{fig:cifar_error_sub1}
    \end{subfigure}
    \hfill
    \begin{subfigure}[b]{0.32\textwidth}
        \includegraphics[width=\textwidth]{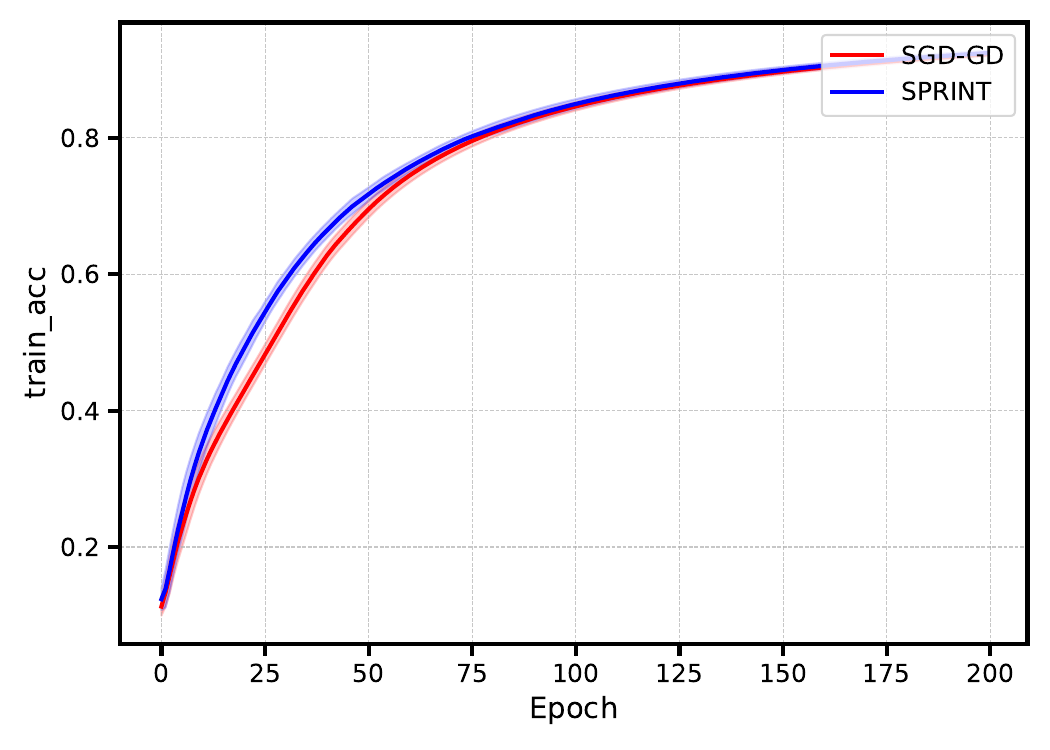}
        \label{fig:cifar_error_sub2}
    \end{subfigure}
    \hfill
    \begin{subfigure}[b]{0.32\textwidth}
        \includegraphics[width=\textwidth]{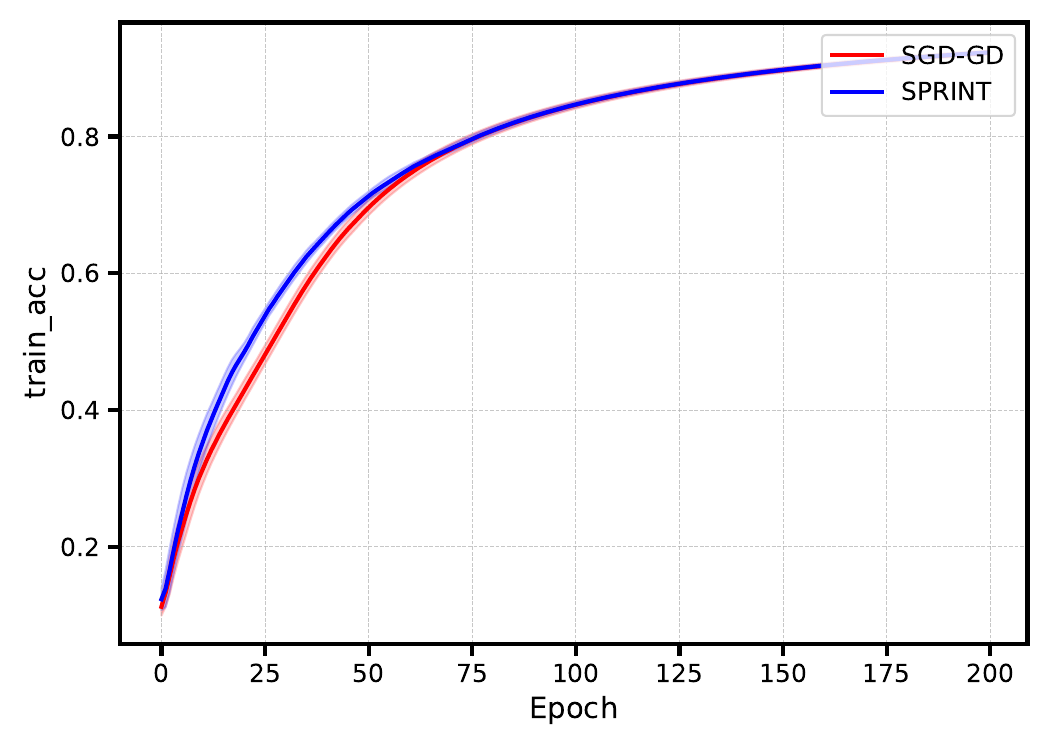}
        \label{fig:cifar_error_sub3}
    \end{subfigure}

    \begin{subfigure}[b]{0.32\textwidth}
        \includegraphics[width=\textwidth]{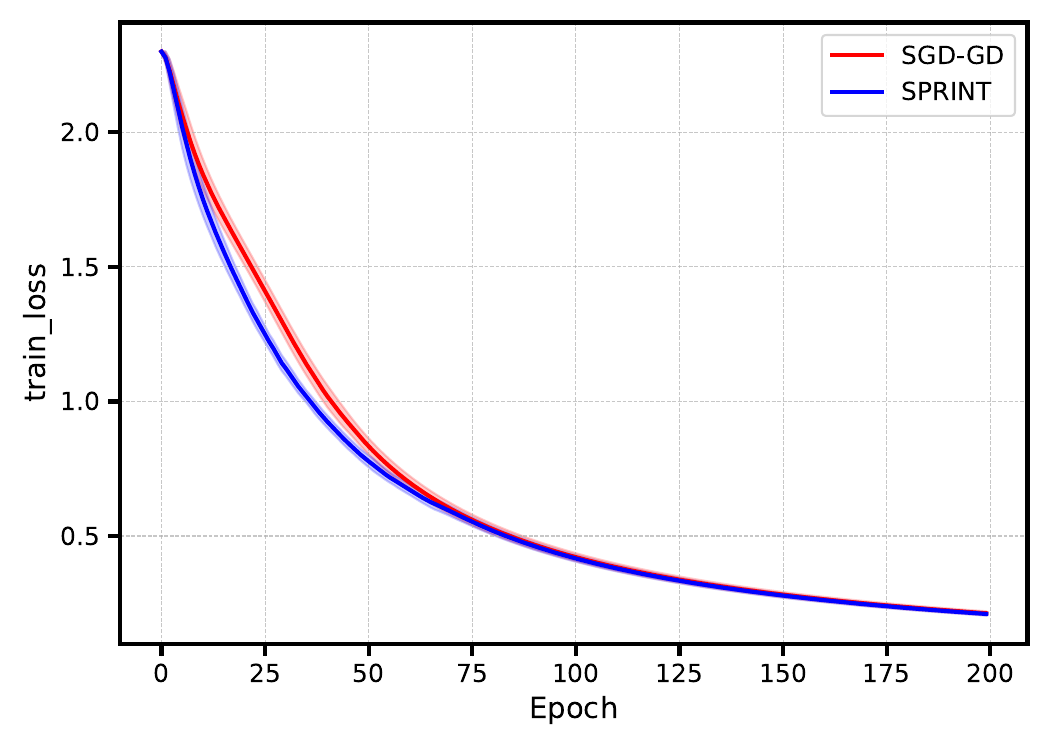}
        \caption{$\alpha = 20$}
        \label{fig:cifar_error_sub4}
    \end{subfigure}
    \hfill
    \begin{subfigure}[b]{0.32\textwidth}
        \includegraphics[width=\textwidth]{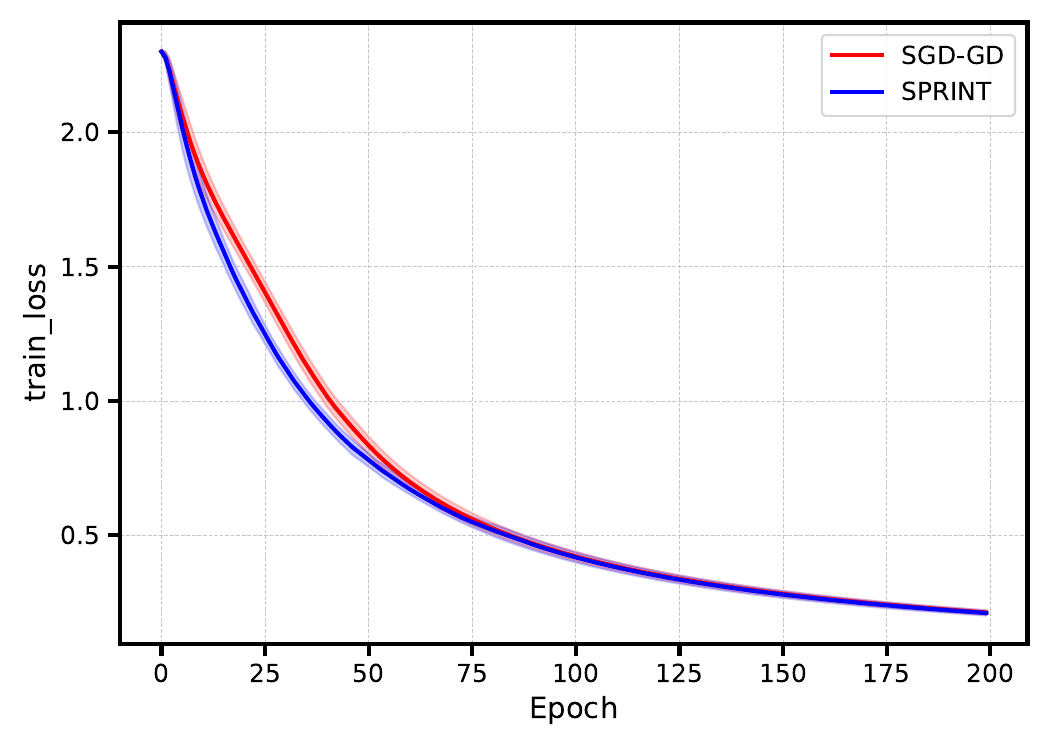}
        \caption{$\alpha = 50$}
        \label{fig:cifar_error_sub5}
    \end{subfigure}
    \hfill
    \begin{subfigure}[b]{0.32\textwidth}
        \includegraphics[width=\textwidth]{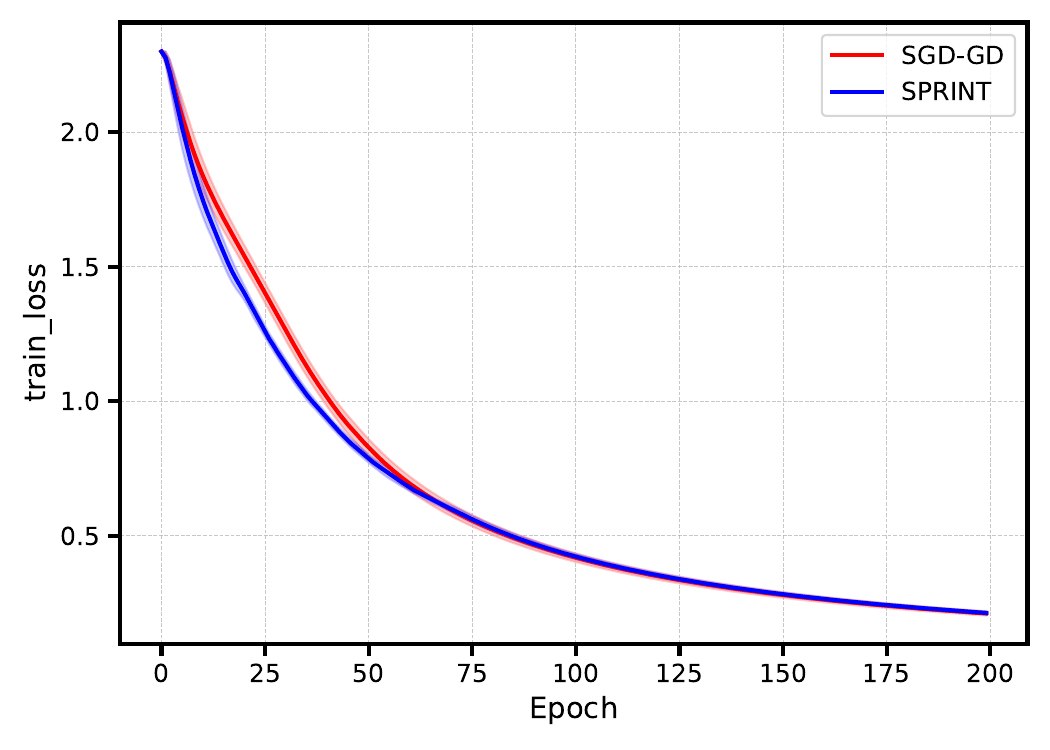}
        \caption{$\alpha = 80$}
        \label{fig:cifar_error_sub6}
    \end{subfigure}
    \caption{\textbf{CIFAR-10} dataset: training accuracy (first row) and training loss (second row) of SGD-GD and \alg under different $\alpha$. Larger $\alpha$ means more intense retention dynamics and the shadow demonstrates standard errors.}
    \label{fig:cifar_error}
\end{figure}

\begin{figure}[h]
    \centering
    \begin{subfigure}[b]{0.32\textwidth}
        \includegraphics[width=\textwidth]{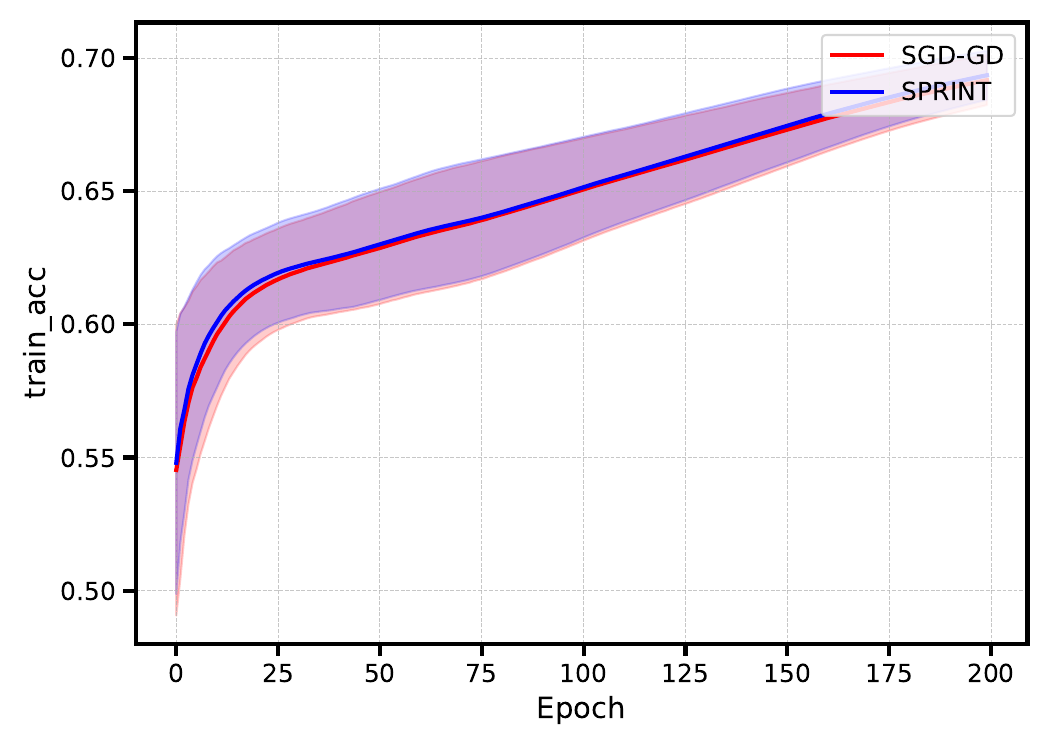}
        \label{fig:credit_error_sub1}
    \end{subfigure}
    \hfill
    \begin{subfigure}[b]{0.32\textwidth}
        \includegraphics[width=\textwidth]{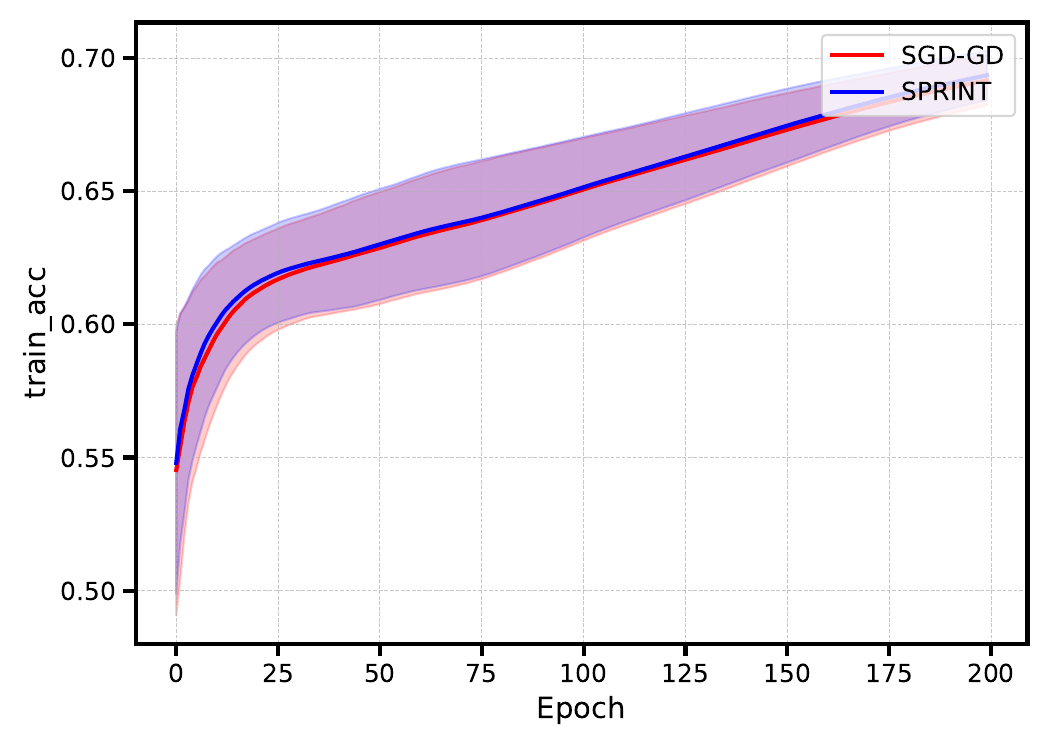}
        \label{fig:credit_error_sub2}
    \end{subfigure}
    \hfill
    \begin{subfigure}[b]{0.32\textwidth}
        \includegraphics[width=\textwidth]{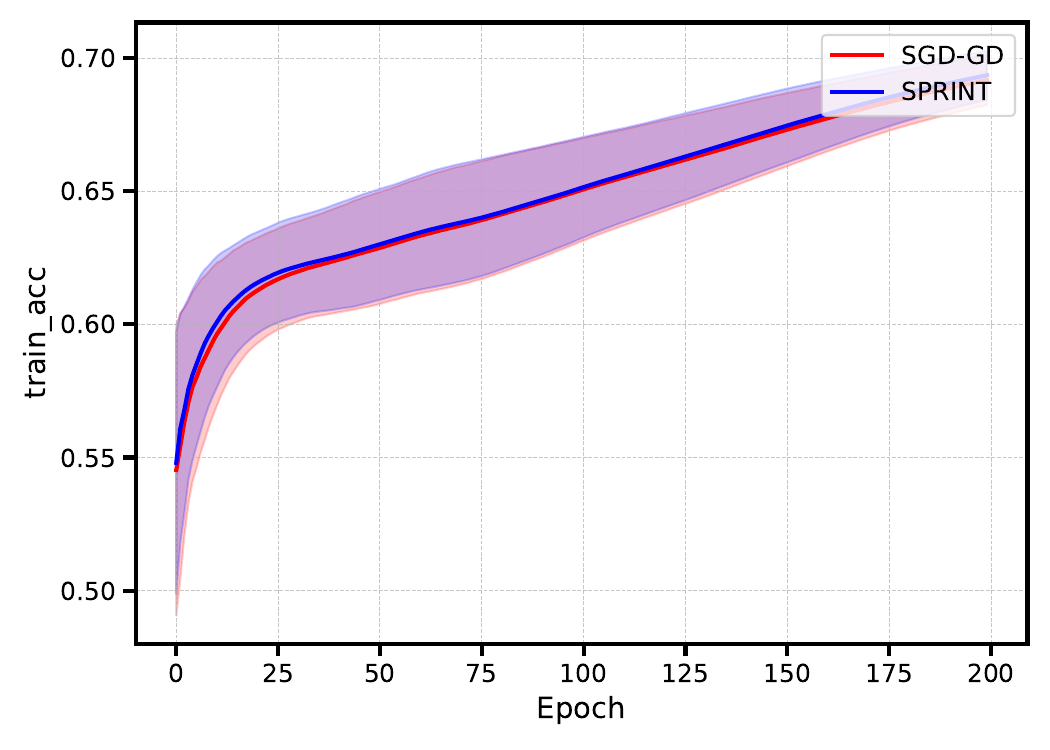}
        \label{fig:credit_error_sub3}
    \end{subfigure}

    \begin{subfigure}[b]{0.32\textwidth}
        \includegraphics[width=\textwidth]{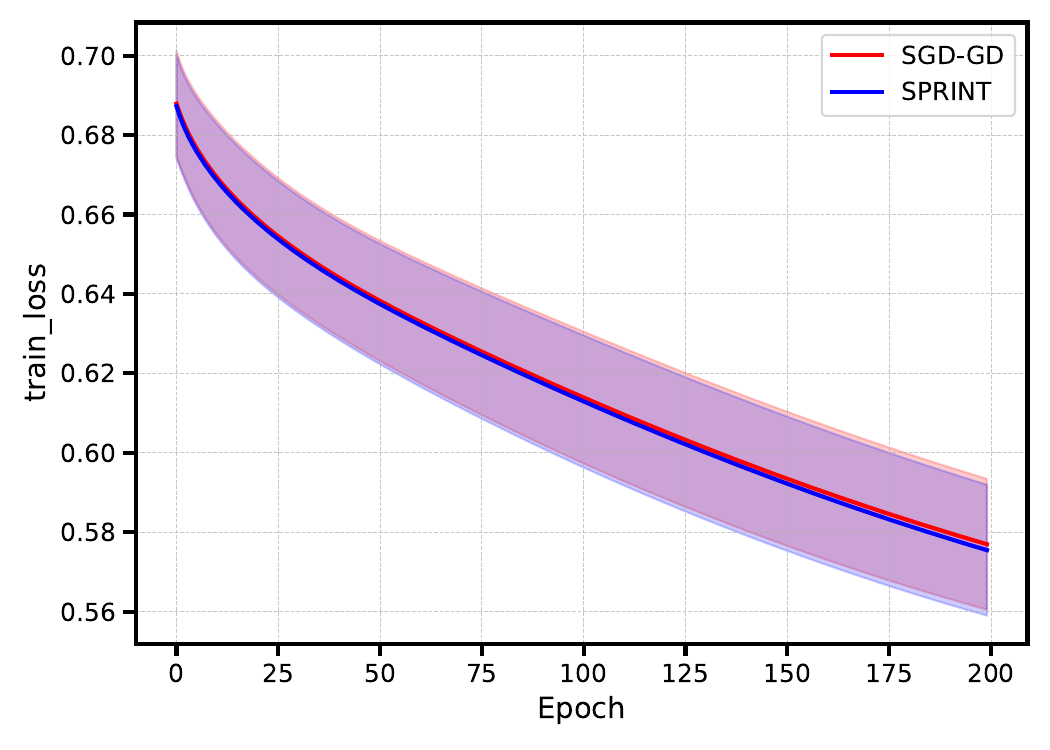}
        \label{fig:credit_error_sub4}
        \caption{$\alpha = 0.01$}
    \end{subfigure}
    \hfill
    \begin{subfigure}[b]{0.32\textwidth}
        \includegraphics[width=\textwidth]{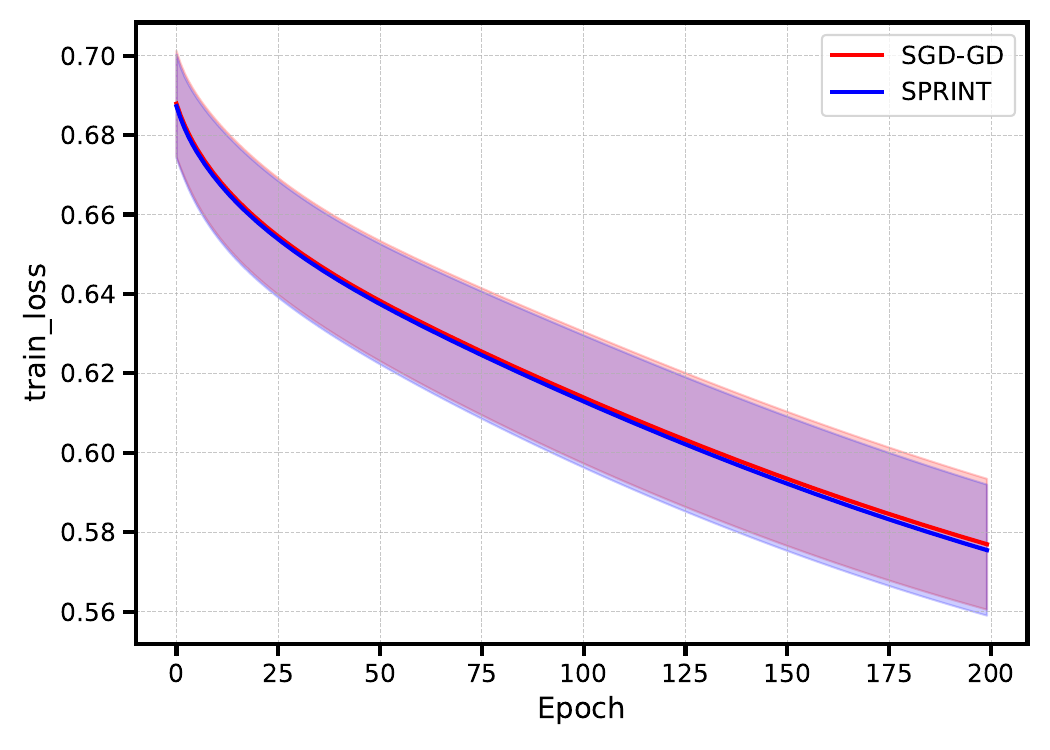}
        \label{fig:credit_error_sub5}
        \caption{$\alpha = 0.2$}
    \end{subfigure}
    \hfill
    \begin{subfigure}[b]{0.32\textwidth}
        \includegraphics[width=\textwidth]{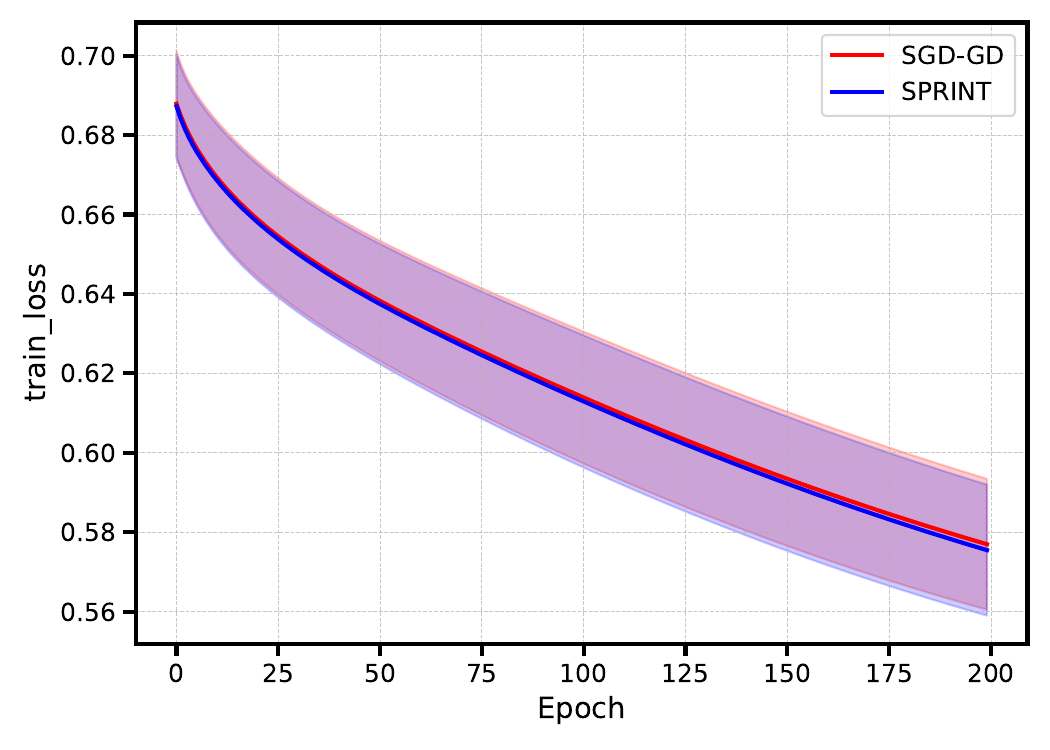}
        \label{fig:credit_error_sub6}
        \caption{$\alpha = 0.4$}
    \end{subfigure}

    % \begin{subfigure}[b]{0.32\textwidth}
    %     \includegraphics[width=\textwidth]{figs/eval_metric/credit/temp0.01/eval_acc.pdf}
    %     \label{fig:credit_eval_acc_sub1}
    %     \caption{$\alpha = 0.01$}
    % \end{subfigure}
    % \hfill
    % \begin{subfigure}[b]{0.32\textwidth}
    %     \includegraphics[width=\textwidth]{figs/eval_metric/credit/temp0.2/eval_acc.pdf}
    %     \label{fig:credit_eval_acc_sub2}
    %     \caption{$\alpha = 0.2$}
    % \end{subfigure}
    % \hfill
    % \begin{subfigure}[b]{0.32\textwidth}
    %     \includegraphics[width=\textwidth]{figs/eval_metric/credit/temp0.4/eval_acc.pdf}
    %     \label{fig:credit_eval_acc_sub3}
    %     \caption{$\alpha = 0.4$}
    % \end{subfigure}

    \caption{\textbf{Credit} dataset: training accuracy (first row) and training loss (second row) of SGD-GD and \alg under different $\alpha$. Larger $\alpha$ means more intense strategic behaviors and the shadow demonstrates standard errors.}
    \label{fig:credit_error}
\end{figure}

\begin{figure}[h]
    \centering
    \begin{subfigure}[b]{0.32\textwidth}
        \includegraphics[width=\textwidth]{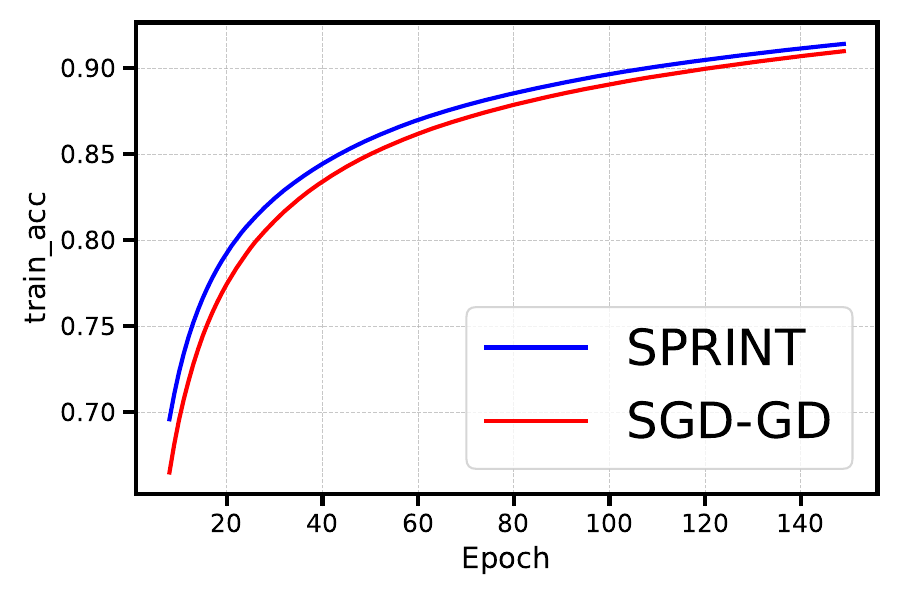}
        \label{fig:mnist_sub1}
    \end{subfigure}
    \hfill
    \begin{subfigure}[b]{0.32\textwidth}
        \includegraphics[width=\textwidth]{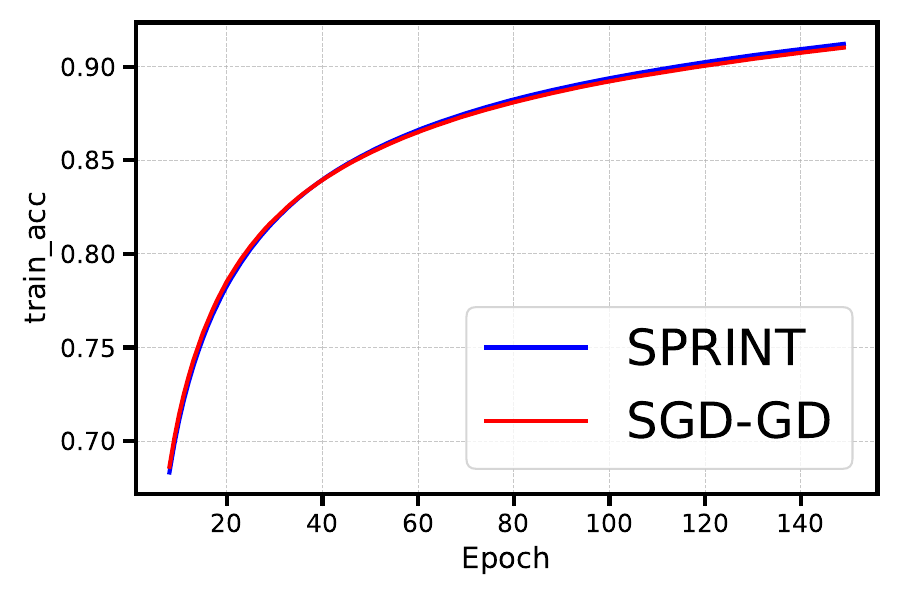}
        \label{fig:mnist_sub2}
    \end{subfigure}
    \hfill
    \begin{subfigure}[b]{0.32\textwidth}
        \includegraphics[width=\textwidth]{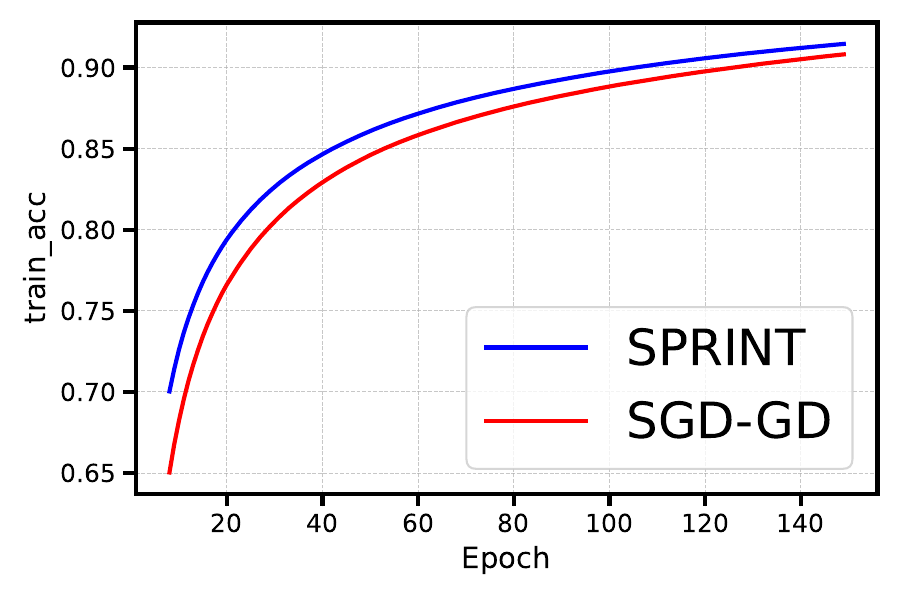}
        \label{fig:mnist_sub3}
    \end{subfigure}

    \begin{subfigure}[b]{0.32\textwidth}
        \includegraphics[width=\textwidth]{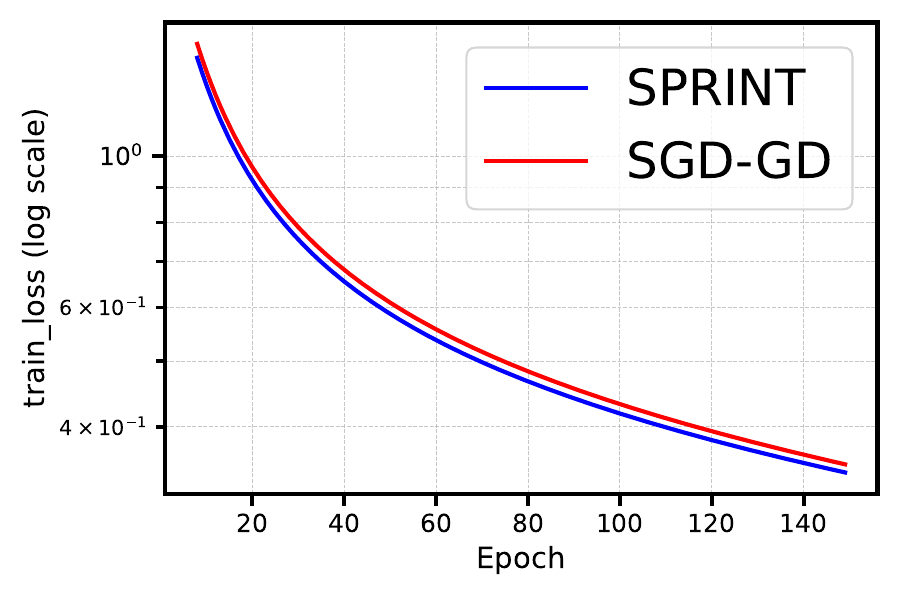}
        \caption{$\alpha = 20$}
        \label{fig:mnist_sub4}
    \end{subfigure}
    \hfill
    \begin{subfigure}[b]{0.32\textwidth}
        \includegraphics[width=\textwidth]{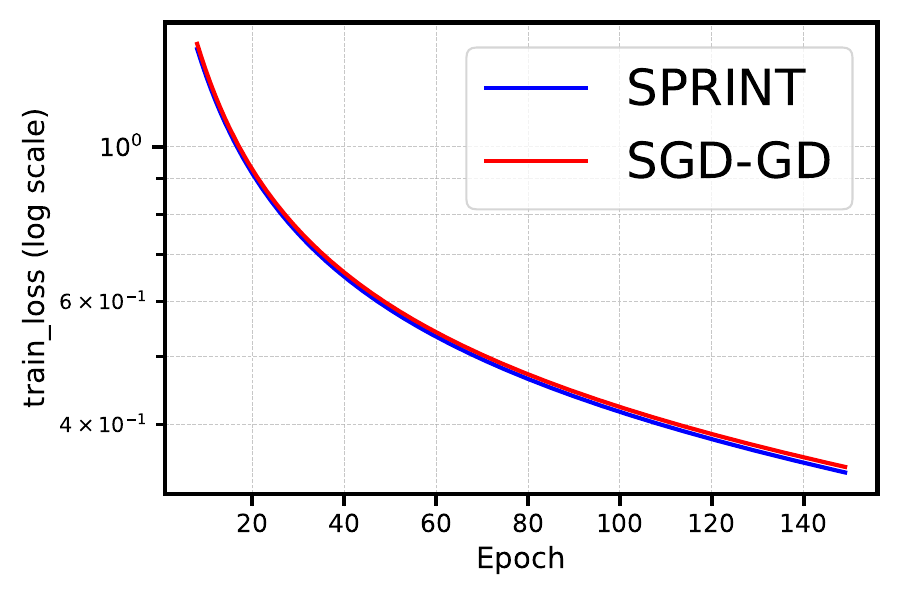}
        \caption{$\alpha = 50$}
        \label{fig:mnist_sub5}
    \end{subfigure}
    \hfill
    \begin{subfigure}[b]{0.32\textwidth}
        \includegraphics[width=\textwidth]{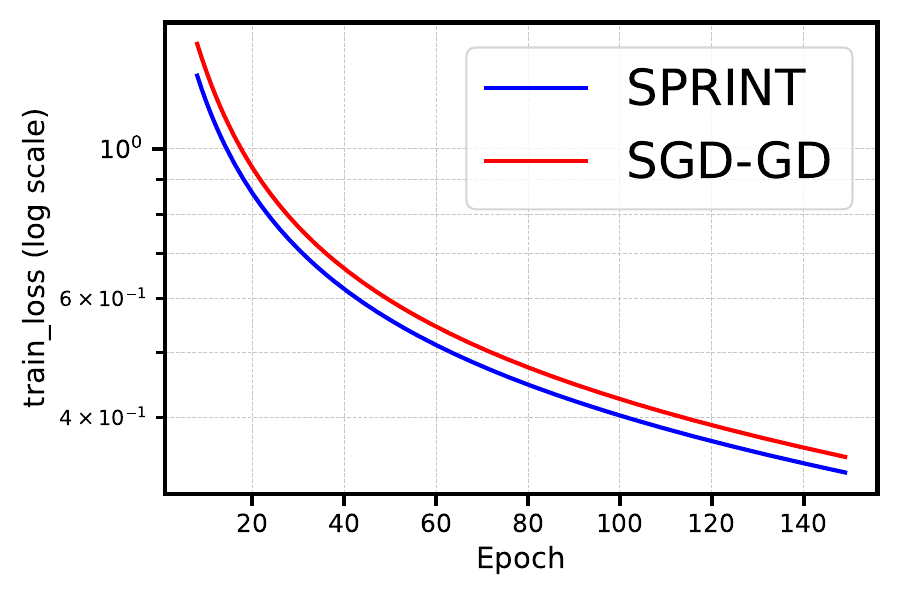}
        \caption{$\alpha = 80$}
        \label{fig:mnist_sub6}
    \end{subfigure}

    \caption{\textbf{MNIST} dataset: training accuracy (first row) and training loss (second row) of SGD-GD and \alg under different $\alpha$. Larger $\alpha$ means more intense retention dynamics and the shadow demonstrates standard errors.}
    \label{fig:mnist}
\end{figure}

\noindent \textbf{CIFAR-10 dataset.} Furthermore, we use the \textbf{CIFAR-10 dataset \citep{krizhevsky2009learning}} to predict 10 possible class labels $Y \in \{0, 1, ...,9\}$. 
Each label corresponds to a common and well-separated category such as \textit{dog} and \textit{truck}. 
The dataset contains $50,000$ images in total and each class includes $5,000$ samples. 
We design a CNN model with two convolutional layers to train the algorithms with a $0.05$ learning rate and cross-entropy loss. 
We simulate performative effects motivated by retention dynamics \citep{hashimoto2018fairness, jin2024addressingpolarizationunfairnessperformative, Zhang_2019_Retention} where the class distribution will change based on the performance of the current model. 
The fraction in class $c$ at iteration $k$ of $p^{k+1}_c = \frac{e^{-\alpha \ell^k_c}}{\sum_{c'}e^{-\alpha \ell^k_{c'}}}$, where $\ell^k_c$ is the loss expectation of class $c$ in iteration $k$. 
This means that the image class with larger loss now will be less likely to appear in the next iteration \citep{Zhang_2019_Retention}. 
We set $\alpha$ at $20, 50, 80$ and train \textcolor{blue}{1000} epochs in total. 
The experimental results of squared gradient norm are shown in Fig.~\ref{fig:grad}. 
Consistent with the results observed on the Credit dataset, \alg also exhibits a faster convergence on the CIFAR-10 dataset. We further note that the CIFAR-10 experiments employ CNN architectures for which Assumption~\ref{assumption:lipschitzness} may not hold globally. The stable convergence observed empirically suggests that \alg's performance is robust to mild violations of these assumptions in practice.

\begin{figure}[h]
    \centering
    \begin{subfigure}[b]{0.32\textwidth}
        \includegraphics[width=\textwidth]{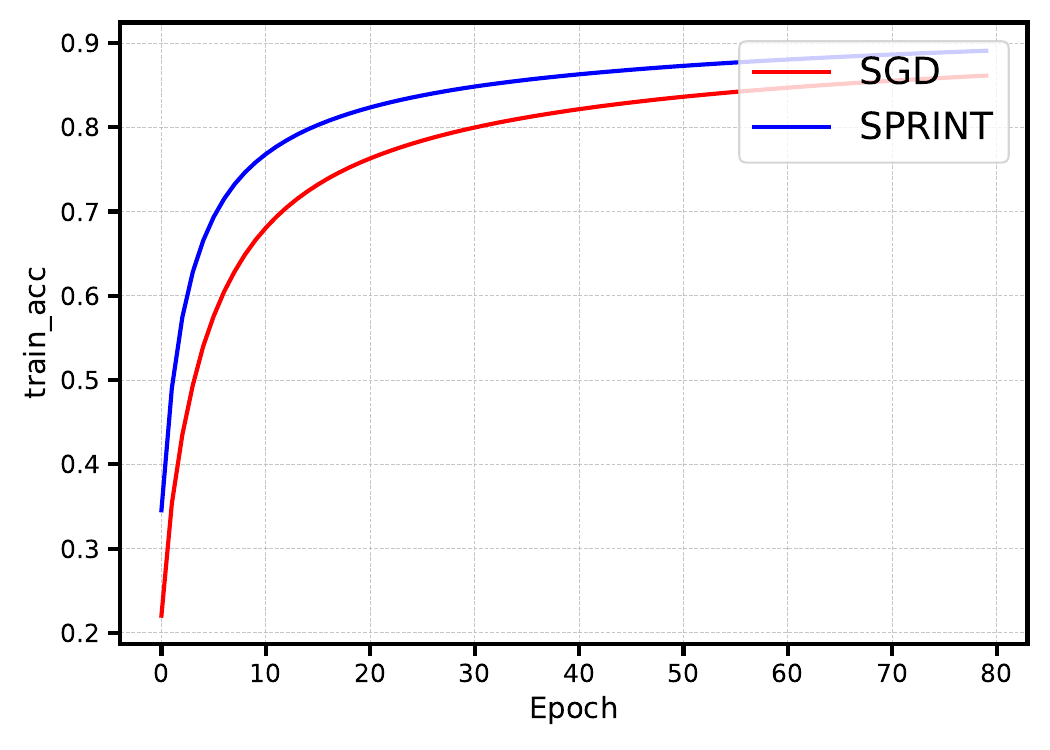}
        \label{fig:mnist_lr_sub1}
    \end{subfigure}
    \hfill
    \begin{subfigure}[b]{0.32\textwidth}
        \includegraphics[width=\textwidth]{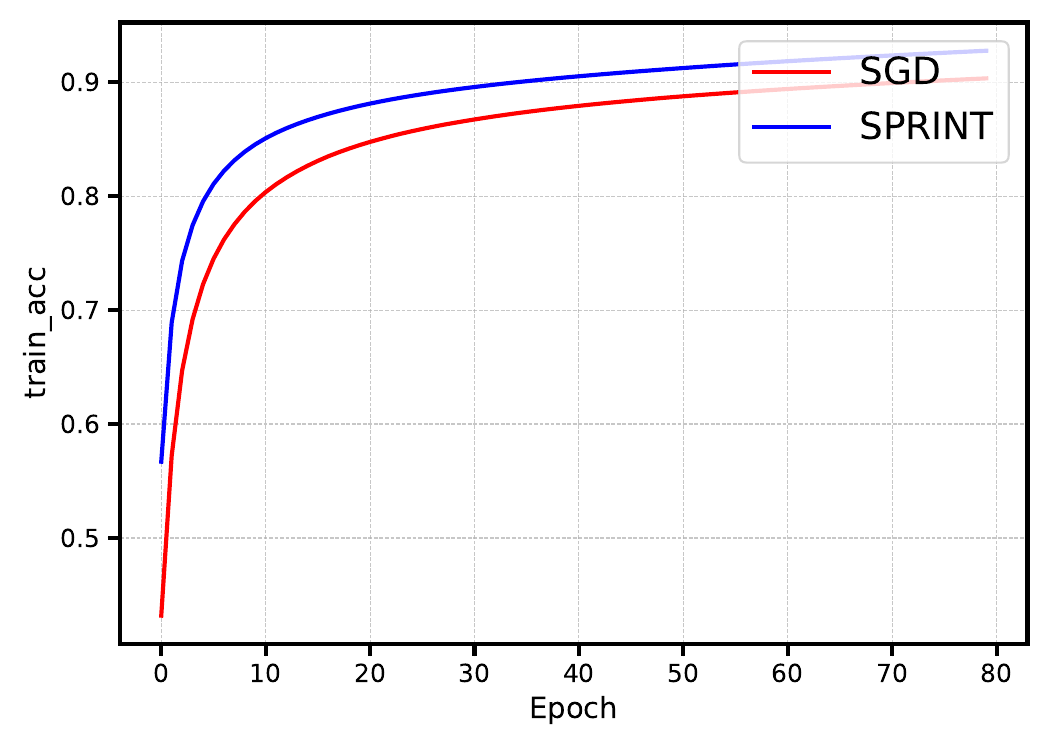}
        \label{fig:mnist_lr_sub2}
    \end{subfigure}
    \hfill
    \begin{subfigure}[b]{0.32\textwidth}
        \includegraphics[width=\textwidth]{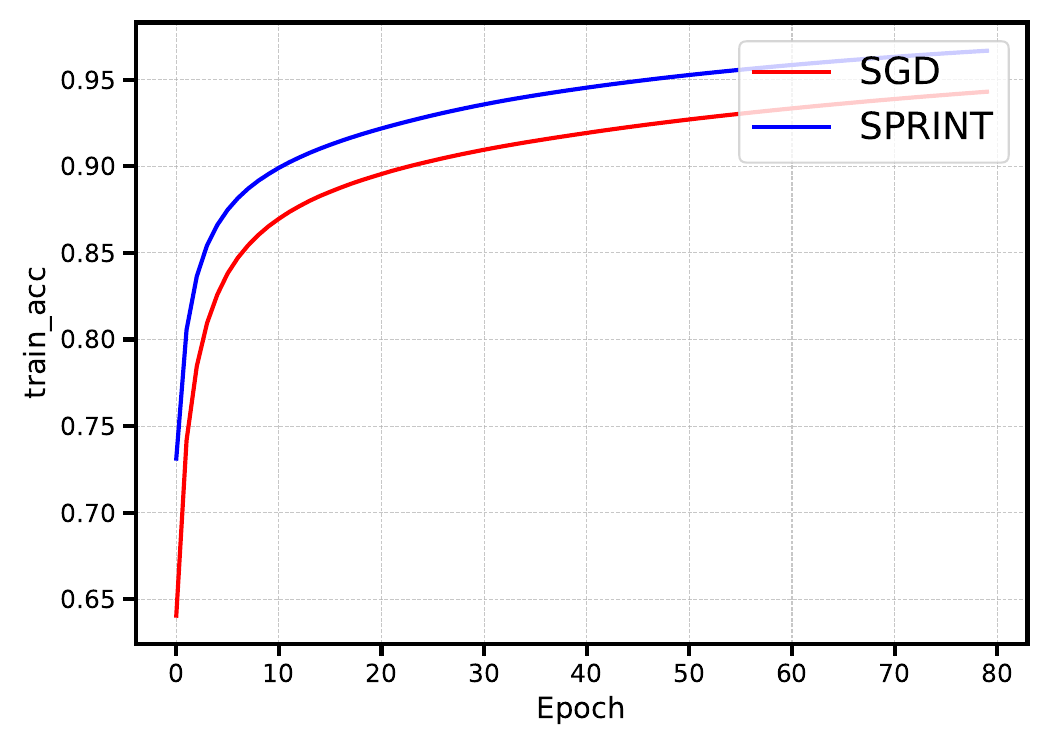}
        \label{fig:mnist_lr_sub3}
    \end{subfigure}

    \begin{subfigure}[b]{0.32\textwidth}
        \includegraphics[width=\textwidth]{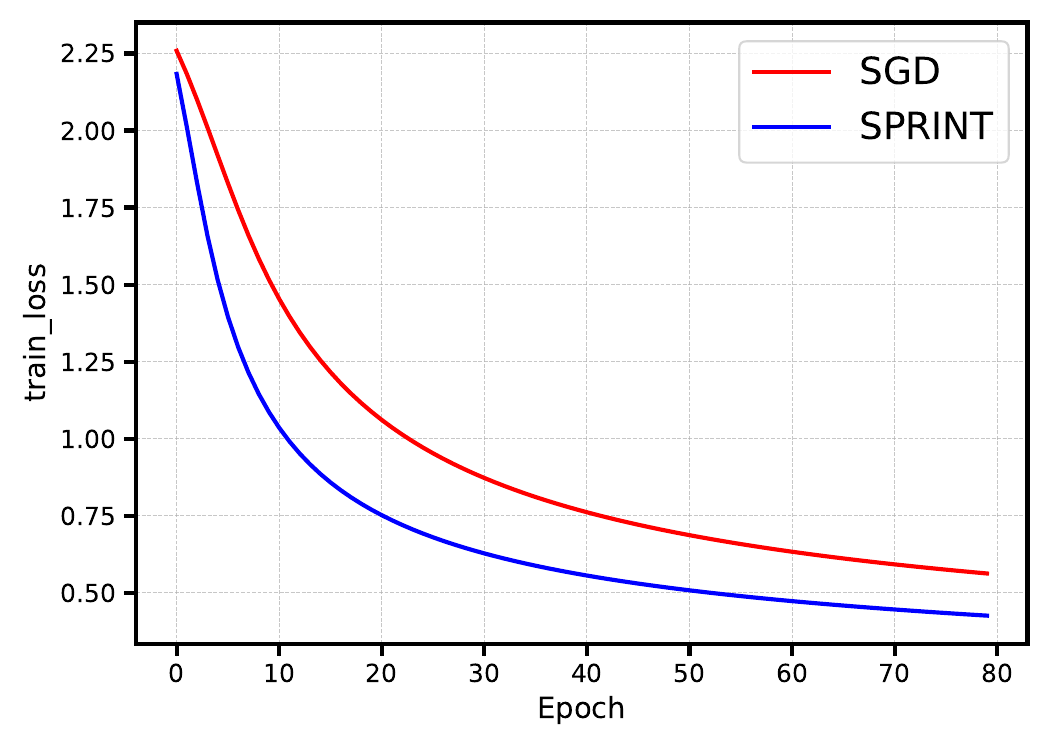}
        \caption{$\gamma = 0.001$}
        \label{fig:mnist_lr_sub4}
    \end{subfigure}
    \hfill
    \begin{subfigure}[b]{0.32\textwidth}
        \includegraphics[width=\textwidth]{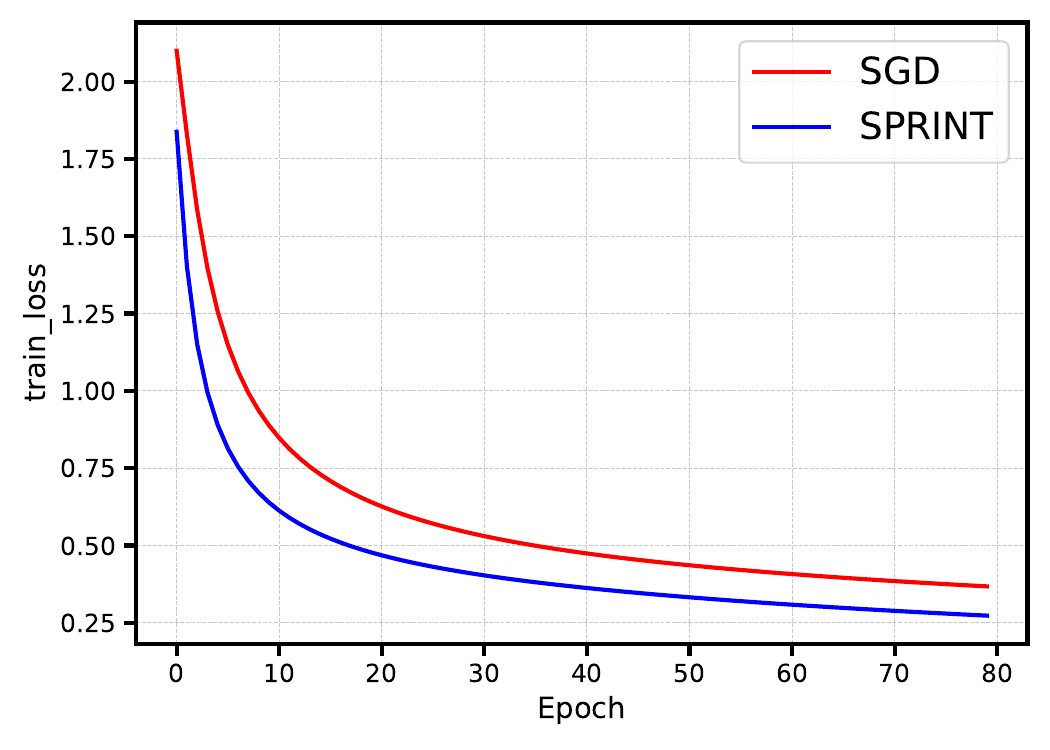}
        \caption{$\gamma = 0.003$}
        \label{fig:mnist_lr_sub5}
    \end{subfigure}
    \hfill
    \begin{subfigure}[b]{0.32\textwidth}
        \includegraphics[width=\textwidth]{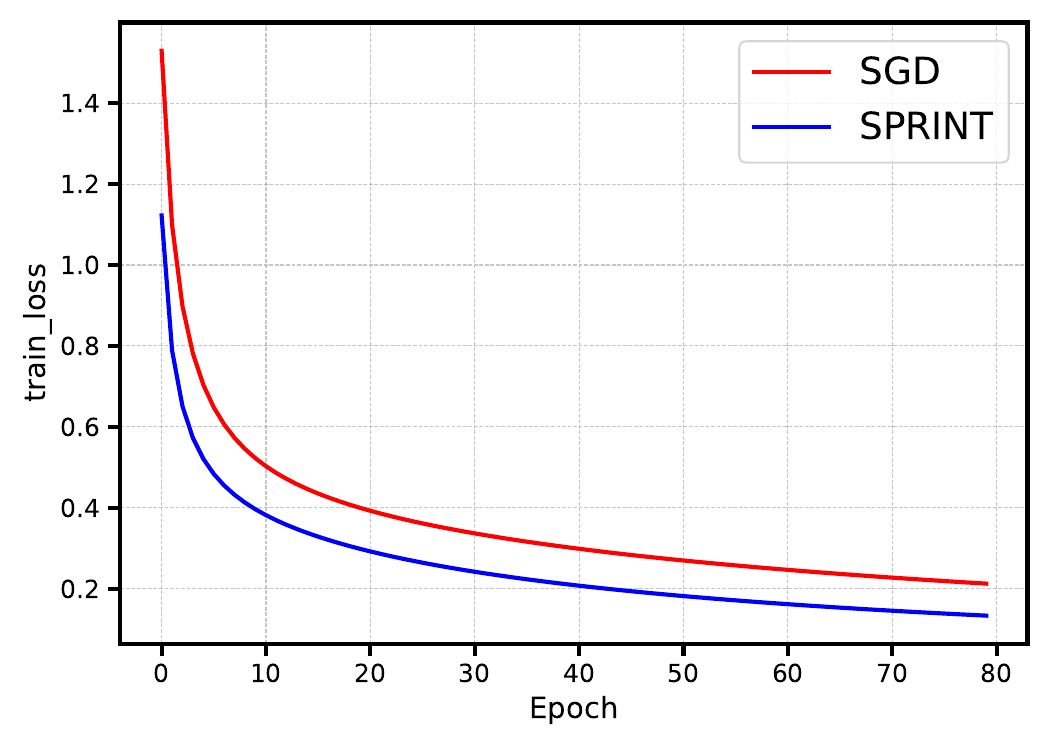}
        \caption{$\gamma = 0.01$}
        \label{fig:mnist_lr_sub6}
    \end{subfigure}
    \caption{\textbf{MNIST} dataset: training accuracy (first row) and training loss (second row) of SGD-GD and \alg under different learning rate $\gamma$.}
    \label{fig:mnist_lr}
\end{figure}

\begin{figure}[h]
    \centering
    \begin{subfigure}[b]{0.32\textwidth}
        \includegraphics[width=\textwidth]{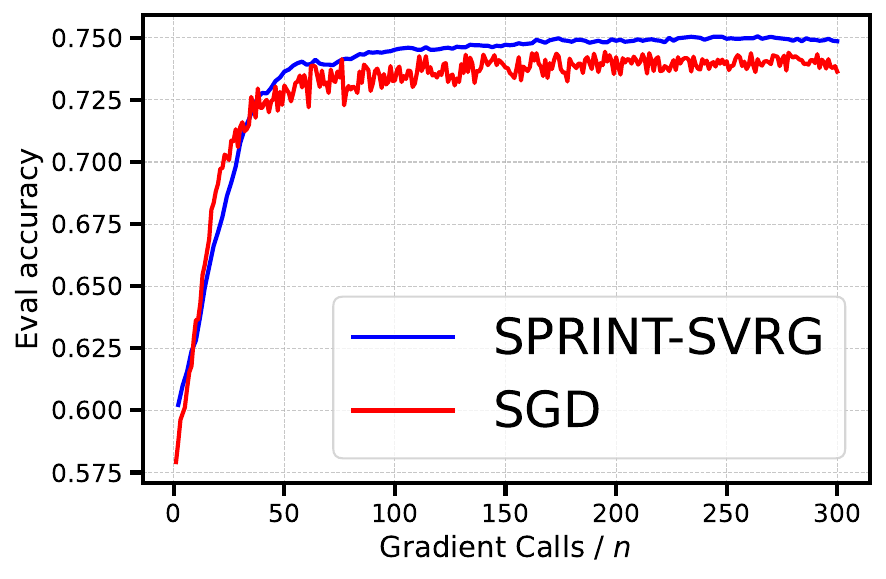}
        \caption{$\alpha = 0.01$}
        \label{fig:credit_eval_1}
    \end{subfigure}
    \hfill
    \begin{subfigure}[b]{0.32\textwidth}
        \includegraphics[width=\textwidth]{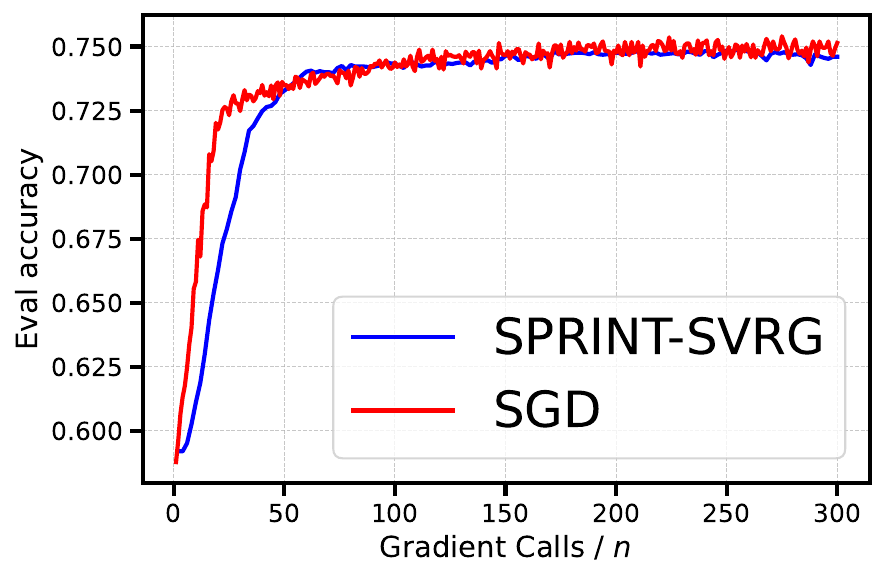}
        \caption{$\alpha = 0.2$}
        \label{fig:credit_eval_2}
    \end{subfigure}
    \hfill
    \begin{subfigure}[b]{0.32\textwidth}
        \includegraphics[width=\textwidth]{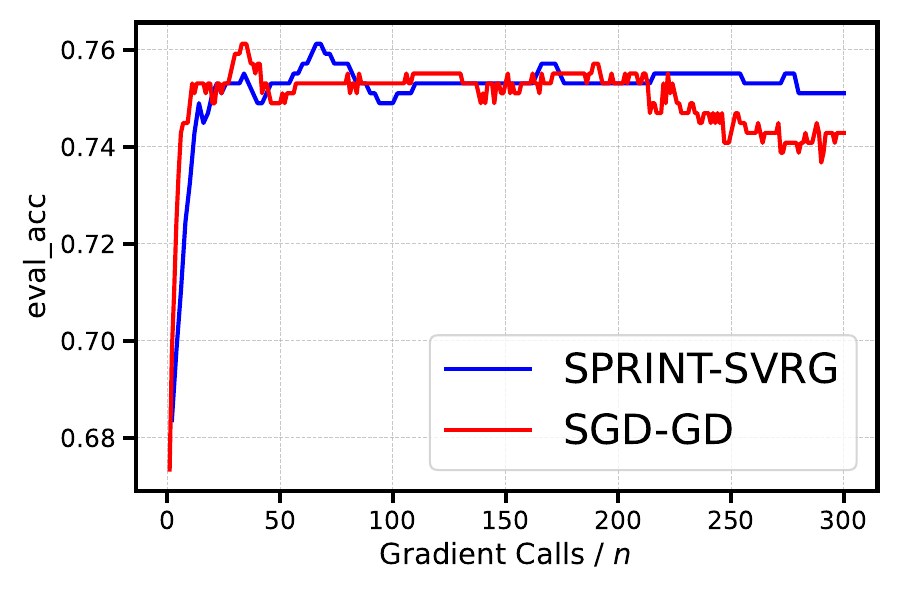}
        \caption{$\alpha = 0.4$}
        \label{fig:credit_eval_3}
    \end{subfigure}

    \begin{subfigure}[b]{0.32\textwidth}
        \includegraphics[width=\textwidth]{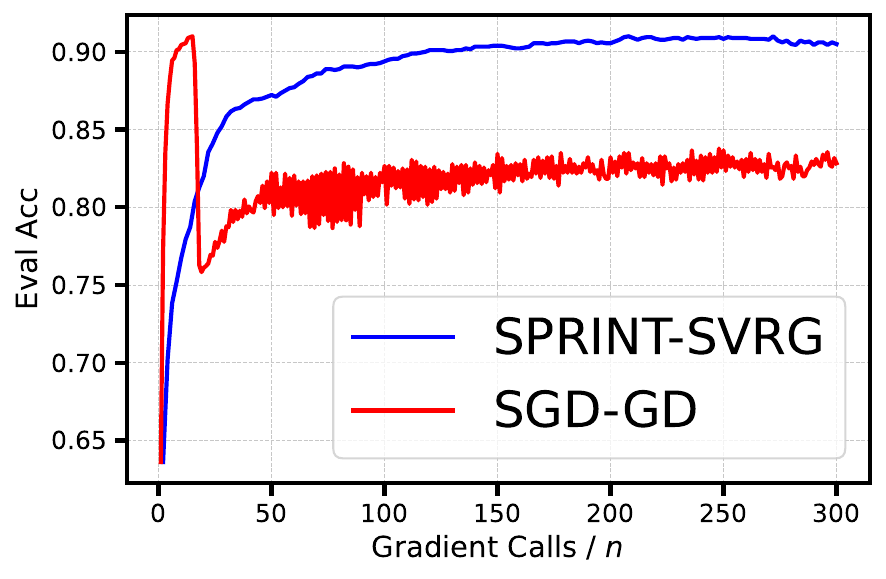}
        \caption{$\alpha = 40$}
        \label{fig:spambase_eval_1}
    \end{subfigure}
    \hfill
    \begin{subfigure}[b]{0.32\textwidth}
        \includegraphics[width=\textwidth]{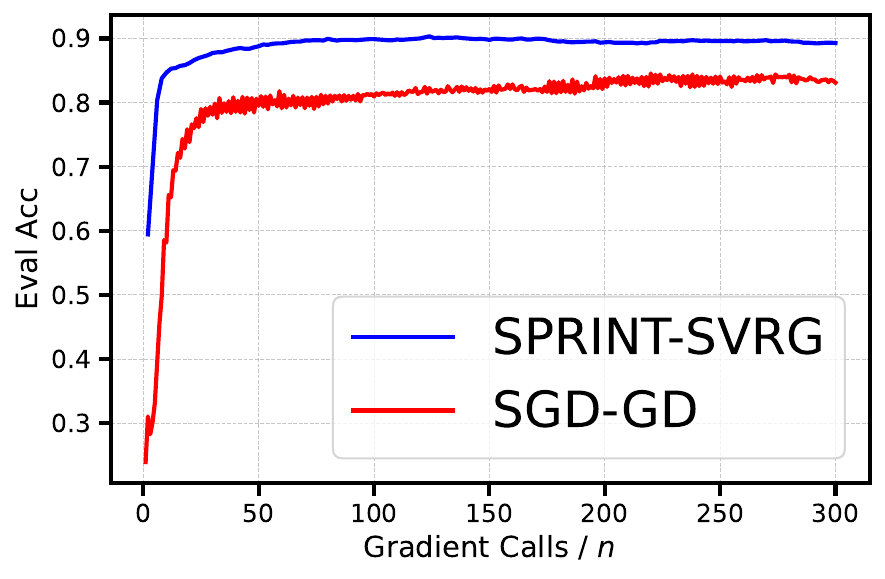}
        \caption{$\alpha = 60$}
        \label{fig:spambase_eval_2}
    \end{subfigure}
    \hfill
    \begin{subfigure}[b]{0.32\textwidth}
        \includegraphics[width=\textwidth]{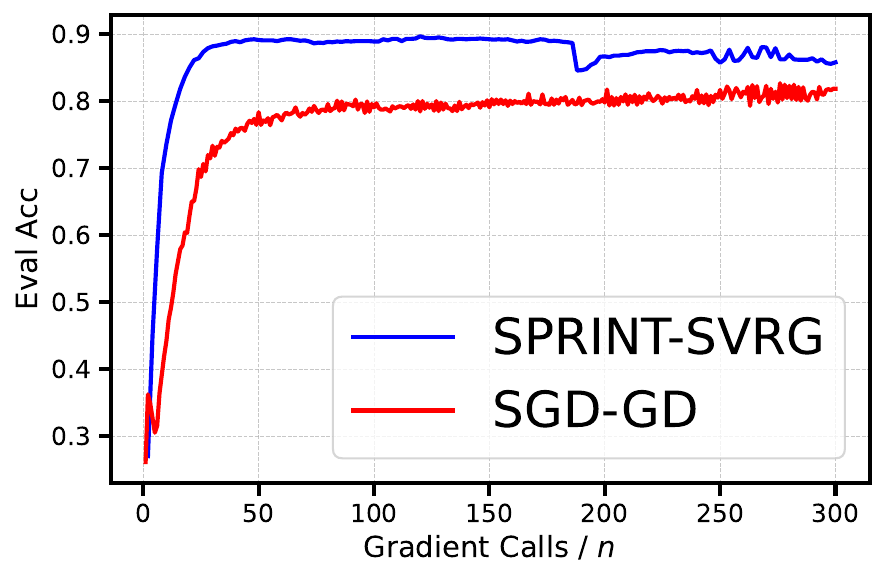}
        \caption{$\alpha = 80$}
        \label{fig:spambase_eval_3}
    \end{subfigure}
    \caption{Out-of-sample evaluation accuracy of SGD-GD and \alg on the \textbf{Credit} and \textbf{SpamBase} datasets under different performativity strengths $\alpha$. From top to bottom are \textbf{Credit} dataset and \textbf{SpamBase} dataset. \alg consistently achieves competitive or superior evaluation accuracy compared to SGD-GD, demonstrating that faster convergence to an SPS solution does not come at the cost of generalization.}
    \label{fig:eval_acc}
\end{figure}

\textbf{The average of squared gradient norm.} An SPS solution is characterized by a small gradient norm $\nabla \|\mathcal{J}(\tilde{\theta}^s, \tilde{\theta}^s)\|^2 \approx 0 $. To empirically verify the \alg converges to an SPS solution, we track this quantity directly. We visualize the cumulative average of the squared gradient norm at the end of each epoch $\|\nabla \mathcal{J}(\widetilde{\boldsymbol{\theta}}^s; \widetilde{\boldsymbol{\theta}}^s)\|^2$ of SGD-GD and \alg in all $3$ settings as specified in Sec. \ref{sec:exp} shown in Fig. \ref{fig:grad}. Notably, in MNIST setting (Fig. \ref{fig:grads_sub4} to Fig. \ref{fig:grads_sub6}), \alg constantly has much smaller squared gradient norm, demonstrating the effectiveness of our algorithm. In Credit setting (Fig. \ref{fig:grads_sub1} to Fig. \ref{fig:grads_sub3}), \alg constantly has a smaller squared gradient norm when $\alpha = 0.01$ and $0.4$. When $\alpha = 0.2$, \alg has a much worse initial point than SGD-GD, resulting in the large squared gradient norm at first but similar norm at the end. Similarly, in CIFAR-10 setting  (Fig. \ref{fig:grads_sub7} to Fig. \ref{fig:grads_sub9}), \alg has a smaller squared gradient norm when $\alpha = 50$ and $80$. When $\alpha = 20$, \alg has a much worse initial point than SGD-GD, resulting in the large squared gradient norm at first but similar norm at the end. 

Note that Thm. \ref{theorem:converge} is on the expectation of the squared gradient norm, but we are only able to conduct one experiment for each setting due to the limitations of computational resource. Thus, we can have some bad initialization point of \alg to influence the results.

\noindent \textbf{Out-of-Sample Evaluation.}\label{appx:oos} To assess whether \alg converges to a generalizable solution, we evaluate the out-of-sample evaluation accuracy on a held-out test set for the Credit Dataset and MNIST Dataset under different performativity strengths $\alpha$. The results are shown in Figure~\ref{fig:eval_acc}. The results demonstrate that \textsc{SPRINT} achieves competitive or superior evaluation accuracy compared to SGD-GD across all settings, indicating that its faster convergence to an SPS solution does not come at the cost of generalization. We note that evaluating out-of-sample performance in the PP setting is non-trivial: since the data distribution $D(\theta)$ is itself model-dependent, the test distribution shifts as the model evolves. We evaluate on a held-out subset drawn from the distribution induced by the final deployed model, which reflects the practical deployment scenario where the model has reached a stable point.

\end{document}

